\documentclass[accepted]{uai2026} 
                        
\usepackage[american]{babel}

\usepackage{natbib} 
\usepackage{mathtools} 
\usepackage{booktabs} 
\usepackage{tikz} 

\usepackage{url}

\usepackage{amsmath}
\usepackage{amsthm}
\usepackage{amssymb}
\usepackage{algorithm}
\usepackage{algpseudocode}

\DeclareMathOperator*{\argmin}{arg\,min}

\usepackage{subfig}
\usepackage{multirow}
\usepackage[utf8]{inputenc}
\usepackage{url}
\usepackage[overload]{empheq}
\usepackage{cleveref} 
\usepackage{mathtools}
\usepackage{amsmath}
\usepackage{amssymb}
\usepackage{mathtools}
\usepackage{amsthm}
\usepackage{comment}
\usepackage{algorithm}
\usepackage{algorithmicx}
\usepackage{algpseudocode}
\usepackage{braket}
\usepackage{relsize}
\usepackage{adjustbox}
\usepackage{lmodern}
\theoremstyle{plain}
\usepackage{pifont}
\usepackage{enumitem}

\usepackage[utf8]{inputenc} 
\usepackage{url}            
\usepackage{booktabs}       
\usepackage{amsfonts}       
\usepackage{nicefrac}       
\usepackage{microtype}      

\theoremstyle{plain}
\newtheorem{theorem}{Theorem}[section]
\newtheorem{proposition}[theorem]{Proposition}
\newtheorem{lemma}[theorem]{Lemma}
\newtheorem{corollary}[theorem]{Corollary}
\theoremstyle{definition}
\newtheorem{definition}[theorem]{Definition}
\newtheorem{assumption}[theorem]{Assumption}
\theoremstyle{remark}
\newtheorem{remark}[theorem]{Remark}

\newcommand{\abs}[1]{\left | #1 \right | }

\newcommand{\E}{\mathbb{E}}
\newcommand{\kp}{\mathsf P}
\newcommand{\cp}{\mathcal{P}}
\newcommand{\e}{\mathbf e}
\newcommand{\mcs}{\mathcal{S}}
\newcommand{\mca}{\mathcal{A}}

\newcommand{\mE}{\mathbb{E}}

\DeclareMathOperator*{\cO}{\mathcal{O}}

\title{Efficient Q-Learning and Actor-Critic Methods for Robust Average-Reward Reinforcement Learning}

\author[1]{Yang Xu}
\author[1]{Swetha Ganesh}
\author[1]{Vaneet Aggarwal}
\affil[1]{%
    Purdue University\\
    West Lafayette, Indiana, USA 47907
}
  
\begin{document}
\maketitle

\begin{abstract}
We study model-free methods for distributionally robust infinite-horizon average-reward Markov decision processes (MDPs). We present non-asymptotic convergence analyses of $Q$-learning and actor–critic algorithms for robust average-reward MDPs under contamination, total-variation distance, and Wasserstein uncertainty sets. A key ingredient of our analysis is showing that the optimal robust Bellman operator is a strict contraction with respect to a carefully designed semi-norm. This property enables a stochastic approximation update that learns the optimal robust $Q$-function with $\tilde{\mathcal{O}}(\epsilon^{-2})$ dependence on the target accuracy. We also establish robust TD convergence bounds whose constants are uniform over all stationary policies, yielding an efficient data-driven routine for robust critic estimation. Building on this, we introduce an actor–critic algorithm that learns an $\epsilon$-optimal robust policy with $\tilde{\mathcal{O}}(\epsilon^{-2})$ dependence on the target accuracy. We provide numerical simulations  to illustrate the qualitative behavior of the proposed algorithms. Our results contribute to the theoretical foundations of robust planning under model misspecification, and to model-free approaches for building robust long-run policies directly from simulation data.
\end{abstract}

\section{Introduction}

Reinforcement learning (RL) has produced impressive results in fields such as robotics, finance, and health care by allowing agents to discover effective actions through interaction with their environments. Yet in many practical settings direct interaction is unsafe, prohibitively costly, or constrained by strict data budgets \cite{sunderhauf2018limits, hofer2021sim2real}. Practitioners therefore train agents within simulated environments before deploying them into the physical world. This approach is a common choice for robotic control and autonomous driving. The unavoidable gap between a simulator and reality can nonetheless erode performance once the policy encounters dynamics that were absent during training. Robust RL mitigates this risk by framing learning as an optimization problem over a family of transition models, aiming for reliable behaviour under the most adverse member of that family. In this work, we focus on the setting where a planner interacts only with a fixed nominal simulator but seeks a policy that is robust to transition uncertainty after deployment.

Reinforcement learning with an infinite horizon is usually studied under two reward formulations: the discounted formulation, which geometrically discounts future rewards, and the average-reward formulation, which maximizes the long-run mean return. Although the discounted objective has been widely used, its bias toward immediate gains can produce shortsighted policies in tasks that demand sustained efficiency. The average-reward formulation naturally fits these domains because each decision shapes cumulative performance over time. Typical examples include fleet management in ride-hailing, production scheduling in factories, and network resource allocation, where planners must optimize long-run throughput or service quality under uncertain dynamics. Yet the literature on robust reinforcement learning remains largely unexplored under the average-reward criterion.

Existing works on the robust average-reward RL is still sparse. Existing studies on analyzing $Q$-learning updates provide only asymptotic convergence results \citep{wang2023model, wang2024robust}. \citep{sunpolicy2024} examines the iteration complexity of vanilla policy gradient for the same problem but assumes an oracle that yields the exact robust $Q$ functions. Sample complexity studies on model-based line of work converts the average-reward objective into an equivalent discounted task and then applies planning algorithms \citep{wang2023robust, chen2025sample}. \citep{roch2025finite} provides a model-free non-asymptotic convergence method by bounding the sample complexity of a robust Halpern iteration with recursive sampling. Non-asymptotic results for standard model-free control methods such as $Q$-learning or actor–critic therefore remain open.

Tackling the robust average-reward RL problem is significantly harder than solving its standard non-robust counterpart. In actor-critic and $Q$-learning algorithms, the learner needs to evaluate the value function, the $Q$ function, and the long-run average reward, all under the most adverse transition kernel in the uncertainty set, while only observing samples from the single nominal kernel. We consider three common families of uncertainty sets: contamination sets, total-variation (TV) distance uncertainty sets, and Wasserstein distance uncertainty sets. In the non-robust setting, one can estimate values and rewards directly from trajectories, but robustness introduces an extra optimization layer in which an adversary reshapes the dynamics. Classical estimators \citep{wei2020model, agarwal2021theory, jin2018q} disregard this adversarial aspect and therefore fail. Progress requires new algorithms that can reason about the worst-case dynamics using only the data generated by the nominal model.

\subsection{Challenges and Contributions}
\label{sec:challenges-contributions}

Extending model-free methods such as $Q$-learning and actor-critic algorithms to the robust average-reward setting raises two principal obstacles. First, the optimal robust Bellman operator for $Q$–learning no longer inherits a contraction from a discount factor, and until now no norm nor semi-norm was known to shrink its one–step error. Second, actor–critic methods require accurate policy–gradient estimates, yet these depend on a robust $Q$ function that cannot be obtained with finite–sample accuracy using naive rollouts. Although recent work of \citep{xu2025finite} studies a value–based critic with non–asymptotic error bounds, its result is for a single fixed policy with policy-dependent constants. Thus, its results cannot directly address policy improvements and therefore leaves both obstacles open.

This paper overcomes these difficulties for both $Q$–learning and actor–critic control by
developing a uniform contraction framework for robust average-reward Bellman operators and
by integrating data-driven critic errors into robust mirror-descent policy optimization.  Our main contributions are as follows:

\begin{itemize}
    \item \textbf{Uniform contraction of the optimal robust Bellman operator.}
    Under Assumption~\ref{ass:Qlearning} together with the policy-sequence overlap condition in Assumption~\ref{assump:B}, we construct a semi-norm on the
    $Q$-function space and prove that the optimal robust average–reward Bellman operator for
    $Q$-learning is a strict contraction with a factor $\gamma_H < 1$ that is uniform over all deterministic policies and admissible uncertainty sets
    (Theorem~\ref{thm:Q-learningcontraction}). 
    \item \textbf{Robust $Q$–learning with non-asymptotic sample complexity.}
    Leveraging this uniform contraction, we design a model–free robust $Q$–learning algorithm and
    show that the iterates converge to the optimal robust $Q$–function at a rate that yields a
    sample complexity of $\tilde{\mathcal{O}} \bigl(\epsilon^{-2}\bigr)$ for returning an
    $\epsilon$–optimal robust $Q$ function
    (Theorem~\ref{thm:QleariningComplexity}).  
    All constants in this bound are independent of the particular policy sequence, in contrast to
    prior robust average-reward evaluation results.
\item \textbf{Robust actor–critic with noisy critic and uniform guarantees.}
    We develop a robust actor–critic algorithm that uses a critic subroutine to estimate the $Q$-function and updates the policy via mirror ascent. Different from \citep{xu2025finite}, our analysis is based on a one-step semi-norm contraction of the robust Bellman operator, with the contraction factor being uniform across all stationary policies (Lemma \ref{lem:robust_seminorm-contraction}). This yields critic error bounds whose rates and constants are independent of the policy sequence (Theorem \ref{thm:uniform-V}), which is essential to analyze actor–critic updates that change the policy over time.   Incorporating our uniform critic bounds into the mirror-ascent analysis, we show that the robust actor converges to a near-optimal robust policy in $\mathcal{O}(\log T)$ iterations, yielding total sample complexity $\tilde{\mathcal O}(|\mathcal S||\mathcal A|\epsilon^{-2})$ for contamination uncertainty and $\tilde{\mathcal O}((|\mathcal S||\mathcal A|)^2\epsilon^{-2})$ for TV and Wasserstein uncertainty under the plain-average support estimator in Algorithm~\ref{alg:Qsampling}
    (Theorem~\ref{thm:pg-bound}).  
    This extends the average-reward robust mirror-descent framework of \citep{sunpolicy2024}, which assumes exact robust gradients, to the practical setting with finite-sample critic errors.

\end{itemize}

Conceptually, our work provides the first model-free robust planners for average-reward MDPs whose entire pipeline admits non-asymptotic sample-complexity guarantees, without assuming an oracle for robust $Q$-functions or gradients.

\section{Related Work}

The convergence guarantees of robust average reward  RL have been studied by the following works. \cite{wang2023robust,chen2025sample} both take a model-based perspective, approximating robust average-reward MDPs with discounted MDPs and proving uniform convergence of the robust discounted value function as the discount factor approaches one. \cite{roch2025finite} proposed a model-free robust Halpern iteration that alternates between a Halpern step in the quotient space and a recursively coupled Monte-Carlo sampler to estimate the worst-case Bellman operator. In particular, \cite{chen2025sample, roch2025finite} both obtained the sample complexity of $\tilde{\cO}(\epsilon^{-2})$. Regarding $Q$-learning and actor-critic based methods, \cite{wang2023model} proposed a model-free approach by developing a temporal difference learning framework for robust $Q$-learning algorithm, proving asymptotic convergence using ODE methods in stochastic approximation, martingale theory, and multi-level Monte Carlo estimators to handle non-linearity in the robust Bellman operator. \cite{sunpolicy2024, wang2025provable} developed iteration complexities of $\tilde{\cO}(\epsilon^{-1})$ for variants of vanilla policy gradient algorithm of average reward robust MDPs, both assuming direct queries of the robust $Q$ functions. Thus, \cite{sunpolicy2024, wang2025provable} do not establish explicit convergence rate guarantees in terms of sample complexity.

Finite-sample guarantees are well established in distributionally robust RL under infinite horizon discounted reward settings. with the key recent works on actor-critic methods \cite{wang2022policy, zhou2024natural, li2022first,li2023first, kumar2023policy, kuang2022learning} and $Q$-learning methods \cite{wang2023finite, wang2024sample} focusing on solving the robust Bellman equation by finding the unique fixed-point solution. This approach is made feasible by the strict contraction property of the robust Bellman operator, which arises due to the presence of a discount factor $\gamma < 1$. However, this fundamental approach does not directly extend to the robust average-reward setting, where the absence of a discount factor removes the contraction property under any norm. As a result, techniques that succeed in $Q$-learning and actor-critic algorithms for the discounted case do not transfer directly to robust average-reward RL. Our semi-norm construction can be viewed as recovering a contraction-type structure in the average-reward robust setting where discounting is unavailable.

There has been a rich literature on non-robust average-reward RL with finite-sample guarantees for $Q$-learning and actor–critic algorithms.
These methods typically exploit the linearity of the non-robust Bellman operator \citep{zhang2021finite,ganesh2025sharper}, span-contraction properties \citep{chen2025nonQ,zhang2021finite}, or direct sampling of policy gradients \citep{wei2020model,bai2024regret,xu2025global,ganesh2025regret}, together with ergodicity assumptions. In the robust setting, the Bellman operator becomes non-linear due to the inner maximization over uncertainty sets, and robustness is defined with respect to worst-case transitions that are not directly observed. Consequently, the existing techniques for non-robust average-reward RL do not directly carry over; our work can be viewed as extending contraction-based analyses to this more challenging robust regime.

\section{Setting} 
\label{sec:ramdp}
Consider a robust MDP with state space $\mcs$, action space $\mca$, and deterministic reward function $r$ with $R_{\rm sp}\coloneqq \max_{s,a} r(s,a)-\min_{s,a} r(s,a)$, where $|\mcs|=S$ and $|\mca|=A$. We focus on the standard reward-span normalization of $r\in[0,1]$ and  $R_{\rm sp}\leq 1$. Throughout, $\Delta(\mcs)$ denotes the probability simplex over $\mcs$, and $\mathbf e$ denotes the all-ones vector. The transition kernel is assumed to belong to an uncertainty set $\mathcal P$. At each time step, the environment transits to the next state according to an arbitrary transition kernel $\kp\in\cp$. In this paper, we focus on compact $(s,a)$-rectangular uncertainty sets \citep{nilim2004robustness,iyengar2005robust}, i.e., $\mathcal{P}=\bigotimes_{s,a} \mathcal{P}^a_s$, where $\mathcal{P}^a_s \subseteq \Delta(\mcs)$. This rectangular structure is standard in robust MDPs because it preserves dynamic programming and tractable Bellman operators, but it also rules out correlated perturbations across different state-action pairs. In the contraction proofs below we use convexity of each row set $\mathcal P_s^a$; this property holds for the contamination, total variation, and Wasserstein uncertainty sets studied in this paper. Popular uncertainty sets include those defined by the contamination model \citep{hub65,wang2022policy}, total variation \cite{lim2013reinforcement}, and Wasserstein distance \cite{gao2022distributionally}.
We consider the worst-case average-reward over the uncertainty set of MDPs. Specifically, we define the  robust average-reward of a policy $\pi$ as 
\begin{align}\label{eq:Vdef}
    g^\pi_\cp(s)\coloneqq \min_{\kappa\in\bigotimes_{n\geq 0} \mathcal{P}} \lim_{T\to\infty}\mathbb{E}_{\pi,\kappa}\left[\frac{1}{T}\sum^{T-1}_{t=0}r_t|S_0=s\right],
\end{align}
where $\kappa=(\mathsf P_0,\mathsf P_1...)\in\bigotimes_{n\geq 0} \mathcal{P}$ is an arbitrary sequence of kernels with $\mathsf P_t \in \cp$. It was shown in \cite{wang2023robust} that the worst case under the time-varying model is equivalent to the one under the stationary model:
\begin{align}\label{eq:5}
    g^\pi_\cp(s)= \min_{\kp\in\mathcal{P}} \lim_{T\to\infty}\mathbb{E}_{\pi,\kp}\left[\frac{1}{T}\sum^{T-1}_{t=0}r_t|S_0=s\right].
\end{align}
Therefore, we limit our focus to the stationary model. We refer to the minimizers of \eqref{eq:5} as the worst-case transition kernels for the policy $\pi$, and denote the set of all possible worst-case transition kernels by $\Omega^\pi_g$, i.e., $\Omega^\pi_g \triangleq \{\kp\in\cp: g^\pi_\kp=g^\pi_\cp \}$, where $g^\pi_\kp$ denotes the average reward of policy $\pi$ under the single transition $\kp\in\cp$:
\begin{align}
    g_\kp^\pi(s)\triangleq \lim_{T\to\infty} \mE_{\pi,\kp}\bigg[\frac{1}{T}\sum^{T-1}_{n=0} r_t|S_0=s \bigg].
\end{align}

In this paper, we make similar assumptions for our algorithms regarding ergodicity, which are explicitly stated in Assumption \ref{ass:Qlearning} and Assumption \ref{ass:AC}. The ergodicity assumption is commonly used in the current robust average reinforcement learning literature \cite{sunpolicy2024,wang2025provable}. Under ergodicity of the nominal distribution, with suitable radius restrictions of the uncertainty set, we can derive ergodicity for all kernels in the entire uncertainty set. This indicates that the average reward is independent of the starting state, i.e., for any $\kp\in\cp$ and all $s,s' \in \mcs$,  we have  $ g^\pi_\kp(s) = g^\pi_\kp(s')$. We thus drop the dependence on the initial state and denote $g^\pi_\kp$ as the robust average reward. 
Furthermore, for facilitating analysis, define the mixing time of any $p\in\mathcal{P}$ under policy $\pi$ by
\begin{equation}
    t^p_{\mathrm{mix}}\coloneqq\arg\min_{t \geq 1} \left\{ \max_{\mu_0} \left\| (\mu_0 p_{\pi}^{t})^{\top} - \nu^{\top} \right\|_1 \leq \frac{1}{2} \right\},
\end{equation}
where $p_\pi$ is the finite state Markov chain induced by $\pi$, $\mu_0$ is any initial probability distribution on $\mcs$, and $\nu$ is its invariant distribution. Define the span semi-norm as $\|v\|_{\rm sp}= \max_s v(s)-\min_s v(s)$, and let $h^{\pi,p}$ denote the average-reward relative value function of policy $\pi$ under transition kernel $p$. By Lemma~\ref{lem:wanglemma9}, if $p_\pi$ is irreducible and aperiodic, then
\begin{equation} \label{eq:boundofVsp}
    t^p_{\mathrm{mix}} < +\infty
    \quad \text{and} \quad
    \|h^{\pi,p}\|_{\mathrm{sp}}
    \leq 4 R_{\rm sp} t^p_{\mathrm{mix}}
    \leq 4 R_{\rm sp} t_{\mathrm{mix}} .
\end{equation}
Under the reward normalization $R_{\rm sp}\le 1$, this gives
$\|h^{\pi,p}\|_{\mathrm{sp}}\le 4t_{\mathrm{mix}}$.
Here
\begin{equation*}
t_{\mathrm{mix}} \coloneqq \sup_{p\in\mathcal{P}, \pi\in\Pi} t^p_{\mathrm{mix}}.
\end{equation*}
Under the ergodicity assumptions stated in Assumptions~\ref{ass:Qlearning} and~\ref{ass:AC}, together with the stated radius conditions, the chains are uniformly geometrically ergodic. Hence $t_{\mathrm{mix}}$ is finite and admits a uniform bound over the admissible pairs, which also justifies dropping the dependence on the initial state for the average reward.

We focus on the model-free setting, where only samples from the nominal MDP denoted as $\tilde{\kp}$ (the centroid of the uncertainty set) are available. We now formally define the robust value function $ V^\pi_{\kp_V}$ by connecting it with the following robust Bellman equation: 

\begin{theorem}[Robust Bellman Equation, Theorem 3.1 in \cite{wang2023model}]\label{thm:robust Bellman} 
For a fixed policy $\pi$ and for each $ s \in \mathcal{S}$, define the Robust Bellman operator with scalar $g$ as follows
\begin{equation} \label{eq:bellmanoperator}
    \mathbf{T}^\pi_g(V)(s) = \sum_{a} \pi(a|s) \big[ r(s,a) - g +  \sigma_{\cp^a_s}(V) \big],
\end{equation}
where $\sigma_{\cp^a_s}(V) \coloneqq \min_{p\in\cp^a_s} p^\top V$. If $(g,V)$ is a solution to the robust Bellman equation:
\begin{equation}\label{eq:bellman}
    V(s) = \mathbf{T}^\pi_g(V)(s), \quad \forall s \in \mathcal{S},
\end{equation}
then the scalar $g$ corresponds to the robust average reward, i.e., $g = g^\pi_\cp$, and the worst-case transition kernel $\kp_V$ belongs to the set of minimizing transition kernels, i.e., $\kp_V \in \Omega^\pi_g$, where 
$\Omega^\pi_g \triangleq \{ \kp \in \cp : g^\pi_\kp = g^\pi_\cp \} $. Furthermore, $ V = V^\pi_{\kp_V} + c \mathbf{e}$ for some $c \in \mathbb{R}$ where $ V^\pi_{\kp_V}$ is defined as the relative value function of the policy $\pi$ under the single transition $\kp_V$ as follows:
\begin{align}\label{eq:relativevaluefunction}
    V^\pi_{\kp_V}(s)\triangleq \mE_{\pi,\kp_V}\bigg[\sum^\infty_{t=0} (r_t-g^\pi_{\kp_V})|S_0=s \bigg].
\end{align}
where $S_0$ is the initial state.
\end{theorem}
Theorem \ref{thm:robust Bellman} implies that the robust Bellman equation \eqref{eq:bellman} identifies both the worst-case average reward $g$ and a corresponding value function $V$ that is determined only up to an additive constant. In particular, $\sigma_{\cp^a_s}(V)$ represents the worst-case transition effect over the uncertainty set $\cp^a_s$.

For the critic and actor analysis, we fix the representative $V^\pi$ of the robust relative value function by the same anchor state $s_0$ used in Algorithm~\ref{alg:RobustTD}. Thus $V^\pi(s_0)=0$. Given this representative, define the rowwise Bellman-minimizer set
\begin{equation}
\mathcal M_{s,a}(V^\pi)
\triangleq
\argmin_{p\in\cp_s^a} p^\top V^\pi .
\label{eq:rowwise-minimizer-set}
\end{equation}
Since each row set is compact and $p\mapsto p^\top V^\pi$ is continuous, $\mathcal M_{s,a}(V^\pi)$ is nonempty. By rectangularity, we may choose a transition kernel $\kp_V^\pi\in\cp$ whose row at every state-action pair satisfies
\begin{equation}
\kp_V^\pi(\cdot|s,a)\in \mathcal M_{s,a}(V^\pi),\qquad \forall (s,a)\in\mcs\times\mca .
\label{eq:rowwise-selector}
\end{equation}
Throughout the actor-critic section, $\kp_V^\pi$ denotes such a rowwise Bellman-minimizing selector. This convention removes the ambiguity caused by nonunique worst-case kernels on rows that may not be used by $\pi$.

The robust one-step-deviation $Q$-function is defined by
\begin{align}
\label{eq:robustQdef}
Q^\pi(s,a)
&\triangleq r(s,a)-g^\pi_\cp+\sigma_{\cp_s^a}(V^\pi) \nonumber\\
&=r(s,a)-g^\pi_\cp+\big(\kp_V^\pi(\cdot|s,a)\big)^\top V^\pi .
\end{align}
Equivalently, under the selector $\kp_V^\pi$ in \eqref{eq:rowwise-selector},
\begin{align}
\label{eq:robustQexpectation}
Q^\pi(s,a)
=\mE_{\pi,\kp_V^\pi}\bigg[\sum^\infty_{t=0}(r_t-g^\pi_\cp)\,\bigg|\,S_0=s,A_0=a\bigg].
\end{align}
In the following, we present the explicit forms of $\sigma_{\cp^a_s}(V)$ for different compact uncertainty sets, with $(s,a)-$rectangular structures. 

\noindent \textbf{Contamination Uncertainty Set}\label{sec:con}
The $\delta$-contamination uncertainty set is
$
    \cp^a_s=\{(1-\delta)\tilde{\kp}^a_s+\delta q: q\in\Delta(\mcs) \}, 
$
where $0<\delta<1$ is the radius \citep{hub65}. Under this uncertainty set, the support function can be computed as 
\begin{equation}\label{eq:contamination}
    \sigma_{\cp^a_s}(V)=(1-\delta)(\tilde{\kp}^a_s)^\top V+\delta \min_s V(s),
\end{equation}
and this is linear in the nominal transition kernel $\kp^a_s$.  A direct approach \citep{wang2023model,xu2025finite} is to use the transition to the subsequent state to construct our estimator:
\begin{align}\label{eq:contaminationquery}
    \hat{\sigma}_{\cp^a_s}(V)\triangleq (1-\delta) V(s')+\delta\min_x V(x),
\end{align}
where $s'$ is a subsequent state sample after $(s,a)$. 

\noindent \textbf{Total Variation Uncertainty Set.}
The total variation uncertainty set is  
$
    \cp^a_s=\{q\in\Delta(\mcs): \frac{1}{2}\|q-\tilde{\kp}^a_s\|_1\leq \delta \},
$
the support function can be computed using its dual function \cite{iyengar2005robust}:
\begin{align}\label{eq:tv dual}
    \sigma_{\cp^a_s}(V)=\max_{\mu \in \mathbb R_+^{S}}\big((\tilde{\kp}^a_s)^\top(V-\mu)-\delta \|V-\mu\|_{\mathrm{sp}}  \big).
\end{align}
Here $\mu$ is a vector in $\mathbb R_+^S$. Equivalently, one may write the dual as a maximization over $z\leq V$ after setting $z=V-\mu$.

\textbf{Wasserstein Distance Uncertainty Sets.}
Consider the metric space $(\mathcal{S},d)$ by defining some distance metric $d$. For some parameter $l\in[1,\infty)$ and two distributions $p,q\in\Delta(\mathcal{S})$, define the $l$-Wasserstein distance between them as 
$W_l(q,p)=\inf_{\mu\in\Gamma(p,q)}\|d\|_{\mu,l}$, where $\Gamma(p,q)$ denotes the distributions over $\mathcal{S}\times\mathcal{S}$ with marginal distributions $p,q$, and $\|d\|_{\mu,l}=\big(\mE_{(X,Y)\sim \mu}\big[d(X,Y)^l\big]\big)^{1/l}$. The Wasserstein distance uncertainty set is then defined as 
\begin{equation}
    \cp^a_s=\left\{q\in\Delta(\mcs): W_l(\tilde{\kp}^a_s,q)\leq \delta \right\}.
\end{equation}
The support function w.r.t. the Wasserstein distance set, can be calculated as follows \cite{gao2023distributionally}:
\begin{equation}\label{eq:wd dual}
    \sigma_{\cp^a_s}(V)=\sup_{\lambda\geq 0}\left(-\lambda\delta^l+\mE_{\tilde{\kp}^a_{s}}\big[\inf_{y}\big(V(y)+\lambda d(S,y)^l \big)\big] \right).
\end{equation}

Our goal is to learn a policy that maximizes the worst-case long-run reward $\pi^\star = \arg\max_{\pi}\, g_\cp^{\pi}$.

In the following two sections, we present two model-free algorithms: Robust $Q$-learning and Robust Actor-Critic, for solving the above problem with $\tilde{\mathcal{O}}(\epsilon^{-2})$ dependence on the target accuracy.

\section{Robust $Q$-learning} \label{Qlearning}

In this section, we formally study the convergence of the $Q$-learning algorithm for robust average-reward RL.

\subsection{Optimal robust Bellman equation}
We begin by characterizing the optimal robust Bellman equation in the $Q$-function space and clarifying its solutions' properties. This foundation will support the contraction analysis and the robust $Q$-learning algorithm presented in the remainder of the section. For any $Q\in\mathbb R^{S\times A}$, write $V_Q(s):=\max_{b\in\mca}Q(s,b)$. We now characterize the optimal robust Bellman operator along with the ergodicity assumption used in this setting.
\begin{lemma}[Lemma 4.1 in \cite{wang2023model}]\label{thm:optimal robust Bellman}
Consider the uncentered optimal robust Bellman operator
\begin{align}\label{eq:H}
{\mathbf H}&[Q](s,a)
=r(s,a) \nonumber\\
&+\min_{p\in\mathcal P_{s}^a}
\sum_{s'}p(s'|s,a)\max_{b\in\mca}Q(s',b).
\end{align}
If $(g,Q)$ is a solution to the optimal robust Bellman equation
\begin{equation}\label{eq:optimal bellman}
g+Q(s,a)={\mathbf H}[Q](s,a), \qquad \forall(s,a)\in\mcs\times\mca,
\end{equation}
then
1) $g=g^*_\cp$ \cite{wang2023robust};
2) the greedy policy w.r.t. $Q$, $\pi_Q(s)=\arg\max_a Q(s,a)$, is an optimal robust policy \cite{wang2023robust};
3) $V_Q=V^{\pi_Q}_\kp+c\mathbf{e}$ for some $\kp\in\Omega^{\pi_Q}_g$ and $c\in\mathbb{R}$.
\end{lemma}
\begin{assumption}\label{ass:Qlearning}
Let $\Pi_{\det}$ be the finite set of deterministic policies. 
For each $\pi\in\Pi_{\det}$, the center kernel induced $\tilde{\kp}^\pi$ is irreducible and aperiodic. 
\end{assumption}

Lemma~\ref{thm:optimal robust Bellman} states that any solution $(g,Q)$ of
\eqref{eq:optimal bellman} yields the optimal robust average reward $g_\cp^*$ and an associated deterministic greedy optimal policy. Thus,
Assumption~\ref{ass:Qlearning} is stated only over deterministic stationary
policies, which matches the policy class produced by $Q$-learning and incurs no loss for the optimal-control equation. For the contraction theorem below, Assumption~\ref{ass:Qlearning} is supplemented by
Assumption~\ref{assump:B} in the appendix. Assumption~\ref{assump:B} is an
$L$-step uniform overlap condition for the nominal center kernels along
arbitrary deterministic policy sequences. It is not needed for the Bellman
characterization itself; it is used in Appendix~\ref{Qcontractionproof} to obtain a contraction modulus that is uniform over the changing greedy policies and admissible robust kernels encountered in the $Q$-learning analysis.

\subsection{Semi-norm contraction of the optimal robust Bellman operator}

Lemma \ref{thm:optimal robust Bellman} characterizes the optimal value function as the unique $Q$ (up to a constant shift) that satisfies the robust Bellman equation, yet it does not by itself tell us how to compute that solution.  
In classical discounted RL, the discounted Bellman operator is a contraction in the sup norm, so fixed-point iteration and Q-learning converge. In the robust average-reward setting no discount factor is present and the Bellman operator is non-linear, so contraction is far from obvious. Prior work \cite{xu2025finite} establishes a semi-norm contraction only for the policy-evaluation operator of a single fixed policy. In contrast, we show that the optimal-control operator $\mathbf{H}$, which acts on the larger space $\mathbb{R}^{S\times A}$ and includes a maximization over actions, is a strict contraction under a carefully constructed semi-norm.

To justify the non-asymptotic convergence of a standard $Q$–learning update, the following theorem illustrates the one-step contraction property of the optimal robust Bellman operator under a suitable semi-norm, which in turn enables stochastic iterations for solving the optimal robust Bellman equation.

\begin{theorem} \label{thm:Q-learningcontraction} 
   
    Let Assumptions \ref{ass:Qlearning} and \ref{assump:B} hold. Then, with certain restrictions on the radius of the uncertainty sets, there exists a semi-norm $\|\cdot\|_H$ and $\gamma_H \in (0,1)$ with kernel being $\{c\e: c\in \mathbb{R}\}$ such that for any $Q_1, Q_2 \in \mathbb{R}^{SA}$, we have that 
    \begin{equation} \label{eq:Q-learningcontraction}
        \| {\mathbf H}Q_1 - {\mathbf H}Q_2 \|_H \leq \gamma_H \|Q_1 -Q_2 \|_H
    \end{equation}
\end{theorem}

\begin{proof}[Proof sketch]
Our argument proceeds in three steps. First, Appendix~\ref{Qcontractionproof}
constructs a semi-norm $\|\cdot\|_H$ on $\mathbb R^{S\times A}$ that quotients out constant shifts and is built from the fluctuation matrices associated with all admissible deterministic policies and robust kernels. Under Assumption~\ref{assump:B} and the corresponding radius restrictions, these matrices contract by a common factor.

Second, for any $Q_1,Q_2$, let $\Delta Q=Q_1-Q_2$ and
$\Delta V=V_{Q_1}-V_{Q_2}$. For each state $s$, the scalar $\Delta V(s)$ lies in the convex hull of the action-wise differences
$\{\Delta Q(s,a):a\in\mca\}$. Hence the constructed $Q$-space semi-norm
controls the value-difference term that appears after taking the statewise
maximum.

Third, Lemma~\ref{lem:Hdiff_representation} shows that, for each action slice,
$\bigl({\mathbf H}[Q_1]-{\mathbf H}[Q_2]\bigr)(\cdot,a)$ can be represented by applying an admissible robust kernel and a deterministic policy to
$\Delta V$. Combining this representation with the uniform fluctuation
contraction yields
\begin{equation*}
\|\mathbf H Q_1-\mathbf H Q_2\|_H
\leq \gamma_H\|Q_1-Q_2\|_H,
\end{equation*} 
for some $\gamma_H<1$.
\end{proof}

The concrete proof of Theorem \ref{thm:Q-learningcontraction} including the detailed construction of the semi-norm $\|\cdot\|_H$ is in Appendix \ref{Qcontractionproof}.

\subsection{Simulation-based robust $Q$-learning}

Equipped with the contraction property of ${\mathbf H}$, we now describe a simulation‐based construction of 
$\widehat {\mathbf H}$ that approximates the robust Q‐Bellman operator as follows
\begin{equation} \label{eq:HQ}
\hat{{\mathbf H}}[Q](s,a)=r(s,a)
+\hat{\sigma}_{\cp_{s}^a}\bigl(\max_{b}Q(\cdot,b)\bigr),
\end{equation}
where $\hat{\sigma}_{\cp_{s}^a}$ is an estimator of ${\sigma}_{\cp_{s}^a}$. Recall that Section \ref{sec:ramdp} characterized the closed form expression of ${\sigma}_{\cp_{s}^a}$ under contamination, TV distance, and Wasserstein distance uncertainty sets. Algorithm \ref{alg:sampling} further \citep{xu2025finite} provides sampling methods for obtaining estimators of ${\sigma}_{\cp_{s}^a}$ under all the above uncertainty sets with finite sampling complexities, bounded variances and decaying biases.

We now formally present our robust average-reward $Q$-learning procedure, described in Algorithm \ref{alg:Qlearning}. Algorithm \ref{alg:Qlearning} uses Algorithm \ref{alg:sampling} to obtain $\hat{{\mathbf H}}[Q]$ in the form of \eqref{eq:HQ}, and adapts a synchronous TD learning approach to solve the optimal robust Bellman equation in \eqref{eq:optimal bellman}. The algorithm assumes generative access to the nominal simulator $\tilde\kp$ and synchronously updates all state-action pairs; this is the sampling model analyzed here rather than a trajectory-only online RL setting. An asynchronous variant would naturally update only visited pairs, but proving the same rate would require additional coverage and Markovian-noise arguments. Because the semi-norm in Theorem~\ref{thm:Q-learningcontraction} ignores constant shifts, Line 10 of Algorithm \ref{alg:Qlearning} subtracts the current value at a fixed anchor state-action pair $(s_0,a_0)$. This keeps the iterates inside the quotient space on which ${\bf H}$ contracts.

\begin{algorithm}[ht]
\caption{Robust Average Reward $Q$-Learning}
\label{alg:Qlearning}
\textbf{Input}: Initial values $Q_0$, Stepsizes $\eta_t$, Anchor state $s_0\in\mcs$, Anchor action $a_0\in\mca$
\begin{algorithmic}[1] 
\For {$t = 0,1,\ldots, T-1$}
\State $V_{Q_t} \leftarrow \max_{b\in\mca}Q_t(\cdot,b)$
\For {each $(s,a)\in\mcs\times\mca$} 
\If {Contamination} Sample $\hat{\sigma}_{\cp^a_s}(V_{Q_t})$ according to \eqref{eq:contaminationquery}
\ElsIf{TV or Wasserstein} Sample $\hat{\sigma}_{\cp^a_s}(V_{Q_t})$ according to Algorithm \ref{alg:sampling}
\EndIf
\EndFor
\State $\hat{\mathbf{H}}[Q_t](s,a) \leftarrow  r(s,a) +  \hat{\sigma}_{\cp^a_s}(V_{Q_t}) ,  \forall (s,a) \in \mathcal{S} \! \times \!\mathcal{A}$
\State  $Q_{t+1}(s,a) \leftarrow Q_t(s,a) + \eta_t \left( \hat{\mathbf{H}}(Q_t)(s,a) - Q_t(s,a) \right), \; \forall (s,a) \in \mathcal{S} \times \mca$
\State  $Q_{t+1}(s,a) = Q_{t+1}(s,a) - Q_{t+1}(s_0,a_0), \; \forall (s,a) \in \mathcal{S}\times \mca$
\EndFor \quad
\Return $Q_T$
\end{algorithmic}
\end{algorithm}

The following theorem quantifies the sample complexity needed for Algorithm~\ref{alg:Qlearning} to return an $\epsilon$-accurate estimate of the robust optimal $Q$ function.

\begin{theorem} \label{thm:QleariningComplexity}
If $Q_t$ is generated by Algorithm~\ref{alg:Qlearning} and Assumptions~\ref{ass:Qlearning} and~\ref{assump:B} hold with the radius restrictions in Appendix~\ref{Qcontractionproof}, then with stepsizes $\eta_t=\Theta(1/t)$ there are fixed problem-dependent constants, independent of $\epsilon$ and of the realized policy sequence, such that $\mathbb E[\|Q_T-Q^*\|_\infty^2]\leq \epsilon^2$ using $\tilde{\mathcal O}(\epsilon^{-2})$ dependence on the target accuracy. The state-action factor is $SA$ for all three uncertainty models; the additional logarithmic factors for TV and Wasserstein come from the MLMC support-function estimator.
\end{theorem}

\begin{proof}[Proof sketch]
We view Algorithm~\ref{alg:Qlearning} as a stochastic approximation to
the fixed point of \eqref{eq:optimal bellman} in the quotient space induced by the semi-norm $\|\cdot\|_H$.
The contraction property of $\mathbf H$ (Theorem \ref{thm:Q-learningcontraction}),
together with the bounded variance and decreasing bias of the estimator
$\hat{\sigma}_{P^a_s}$, implies that the mean update has a negative drift under
the smooth semi-Lyapunov function constructed in Appendix~\ref{QleariningComplexityproof}, and that the noise term satisfies standard martingale-difference conditions with linear growth.
Applying the stochastic approximation bound in Appendix~\ref{QleariningComplexityproof} then yields $\tilde{\cO}(\epsilon^{-2})$ synchronous iterations, after absorbing the contraction and Moreau-smoothing constants into the fixed problem-dependent constants. With the MLMC truncation level chosen logarithmically for the TV and Wasserstein estimators, the anchored mean-square sup-norm error satisfies $\mathbb E[\|Q_T-Q^*\|_\infty^2]\leq \epsilon^2$. Counting the state-action samples used in the synchronous updates gives the sample complexities stated in the theorem.
\end{proof}

The concrete proof of Theorem \ref{thm:QleariningComplexity} is in Appendix \ref{QleariningComplexityproof}, where we also discuss the specific radius restrictions under contamination, TV distance, and Wasserstein distance uncertainty sets.

\section{Actor-Critic} \label{actor-critic}

In this section we develop a robust average-reward actor-critic algorithm.
Unlike the robust policy-gradient methods of \citep{sunpolicy2024,wang2025provable},
which assume access to the exact robust $Q$-function or its gradient,
our method learns a robust critic solely from samples of the nominal kernel $\tilde{\kp}$.
On the critic side, we show how to construct, for any policy $\pi$, an estimate
$\hat Q^\pi$ with $\|\hat Q^\pi - Q^\pi\|_\infty \le \epsilon$ using a
number of nominal samples with $\tilde{\mathcal O}(\epsilon^{-2})$ dependence on the target accuracy, and how to obtain
bounds whose constants are uniform over all stationary policies.
On the actor side, we plug these data-driven critics into a robust mirror-ascent update
and prove that the resulting actor-critic algorithm returns an $\epsilon$-optimal robust policy
with the same $\tilde{\mathcal O}(\epsilon^{-2})$ dependence on the target accuracy, with uncertainty-set-dependent state-action factors stated below.
Specifically, given a policy $\pi_k$ at iteration $k$, we first learn the robust critic by running the robust evaluation subroutine to obtain the robust value function and the robust average reward, and then reconstruct an estimate of the robust $Q$ function, denoted as $\hat{Q}^{\pi_k}$. This design allows us to reuse the value-function machinery without assuming ergodicity of the
Markov chain on the enlarged state-action space, which would be required by a direct
Q-based stochastic approximation scheme.

\begin{algorithm}[htb]
\caption{Robust Average Reward Actor-Critic}
\label{alg:AC}
\textbf{Input}: Initial policy $\pi_0$, step size sequence $\{\zeta_k\}$
\begin{algorithmic}[1] 
\For {$k=0,1,\ldots,K-1$}
\State Obtain $\hat{Q}^{\pi_k}$ via Algorithm \ref{alg:Qsampling}
\For {each $s\in\mcs$} 
\State Update policy: $\pi_{k+1}(\cdot|s)
=
\arg\max_{p\in\Delta(\mca)}
\left\{
\zeta_k
\left\langle \hat{Q}^{\pi_k}(s,\cdot),p \right\rangle
-
\|p-\pi_k(\cdot|s)\|^2
\right\}.$

\EndFor 
\EndFor 
\State \Return $\pi_K$
\end{algorithmic}
\end{algorithm}

\paragraph{Critic analysis: }

We now formally describe the critic subroutine, the robust $Q$-function estimation for a fixed policy $\pi$, which is characterized in Algorithm \ref{alg:Qsampling}. Given a policy $\pi$, 
Algorithm \ref{alg:Qsampling} constructs $\hat{Q}^\pi$ by first estimating the robust value function $\hat{V}^\pi$ and robust average reward $\hat{g}^\pi$, and then calculate $\hat{Q}^\pi$ by the Bellman equation (shown in Line 4 of Algorithm \ref{alg:Qsampling}).

\begin{proposition} \label{thm:QfromV}
Let $V^\pi$ be the anchored robust relative value function solving \eqref{eq:bellman}, and let $Q^\pi$ be the robust one-step-deviation $Q$-function defined by \eqref{eq:robustQdef}. Then $Q^\pi$ satisfies the robust Bellman equation
\begin{equation} \label{eq:Qformula}
Q^\pi(s,a)=r(s,a)-g^\pi_\cp+\sigma_{\cp_s^a}(V^\pi),
\end{equation}
and
\begin{equation}
V^\pi(s)=\sum_a \pi(a|s)Q^\pi(s,a),\qquad \forall s\in\mcs .
\end{equation}
Moreover, every selector $\kp_V^\pi$ satisfying \eqref{eq:rowwise-selector} is a worst-case transition kernel for policy $\pi$, and the expectation representation \eqref{eq:robustQexpectation} holds.
\end{proposition}

Proposition~\ref{thm:QfromV} shows that $Q^\pi$ can be reconstructed from $V^\pi$ and $g^\pi_\cp$, both of which are outputs of the policy evaluation routine.

\begin{theorem}\label{thm:Qestimation}
For any fixed policy $\pi$, under Assumption~\ref{ass:AC}, let $\hat Q^\pi$ be produced by Algorithm~\ref{alg:Qsampling}. Under the radius restrictions in Corollary~\ref{cor:AC-radius-to-JSR-primitive}, there are choices of $T$, $N_{\max}$, and $L_Q$, detailed in Appendix~\ref{Qestimationproof}, such that
\begin{equation}
\mathbb E\!\left[\|\hat Q^\pi-Q^\pi\|_\infty^2\right]\leq \epsilon^2,
\qquad
\mathbb E\!\left[\|\hat Q^\pi-Q^\pi\|_\infty\right]\leq \epsilon .
\end{equation}
The total nominal transition-sample complexity is $\tilde{\mathcal O}(|\mathcal S||\mathcal A|\epsilon^{-2})$ for contamination uncertainty and $\tilde{\mathcal O}((|\mathcal S||\mathcal A|)^2\epsilon^{-2})$ for TV or Wasserstein uncertainty, with constants independent of $\pi$.
\end{theorem}
The proof in Appendix~\ref{Qestimationproof} combines finite-sample error bounds for robust value and average-reward estimation with the Lipschitz stability of the $Q$-Bellman map \eqref{eq:Qformula}. The larger TV and Wasserstein state-action factor comes from controlling the final support-function estimates uniformly over all $(s,a)$ pairs.

\begin{remark}
Appendix~\ref{Qestimationproof} proves policy-uniform critic constants. Thus the same critic parameters can be used along the finite policy sequence generated by Algorithm~\ref{alg:AC}; this is the point that fixed-policy evaluation bounds do not provide by themselves.
\end{remark}

\begin{algorithm}[htb]
\caption{Robust Average Reward $Q$-function Estimation}
\label{alg:Qsampling}
\textbf{Input}: Policy $\pi$, number of critic iterations $T$, max level $N_{\rm max}$, support-estimation batch size $L_Q$
\begin{algorithmic}[1]
\State Obtain $V_T^\pi$ and $g_T^\pi$ via Algorithm~\ref{alg:RobustTD}
\For {each $(s,a)\in\mcs\times\mca$}
    \For {$\ell=1,\ldots,L_Q$}
        \If {Contamination}
            Sample $\hat{\sigma}^{(\ell)}_{\cp_s^a}(V_T^\pi)$ according to \eqref{eq:contaminationquery}
        \ElsIf{TV or Wasserstein}
            Sample $\hat{\sigma}^{(\ell)}_{\cp_s^a}(V_T^\pi)$ according to Algorithm~\ref{alg:sampling}
        \EndIf
    \EndFor
    \State
    $\bar{\sigma}_{\cp_s^a}(V_T^\pi)
    \leftarrow
    \frac1{L_Q}\sum_{\ell=1}^{L_Q}
    \hat{\sigma}^{(\ell)}_{\cp_s^a}(V_T^\pi)$
    \State
    $\hat{Q}^\pi(s,a)
    \leftarrow
    r(s,a)-g_T^\pi+\bar{\sigma}_{\cp_s^a}(V_T^\pi)$
\EndFor
\State \Return $\hat{Q}^\pi$
\end{algorithmic}
\end{algorithm}

\paragraph{Actor-critic sample complexity: }
Equipped with the robust critic with finite-sample guarantees, we now characterize the actor subroutine. The robust reward objective $g_{\mathcal P}^\pi=\min_{\kp\in\cp} g_\kp^\pi$ may be nonsmooth because of the inner minimization over transition kernels. The following ergodicity condition is used for the actor analysis.
\begin{assumption}\label{ass:AC}
    For every policy $\pi \in \Pi \coloneqq \{\pi | \pi : \mcs \rightarrow \Delta(\mca)\} $, the center kernel induced Markov chain $\tilde{\kp}^\pi$ is irreducible and aperiodic.
\end{assumption}
We recall the definition of a Fr\'echet sub-gradient:
\begin{definition}
For any function $ f : \mathcal{X} \subseteq \mathbb{R}^N \rightarrow \mathbb{R} $, the Fr\'echet sub-gradient $ u \in \mathbb{R}^N $ is a vector that satisfies
\begin{equation}
\liminf_{\delta \rightarrow 0,\,\delta \neq 0} \frac{f(x+\delta)-f(x)-\langle u,\delta \rangle}{\|\delta\|}\geq 0.
\end{equation}
A vector $u$ is called a Fr\'echet super-gradient of $f$ at $x$ if $-u$ is a Fr\'echet sub-gradient of $-f$ at $x$.
\end{definition}

The result of \citet{sunpolicy2024} is stated for robust average cost, where the adversary maximizes the cost and the policy performs mirror descent. Applying their result to the cost $c=-r$ gives the reward-side statement used here. Specifically, for the robust reward objective, the direction
\begin{equation}
G^\pi(s,a)=d^\pi_{\mathcal P}(s)Q^\pi(s,a)
\end{equation}
is a Fr\'echet super-gradient of $g^\pi_{\mathcal P}$, where $d^\pi_{\mathcal P}$ is the stationary distribution under a worst-case kernel associated with $\pi$. Equivalently, $-G^\pi$ is a Fr\'echet sub-gradient of $-g^\pi_{\mathcal P}$. Thus the reward maximization update is a mirror-ascent update in the direction $Q^\pi$. The proof in Appendix~\ref{acproof} only uses the corresponding reward-side performance-difference inequalities.

Let $D(\pi(\cdot|s), \pi'(\cdot|s))$ denote a Bregman divergence between policies at state $s$, and define the weighted divergence $D_{d}(\pi, \pi') = \sum_{s} d(s)\, D(\pi(\cdot|s), \pi'(\cdot|s))$. The exact robust policy mirror-ascent step takes the form
\begin{equation}
    \pi_{k+1} = \arg\max_{\pi \in \Pi}
    \Bigl\{ \zeta_k \langle G^{\pi_k}, \pi \rangle
    - D_{d^{\pi_k}_{\mathcal{P}}}(\pi, \pi_k) \Bigr\}.
    \label{eq:rmd}
\end{equation}
Choosing $D$ to be the squared Euclidean distance yields the state-wise update below. Indeed, the factor $d^{\pi_k}_{\mathcal P}(s)$ multiplies both the linear term and the squared-distance term at each state, and all its entries are positive under the radius conditions. Hence it does not change the state-wise maximizer.
\begin{align}
    \pi_{k+1}(\cdot|s)
    &= \arg\max_{p \in \Delta(\mathcal{A})}
    \Bigl\{ \zeta_k \langle Q^{\pi_k}(s,\cdot), p \rangle
    - \|p- \pi_k(\cdot|s)\|^2 \Bigr\},
    \label{eq:rmd-statewise}
\end{align}
for all $s\in\mathcal S$.

The update~\eqref{eq:rmd-statewise} assumes access to the exact robust $Q^{\pi_k}$,
which is the setting analyzed by \citet{sunpolicy2024,wang2025provable}.
In our model-free setting, $Q^{\pi_k}$ is unknown and must be estimated.
Algorithm~\ref{alg:AC} therefore replaces $Q^{\pi_k}$ by $\hat Q^{\pi_k}$ obtained
from Algorithm~\ref{alg:Qsampling}, leading to the approximate update

\begin{align}
    \pi_{k+1}(\cdot|s)
    &= \arg\max_{p \in \Delta(\mathcal{A})}
    \Bigl\{ \zeta_k \langle \hat{Q}^{\pi_k}(s,\cdot), p \rangle
    - \|p- \pi_k(\cdot|s)\|^2 \Bigr\},
    \label{eq:rmd-statewise-est}
\end{align}

for all $s \in \mathcal{S}$.
Our analysis shows that, as long as each critic estimate satisfies
$\E\|\hat Q^{\pi_k}-Q^{\pi_k}\|_\infty \le \epsilon$, the actor still converges
to an $\epsilon$-optimal robust policy.

To characterize the overall sample complexity of Algorithm \ref{alg:AC}, define
\begin{equation}
M \coloneqq \sup_{\pi,\,P \in \mathcal{P}} \left\| \frac{d^{\pi^*}_{P}}{d^{\pi}_{P_\pi}} \right\|_\infty,
\label{eq:distribution-mismatch-M}
\end{equation}
where $d^{\pi^*}_P$ and $d^\pi_{P_\pi}$ denote the stationary distributions under policies $\pi^*$ and $\pi$ with transition kernels $P \in \mathcal{P}$ and $P_\pi \in \Omega_g^\pi$, respectively. This quantity also appears in related analyses such as \cite{sunpolicy2024}. Under Assumption~\ref{ass:AC} and the radius restrictions in Corollary~\ref{cor:AC-radius-to-JSR-primitive}, it is finite; see Lemma~\ref{lem:uniform-stationary-lower-bound}. 

\begin{theorem}
\label{thm:pg-bound}
Let Assumption~\ref{ass:AC} hold and suppose the radius restrictions in Corollary~\ref{cor:AC-radius-to-JSR-primitive} hold. Consider Algorithm \ref{alg:AC} with the number of policy iterations set to $K = \Theta\bigl(\log(1/\epsilon)\bigr)$ and stepsize sequence $\zeta_k \geq \zeta_{k-1}\left(1-\frac{1}{M}\right)^{-1}$. Let $\mathcal F_k$ be the history before the $k$-th critic call. Suppose that the critic used at iteration $k$ satisfies
\begin{equation}
\mathbb E\left[\|\hat{Q}^{\pi_k}-Q^{\pi_k}\|_\infty\mid \mathcal F_k\right]\leq \epsilon,
\qquad \forall k\geq 0 .
\end{equation}
This condition is guaranteed by Theorem~\ref{thm:Qestimation} when Algorithm~\ref{alg:Qsampling} is run with fresh nominal simulator samples and with the parameters stated there. Then the policy returned by Algorithm \ref{alg:AC}, denoted by $\pi_K$, satisfies
\begin{equation}
\mathbb E \left[g^{\pi^*}_\mathcal{P}-g^{\pi_K}_\mathcal{P}\right] \leq \mathcal{O}(\epsilon).
\end{equation}
\end{theorem}

\begin{proof}[Proof sketch]
We view \eqref{eq:rmd-statewise-est} as an inexact mirror-ascent step for maximizing $g^\pi_\mathcal P$, where the exact reward-side supergradient direction $Q^{\pi_k}$ is replaced by the noisy critic $\hat Q^{\pi_k}$. The conditional critic guarantee controls this noise at each random policy generated by the previous actor steps. Combining the reward-side performance-difference inequalities with the first-order optimality condition of the Euclidean mirror-ascent step gives a one-step recursion for $g^{\pi^*}_\mathcal P-g^{\pi_k}_\mathcal P$. The recursion contracts by the factor $1-1/M$ up to the critic error and the mirror-ascent residual $2/(M\zeta_k)$. The geometrically increasing sequence $\{\zeta_k\}$ makes this residual summable at the same geometric rate. After $K=\Theta(\log(1/\epsilon))$ actor steps, the remaining error is $\mathcal O(\epsilon)$. Full details are given in Appendix~\ref{acproof}.
\end{proof}

Combining this actor iteration bound with Theorem~\ref{thm:Qestimation} gives total nominal transition-sample complexity $\widetilde{\mathcal O}(|\mathcal S||\mathcal A|\epsilon^{-2})$ under contamination uncertainty and $\widetilde{\mathcal O}((|\mathcal S||\mathcal A|)^2\epsilon^{-2})$ under TV or Wasserstein uncertainty, up to the logarithmic number of actor iterations. This improves on the oracle-gradient setting of \citet{sunpolicy2024,wang2025provable}, where access to $Q^\pi$ (or its robust gradient) is assumed.

\section{Conclusion}
This paper presents the first $Q$-learning and actor–critic algorithms with finite-sample guarantees for distributionally robust average-reward MDPs.
We leverage a semi-norm contraction of the optimal robust Bellman operator and establish uniform finite-sample convergence for the critic estimates across all policies. Together with existing tools, these results yield end-to-end
$\tilde{\mathcal{O}}(\epsilon^{-2})$ dependence on the target accuracy for obtaining an
$\epsilon$-optimal robust policy, with the state-action dependence determined by the uncertainty model.

\section{Acknowledgement}
We would like to thank Ankur Naskar of Purdue University for helpful discussions regarding Appendix A.

\bibliography{main}

\newpage

\onecolumn

\title{Efficient Q-Learning and Actor-Critic Methods for Robust Average-Reward Reinforcement Learning (Supplementary Material)}
\maketitle
\appendix

\section{Missing Proofs for Section \ref{Qlearning}}

Recall that in the $Q$-learning algorithms, the policies are deterministic policies. Thus, we define $\Pi_{\rm det}$ to be the finite set of all deterministic stationary policies with $|\Pi_{\rm det}|= A^S$.

\begin{algorithm}[htb]
\caption{Truncated MLMC Estimator, Algorithm 1 in \cite{xu2025finite}}
\label{alg:sampling}
\textbf{Input}: $s\in \mathcal{S}$, $a\in\mathcal{A}$,  Max level $N_{\max}$, Value function $V$
\begin{algorithmic}[1] 
\State Sample $N\sim\mathrm{Geom}(1/2)$ on $\{0,1,2,\ldots\}$, i.e.,
$\mathbb P(N=n)=2^{-(n+1)}$
\State $N' \leftarrow \min \{N, N_{\max}\}$
\State Collect $2^{N'+1}$ i.i.d. samples of $\{s'_i\}^{2^{N'+1}}_{i=1}$ with $s'_i \sim \tilde{\kp}^a_s$ for each $i$
\State $\hat{\kp}^{a,E}_{s,N'+1} \leftarrow \frac{1}{2^{N'}}\sum_{i=1}^{2^{N'}} \mathbb{1}_{\{s'_{2i}\}}$
\State $\hat{\kp}^{a,O}_{s,N'+1} \leftarrow \frac{1}{2^{N'}}\sum_{i=1}^{2^{N'}} \mathbb{1}_{\{s'_{2i-1}\}}$
\State $\hat{\kp}^{a}_{s,N'+1}\leftarrow\frac{1}{2^{N'+1}}\sum_{i=1}^{2^{N'+1}} \mathbb{1}_{\{s'_i\}}$
\State$\hat{\kp}^{a,1}_{s,N'+1} \leftarrow \mathbb{1}_{\{s'_1\}}$
\If{TV} Obtain $\sigma_{\hat{\kp}^{a,1}_{s,N'+1}}(V), \sigma_{\hat{\kp}^{a}_{s,N'+1}}(V), \sigma_{\hat{\kp}^{a,E}_{s,N'+1}}(V), \sigma_{\hat{\kp}^{a,O}_{s,N'+1}}(V)$ from \eqref{eq:tv dual}
\ElsIf{Wasserstein} Obtain $\sigma_{\hat{\kp}^{a,1}_{s,N'+1}}(V), \sigma_{\hat{\kp}^{a}_{s,N'+1}}(V), \sigma_{\hat{\kp}^{a,E}_{s,N'+1}}(V), \sigma_{\hat{\kp}^{a,O}_{s,N'+1}}(V)$ from \eqref{eq:wd dual}
\EndIf
\State $\Delta_{N'}(V)\leftarrow \sigma_{\hat{\kp}^{a}_{s,N'+1}}(V)-\frac{1}{2}\Bigl[ \sigma_{\hat{\kp}^{a,E}_{s,N'+1}}(V)+  \sigma_{\hat{\kp}^{a,O}_{s,N'+1}}(V)
\Bigr]$
\State Let
\begin{equation*}
p_{N'}\coloneqq \mathbb P(\min\{N,N_{\max}\}=N')
=
\begin{cases}
2^{-(N'+1)}, & 0\leq N'<N_{\max},\\
2^{-N_{\max}}, & N'=N_{\max}.
\end{cases}
\end{equation*}
\State
\begin{equation*}
\hat{\sigma}_{\cp^a_s}(V)
\leftarrow
\sigma_{\hat{\kp}^{a,1}_{s,N'+1}}(V)
+
\frac{\Delta_{N'}(V)}{p_{N'}} .
\end{equation*}
\Return $\hat{\sigma}_{\cp^a_s}(V)$
\end{algorithmic}
\end{algorithm}

\subsection{Proof of Theorem \ref{thm:Q-learningcontraction}} \label{Qcontractionproof}
 For any $\kp \in \cp$, denote $\kp^\pi$ as the transition matrix under policy $\pi \in \Pi_{\rm det}$ and the unique stationary distribution $d^\pi_\kp$. We define the family of fluctuation matrices to be 
\begin{equation} \label{eq:spectralbound}
    \mathcal{F} \coloneqq \{F^\pi_\kp = \kp^\pi - E_\kp : \pi \in \Pi_{\rm det}, \kp \in \cp\}
\end{equation}
where $E_\kp$ is the matrix with all rows being identical to $d^\pi_\kp$.

To construct the desired one-step semi-norm contraction for $\mathbf H$, one key step is to show that the joint spectral radius of the family $\mathcal{F}$ is strictly less than $1$.
We introduce the Dobrushin’s coefficient for an $n\times n$ Markov matrix $P$ defined as follows:
\begin{equation}\label{eq:Dobrushindef}
\tau(P)\ :=\ 1-\min_{i<j}\sum_{s=1}^n \min \big(P_{is},P_{js}\big).
\end{equation}

By Lemma \ref{lem:xulemA.3}, denote the joint spectral radius of $\mathcal{F}$ as $\hat{\rho}(\mathcal{F})$ we have,
\begin{equation}\label{eq:JSRupperbound}
\hat\rho(\mathcal F) \coloneqq \lim_{k\rightarrow \infty} \sup_{F_i \in \mathcal{F}}\|F_k \ldots F_1\|^{\frac{1}{k}}\leq \inf_{m\geq 1}\left(\ \sup_{\substack{\pi_1,\ldots,\pi_m\in\Pi\\ \kp_1, \ldots, \kp_m \in \mathcal P}}
\tau \big(\kp_m^{\pi_m}\cdots \kp_1^{\pi_1}\big)\ \right)^{1/m}.
\end{equation}
Note that \eqref{eq:JSRupperbound} relates the joint spectral radius of $\mathcal{F}$ to the Dobrushin's coefficient of product of $\kp_i^{\pi_i}$. We now introduce a mild assumption to ensure a uniform “minimal mixing’’ across all deterministic policy sequences, which prevents the induced kernels from drifting arbitrarily far apart under adversarial policy switching. The assumption quantifies over deterministic policy sequences and does not require policy randomization.

\begin{assumption}
\label{assump:B}
Let $\{\tilde \kp^{\pi}:\pi\in\Pi_{\det}\}$ be the centroid family.
There exist an integer $L\ge 1$, a constant $\varepsilon_L\in(0,1]$, and a distribution $\nu\in\Delta(\mathcal S)$ such that for every $L$-length deterministic policy sequence $(\pi_1,\ldots,\pi_L)$, for every $s\in\mathcal S$,
\begin{equation} \label{eq:epsilonL}
(\tilde \kp^{\pi_L}\cdots \tilde \kp^{\pi_1})(s,\cdot) \geq \varepsilon_L \nu(\cdot).
\end{equation}
\end{assumption}

\begin{remark}
\label{rem:practicalB}
Assumption~\ref{assump:B} is a sufficient proof condition for obtaining a policy-sequence-uniform contraction; it is not implied by fixed-policy irreducibility alone. It is, however, easy to check in common smoothed tabular models:
\begin{enumerate}
\item \textbf{Single-step minorization.} If there exist $\varepsilon>0$ and $\nu\in\Delta(\mathcal S)$ such that
$\tilde \kp(\cdot | s,a)\ge \varepsilon \nu(\cdot)$ for all $(s,a)$, then Assumption~\ref{assump:B} holds with $L=1$ and $\varepsilon_L=\varepsilon$.
\item \textbf{Rare reset or leakage.} If over every block of $L$ steps, uniformly over the start state and action sequence, the nominal dynamics have probability at least $\beta>0$ of drawing the state from a fixed distribution $\nu$, then Assumption~\ref{assump:B} holds with $\varepsilon_L=\beta$.
\end{enumerate}
Thus, small policy-independent jitter, smoothing, or rare reset events are sufficient to ensure the required uniform overlap.
\end{remark}

Utilizing the $\varepsilon_L$ defined in \eqref{eq:epsilonL}, we now formally discuss the uncertainty radius restrictions in the settings considered.

\begin{lemma}[Contamination uncertainty radius restrictions] \label{lem:contamination_radius_both}
Under Assumption \ref{ass:Qlearning} and Assumption \ref{assump:B}, consider the contamination uncertainty set
\begin{equation}
\cp \coloneqq \Bigl\{\kp: \forall(s,a), \kp(\cdot | s,a) = (1-\delta) \tilde{\kp}(\cdot | s,a) + \delta q(\cdot | s,a),  q(\cdot | s,a)\in\Delta(\mathcal S)\,\Bigr\},
\quad 0\leq \delta <1,
\end{equation}
 no smallness restriction on $\delta$ is needed beyond $\delta<1$.  Specifically, 
\begin{equation} \label{eq:radiusContamination}
\hat\rho(\mathcal F) \leq  \Big(1-(1-\delta)^L \varepsilon_L\Big)^{1/L}\ <\ 1.
\end{equation}
Moreover, for any $\pi$ and $\kp\in\mathcal P$, the induced kernel $\kp^\pi$ is irreducible and aperiodic.
\begin{proof}
For a fixed policy $\pi$, the induced state–transition matrix $\kp^\pi$ is expressed as 
\begin{equation} \label{eq:contaminationppi}
\kp^\pi(s,s') \coloneqq \sum_{a}\pi(a| s) \kp(s'| s,a) = (1-\delta) \sum_{a}\pi(a| s) \tilde{\kp}(s'| s,a) + \delta \sum_{a}\pi(a| s) q(s'| s,a)
\end{equation}
Define the induced uncertainty set $\mathcal P^\pi \coloneqq \{\kp^\pi:  \kp \in\cp, \pi\in\Pi_{\det}\}$ and define $
\tilde{\kp}^\pi (s,s') \coloneqq \sum_{a}\pi(a| s) \tilde{\kp}(s'| s,a).$
Then \eqref{eq:contaminationppi} can be expressed as  
\begin{equation}
\mathcal P^\pi\ = \Bigl\{ (1-\delta) \tilde{\kp}^\pi + \delta q^\pi :q^\pi\ \text{row–stochastic}, \pi\in\Pi_{\det} \Bigr\}.
\end{equation}

Fix any length-$L$ sequence $(\pi_1,\ldots,\pi_L)$ and admissible contaminated kernels $\kp_t^{\pi_t}=(1-\delta)\tilde \kp^{\pi_t}+\delta q_t$.
By nonnegativity, for each $i,j\in \mcs$
\begin{equation*}
\kp_L^{\pi_L}\cdots \kp_1^{\pi_1} (i,j) \geq (1-\delta)^L \tilde \kp^{\pi_L}\cdots \tilde \kp^{\pi_1} (i,j).
\end{equation*}
Assumption~\ref{assump:B} implies for all $i\neq j$,
\begin{equation}
\sum_{s}\min \Big\{\big( \tilde \kp^{\pi_L}\cdots \tilde \kp^{\pi_1}\big)_{is}, \big( \tilde \kp^{\pi_L}\cdots \tilde \kp^{\pi_1}\big)_{js}\Big\}\geq \varepsilon_L >0.
\end{equation}
By the fact that constant scaling preserves minima, hence
\begin{equation}
\sum_{s}\min\Big\{\big(\kp_L^{\pi_L}\cdots \kp_1^{\pi_1}\big)_{is},\ \big(\kp_L^{\pi_L}\cdots \kp_1^{\pi_1}\big)_{js}\Big\}\geq (1-\delta)^L \varepsilon_L >0,
\end{equation}
which is equivalent to the Dobrushin bound
\begin{equation}
\tau \big(\kp_L^{\pi_L}\cdots \kp_1^{\pi_1} \big)\leq 1-(1-\delta)^L \varepsilon_L <\ 1 .
\end{equation}

Taking the supremum over all sequences and applying \eqref{eq:JSRupperbound} implies
\begin{equation*}
\hat\rho(\mathcal F)\leq  \inf_{m\geq 1}\big(\ \sup_{\substack{\pi_1,\ldots,\pi_m\in\Pi\\ \kp_1, \ldots, \kp_m \in \mathcal P}}
\tau \big(\kp_m^{\pi_m}\cdots \kp_1^{\pi_1}\big)\ \big)^{1/m}\leq \Big(1-(1-\delta)^L\varepsilon_L\Big)^{1/L}< 1.
\end{equation*}
Regarding ergodicity, for each fixed $\pi$, since $\tilde \kp^\pi$ is ergodic, there exists $m_\pi$ with $(\tilde \kp^\pi)^{m_\pi}>0$. Then
$
(\kp^\pi)^{m_\pi}\geq (1-\delta)^{m_\pi}(\tilde \kp^\pi)^{m_\pi}>0,
$
so $\kp_\pi$ is also irreducible and aperiodic.
\end{proof}
\end{lemma}

\begin{lemma}[TV distance uncertainty radius restrictions]
\label{lem:tv_radius_both}
Under Assumption \ref{ass:Qlearning} and Assumption \ref{assump:B}, consider the TV uncertainty set
\begin{equation}
\cp \coloneqq \Bigl\{\kp : \forall(s,a), 
\mathrm{TV}\bigl(\kp(\cdot| s,a), \tilde{\kp}(\cdot| s,a)\bigr)\leq \delta \Bigr\},
\quad \delta \ge0,
\end{equation}
Because $\Pi_{\det}$ is finite, there exists $\tilde m\in\mathbb N$ such that $(\tilde\kp^\pi)^{\tilde m}>0$ for all $\pi\in\Pi_{\det}$, set
$
b_0 = \min_{\pi\in\Pi_{\det}} \min_{i,s\in\mathcal S}\ \bigl((\tilde\kp^\pi)^{\tilde m}\bigr)_{is} \in(0,1],
$
and define the TV radius threshold
$
\delta_{\mathrm{TV}}^\star = \min \left\{\frac{\varepsilon_L}{2L}, \frac{b_0}{2\tilde m} \right\}.
$
Then if $0\le \delta<\delta_{\mathrm{TV}}^\star$, then the joint spectral radius of $\mathcal{F}$ is strictly less than $1$. Specifically, 
\begin{equation} \label{eq:radiusTV}
\hat\rho(\mathcal F) \le \Bigl(1-(\varepsilon_L-2L\delta)\Bigr)^{1/L} < 1.
\end{equation}
Moreover, for any $\pi$ and $\kp\in\mathcal P$, the induced kernel $\kp^\pi$ is irreducible and aperiodic.
\begin{proof}
For a fixed policy $\pi$, the induced state–transition matrix $\kp^\pi$ is expressed as 
\begin{equation} \label{eq:tvppi}
\kp^\pi(s,s') \coloneqq \sum_{a}\pi(a| s) \kp(s'| s,a) .
\end{equation}
Then for each state $s$ we have,
\begin{align}
\mathrm{TV}\bigl(\kp^\pi(s,\cdot), \tilde{\kp}^\pi(s,\cdot)\bigr)
=\mathrm{TV} \Bigl(\sum_a \pi(a| s)\kp(\cdot | s,a), \sum_a \pi(a | s)\tilde{\kp}(\cdot | s,a)\Bigr)
\leq \sum_a \pi(a | s)\mathrm{TV}\bigl(\kp(\cdot | s,a),\tilde{\kp}(\cdot |  s,a)\bigr)  \leq \delta,
\end{align}
by convexity of $\mathrm{TV}(\cdot,\cdot)$ in each argument. Hence
\begin{equation} \label{eq:TVpiuncertainty}
\mathcal P^\pi \coloneqq \{\kp^\pi:\kp\in\cp, \pi\in\Pi_{\det}\} \subseteq  \Bigl\{\,M\ \text{row–stochastic}:  \forall s, \exists  \,\pi\in\Pi_{\det}, \mathrm{TV}\bigl(M(s,\cdot), \tilde{\kp}^\pi(s,\cdot)\bigr)\le \delta \Bigr\}. 
\end{equation}
For any $L$-length policy sequence $(\pi_1,\ldots,\pi_L)$ and for any $\kp_t\in\cp$, by \eqref{eq:TVpiuncertainty}, for $t=1,\ldots,L$, we have 
\begin{equation} \label{eq:supTV}
\sup_{i}\mathrm{TV}\bigl(\kp_t^{\pi_t}(i,\cdot),\tilde{\kp}^{\pi_t}(i,\cdot)\bigr)\ \le\ \delta.
\end{equation}
TV is a nonexpansion under right-multiplication by a stochastic matrix, hence the usual telescoping gives, for each start state $i$, by \eqref{eq:supTV} we have
\begin{equation} \label{eq:Ldelta}
\mathrm{TV}\bigl( \kp_1^{\pi_1}\cdots \kp_L^{\pi_L}(i,\cdot),  \tilde\kp^{\pi_1}\cdots \tilde\kp^{\pi_L}(i,\cdot)\bigr)\leq \sum_{t=1}^L \sup_x \mathrm{TV}\bigl(\kp_t^{\pi_t}(x,\cdot),\tilde{\kp}^{\pi_t}(x,\cdot)\bigr)\ \le\ L\delta.
\end{equation}
Then for $i\neq j$,
\begin{align}
\sum_s \min\{\kp_1^{\pi_1}\cdots \kp_L^{\pi_L}(i,s),\kp_1^{\pi_1}\cdots \kp_L^{\pi_L}(j,s)\}
&\geq \sum_s \min\{ \tilde\kp^{\pi_1}\cdots \tilde\kp^{\pi_L}(i,s), \tilde\kp^{\pi_1}\cdots \tilde\kp^{\pi_L}(j,s)\}\nonumber\\
& -\mathrm{TV}\bigl( \kp_1^{\pi_1}\cdots \kp_L^{\pi_L}(i,\cdot),  \tilde\kp^{\pi_1}\cdots \tilde\kp^{\pi_L}(i,\cdot)\bigr)\nonumber\\
&-\mathrm{TV}\bigl( \kp_1^{\pi_1}\cdots \kp_L^{\pi_L}(j,\cdot),  \tilde\kp^{\pi_1}\cdots \tilde\kp^{\pi_L}(j,\cdot)\bigr).
\end{align}
By Assumption \ref{assump:B} and \eqref{eq:Ldelta},  $\sum_s \min\{ \tilde\kp^{\pi_1}\cdots \tilde\kp^{\pi_L}(i,s), \tilde\kp^{\pi_1}\cdots \tilde\kp^{\pi_L}(j,s)\} \ge \varepsilon_L$, and the two TV terms are each $\le L\delta$, whence
\begin{equation*}
\min_{i<j}\sum_s \min\{\kp_1^{\pi_1}\cdots \kp_L^{\pi_L}(i,s),\kp_1^{\pi_1}\cdots \kp_L^{\pi_L}(j,s)\}\ge\ \varepsilon_L-2L\delta.
\end{equation*}
Using the identity $\tau(M)=1-\min_{i<j}\sum_s \min\{M_{is},M_{js}\}$ for row-stochastic $M$ further yields
\begin{equation}
\tau\big(\kp_1^{\pi_1}\cdots \kp_L^{\pi_L}\big) \leq 1-(\varepsilon_L-2L\delta) < 1,
\end{equation}
when $\delta<\varepsilon_L/(2L)$. Taking the supremum over sequences and applying \eqref{eq:JSRupperbound} proves \eqref{eq:radiusTV}.

Regarding ergodicity, fix $\pi\in\Pi_{\det}$. Consider $(\kp^\pi)^{\tilde m}$ versus $(\tilde\kp^\pi)^{\tilde m}$. By the same telescoping as \eqref{eq:Ldelta}, for each start state $i$,
\begin{equation*}
\mathrm{TV}\bigl( (\kp^\pi)^{\tilde m}(i,\cdot), (\tilde\kp^\pi)^{\tilde m}(i,\cdot)\bigr) \leq \tilde m\delta,
\end{equation*}
hence the $\ell_1$ distance between those two state space distributions is at most $2\tilde m\delta$. Thus, for every $(i,s)$,
\begin{equation*}
(\kp^\pi)^{\tilde m}_{is}\ \ge\ (\tilde\kp^\pi)^{\tilde m}_{is}\ -\ \big\|(\kp^\pi)^{\tilde m}(i,\cdot)-(\tilde\kp^\pi)^{\tilde m}(i,\cdot)\big\|_1
\ \ge\ b_0\ -\ 2\tilde m\,\delta.
\end{equation*}
If $\delta< b_0/(2\tilde m)$, then $(\kp^\pi)^{\tilde m}_{is}>0$ for all $i,s$, i.e., $(\kp^\pi)^{\tilde m}>0$ entrywise, so $\kp^\pi$ is irreducible and aperiodic. Since the bound is uniform in $\pi$, the ergodicity holds for all $\pi\in\Pi_{\det}$.
\end{proof}
\end{lemma}

\begin{lemma}[Wasserstein distance uncertainty radius restrictions]
\label{lem:wass_radius_both}
Let $(\mathcal S,d)$ be a finite metric space with
$\delta_{\min}\ :=\ \min_{x\neq y} d(x,y)\ >\ 0$,
and $p\in[1,\infty)$. Under Assumption~\ref{ass:Qlearning} and Assumption~\ref{assump:B}, consider the Wasserstein uncertainty set
$$
\cp \coloneqq \Bigl\{\kp: \forall(s,a), W_p\bigl(\kp(\cdot|s,a),\tilde\kp(\cdot|s,a)\bigr)\leq \delta\Bigr\},\qquad \delta \ge 0.
$$
Since $\Pi_{\det}$ is finite, there exists $\tilde m\in\mathbb N$ with $(\tilde\kp^\pi)^{\tilde m}>0$ for all $\pi$, and define
$
b_0\ = \min_{\pi\in\Pi_{\det}} \min_{i,s\in\mathcal S} \bigl((\tilde\kp^\pi)^{\tilde m}\bigr)_{is} \in(0,1].
$
Set the Wasserstein radius threshold
$
\delta_{\mathrm{W}}^\star \coloneqq\min\left\{\, \frac{\delta_{\min}\varepsilon_L}{2L} , \frac{\delta_{\min}b_0}{2\tilde m}\right\}.
$
Then if $0\le \delta<\delta_{\mathrm{W}}^\star$, then the joint spectral radius of $\mathcal{F}$ is strictly less than $1$. Specifically, 
\begin{equation} \label{eq:radiusW}
\hat\rho(\mathcal F) \le \Bigl(1-(\varepsilon_L-\frac{2L\delta}{\delta_{\min}})\Bigr)^{1/L} < 1.
\end{equation}
Moreover, for any $\pi$ and $\kp\in\mathcal P$, the induced kernel $\kp^\pi$ is irreducible and aperiodic.

\begin{proof}
For a fixed policy $\pi$, the induced state–transition matrix $\kp^\pi$ is expressed as 
\begin{equation} \label{eq:Wppi}
\kp^\pi(s,s') \coloneqq \sum_{a}\pi(a| s) \kp(s'| s,a) 
\end{equation}
For each state $s$, by joint convexity of $W_p^p(\cdot,\cdot;d)$,
\begin{equation*}
\begin{aligned}
W_p^p \Bigl(\kp^\pi(s,\cdot), \tilde{\kp}^{\pi}(s,\cdot);d\Bigr)
&= W_p^p \Bigl(\sum_a \pi(a | s)\kp(\cdot| s,a), \sum_a \pi(a| s)\tilde{\kp}(\cdot| s,a);d\Bigr)\\
&\le \sum_a \pi(a| s)\, W_p^p\bigl(\kp(\cdot| s,a),\tilde{\kp}(\cdot| s,a);d\bigr)
 \leq \delta^p,
\end{aligned}
\end{equation*}
hence $W_p\bigl(\kp^\pi(s,\cdot),\tilde{\kp}^\pi(s,\cdot);d\bigr)\le\delta$ for all $s$, i.e.
\begin{equation}\label{eq:Wpiuncertainty}
 \mathcal P^\pi\ \subseteq  \bigl\{\,M\ \text{row–stochastic}:\ \forall s,\ W_p(M(s,\cdot),\tilde{\kp}^\pi(s,\cdot);d)\leq \delta \bigr\}.
\end{equation}
We now draw connection between \eqref{eq:Wpiuncertainty} and the TV version in \eqref{eq:TVpiuncertainty}. Since the state space is finite, denote $\delta_{\min}:=\min_{x\neq y} d(x,y)>0$. Then, for any distributions $u,v$, we have
\begin{equation*}
W_1(u,v;d)\ \ge\ \delta_{\min}\,\mathrm{TV}(u,v)\qquad\text{and}\qquad W_p(u,v;d)\ \ge\ W_1(u,v;d),
\end{equation*}
which implies that
\begin{equation} \label{eq:W2TVreduction}
\mathrm{TV}(u,v)\ \le\ \frac{W_1(u,v;d)}{\delta_{\min}}\ \leq \frac{W_p(u,v;d)}{\delta_{\min}}.
\end{equation}
Therefore we can reduce \eqref{eq:Wpiuncertainty} into a TV distance uncertainty set characterized as follows:
\begin{equation} \label{eq:W2TVpiuncertainty}
\quad \mathcal P^\pi \subseteq  \Bigl\{\,M\ \text{row–stochastic}:\ \forall s,\ \mathrm{TV}\bigl(M(s,\cdot), \tilde{\kp}^\pi(s,\cdot)\bigr)\le \frac{\delta}{\delta_{\min}} \Bigr\}.
\end{equation}
Invoking Lemma \ref{lem:tv_radius_both} on \eqref{eq:W2TVpiuncertainty} yields the results.
\end{proof}
\end{lemma}

\begin{algorithm}[htb]
\caption{Robust Average Reward TD, Algorithm 2 in \citep{xu2025finite}}
\label{alg:RobustTD}
\textbf{Input}: Policy $\pi$, Initial values $V_0$, $g_0=0$, Stepsizes $\eta_t$, $\beta_t$, Max level $N_{\max}$, Anchor state $s_0\in\mcs$
\begin{algorithmic}[1] 
\For {$t = 0,1,\ldots, T-1$}
\For {each $(s,a)\in\mcs\times\mca$} 
\If {Contamination} Sample $\hat{\sigma}_{\cp^a_s}(V_t)$ according to \eqref{eq:contaminationquery}
\ElsIf{TV or Wasserstein} Sample $\hat{\sigma}_{\cp^a_s}(V_t)$ according to Algorithm \ref{alg:sampling}
\EndIf
\EndFor
\State $\hat{\mathbf{T}}_{g_0}(V_t)(s) \leftarrow \sum_{a} \pi(a|s) \big[ r(s,a) - g_0 +  \hat{\sigma}_{\cp^a_s}(V_t) \big], \quad \forall s \in \mathcal{S}$
\State  $V_{t+1}(s) \leftarrow V_t(s) + \eta_t \left( \hat{\mathbf{T}}_{g_0}(V_t)(s) - V_t(s) \right), \quad \forall s \in \mathcal{S}$
\State  $V_{t+1}(s) = V_{t+1}(s) - V_{t+1}(s_0), \quad \forall s \in \mathcal{S}$
\EndFor
\For {$t = 0,1,\ldots, T-1$}
\For {each $(s,a)\in\mcs\times\mca$} 
\If {Contamination} Sample $\hat{\sigma}_{\cp^a_s}(V_T)$ according to \eqref{eq:contaminationquery}
\ElsIf{TV or Wasserstein} Sample $\hat{\sigma}_{\cp^a_s}(V_T)$ according to Algorithm \ref{alg:sampling}
\EndIf
\EndFor
\State $\hat{\delta}_t(s) \leftarrow \sum_{a}\pi(a|s) \big[ r(s,a) +  \hat{\sigma}_{\cp^a_s}(V_T) \big]- V_T(s)  , \quad \forall s \in \mathcal{S}$
\State $\bar{\delta}_t \leftarrow \frac{1}{S}\sum_s \hat{\delta}_t(s)$
\State $g_{t+1} \leftarrow g_t + \beta_t(\bar{\delta}_t-g_t)$
\EndFor \quad
\Return $V_T$, $g_T$
\end{algorithmic}
\end{algorithm}

Under the conditions of Lemma \ref{lem:contamination_radius_both}-\ref{lem:wass_radius_both}, denote $r^* = \hat{\rho}(\mathcal{F})$, we construct the extremal norm $\|\cdot\|_{\rm ext}$ as follows:
\begin{equation} \label{eq:extremalnormdef}
    \|x\|_{\rm ext} \coloneqq \sup_{k \geq 0} \sup_{F_1,\ldots, F_k \in \mathcal{F}}  \alpha^{-k} \|F_k F_{k-1} \ldots  F_1 x \|_2 
\end{equation}
while $\alpha$ is chosen arbitrarily from  $ (r^*,1)$ and we follow the convention that $\|F_k F_{k-1} \ldots  F_1 x \|_2 = \|x\|_2$ when $k=0$. 

\begin{lemma} \label{lem:extremalnorm}
    Under the conditions of Lemma \ref{lem:contamination_radius_both}-\ref{lem:wass_radius_both}, the operator $\|\cdot\|_{\rm ext}$ is a valid norm with $\|F_\kp^\pi\|_{\rm ext}\leq \alpha<1$ for all $\kp \in \cp$ and $\pi \in \Pi_{\rm det}$.
    \begin{proof}
        We first prove that $\|\cdot\|_{\rm ext}$ is bounded. Following Lemma \ref{lem:bergerlemmaIV} and choosing $\lambda \in (r^*,\alpha)$, then there exist a positive constant $C < \infty$ such that 
        \begin{equation}
             \|F_k F_{k-1} \ldots  F_1 \|_2 \leq C \lambda^k 
        \end{equation}
        Hence for each $k$ and for all $x \in \mathbb{R}^{S}$,
        \begin{equation}
            \alpha^{-k} \|F_k F_{k-1} \ldots  F_1 x \|_2 \leq \alpha^{-k}C\lambda^k \|x\|_2 = C\left(\frac{\lambda}{\alpha}\right)^k\|x\|_2 \longrightarrow 0\quad \text{as} \quad k \rightarrow \infty
        \end{equation}
        Thus the double supremum in \eqref{eq:extremalnormdef} is over a bounded and vanishing sequence, so $\|\cdot\|_{\rm ext}$ bounded. 

        To check that $\|\cdot\|_{\rm ext}$ is a valid norm, note that if $x = \bf{0}$, $\|x\|_{\rm ext}$ is directly $0$. On the other hand, if $\|x\|_{\rm ext}=0$, we have
        \begin{equation}
            \sup_{k = 0} \sup_{F_1,\ldots, F_k \in \mathcal{F}}  \alpha^{-k} \|F_k F_{k-1} \ldots  F_1 x \|_2 = \|x\|_2 = 0 \Rightarrow x = \bf{0}
        \end{equation}
        Regarding homogeneity, observe that for any $c\in \mathbb{R}$ and $x \in \mathbb{R}^{S}$,
        \begin{equation}
            \|cx\|_{\rm ext} = \sup_{k \geq 0} \sup_{F_1,\ldots, F_k \in \mathcal{F}}  \alpha^{-k} \|F_k F_{k-1} \ldots  F_1 (cx) \|_2 = |c| \|x\|_{\rm ext}
        \end{equation}
        Regarding triangle inequality, using $\|F_k \ldots  F_1 (x+y) \|_2 \leq \|F_k \ldots  F_1 x\|_2+\|F_k \ldots  F_1 y \|_2$ for any $x,y \in \mathbb{R}^S$, we obtain,
        \begin{equation}
            \|x+y\|_{\rm ext} = \sup_{k \geq 0} \sup_{F_1,\ldots, F_k \in \mathcal{F}}  \alpha^{-k} \|F_k F_{k-1} \ldots  F_1 (x+y) \|_2 \leq \|x\|_{\rm ext}+\|y\|_{\rm ext}
        \end{equation}
        For any $F_\kp^\pi \in \mathcal{F}$, we have
        \begin{align} \label{eq:extremalcontraction}
            \|F_\kp^\pi x\|_{\rm ext} &= \sup_{k \geq 0} \sup_{F_1,\ldots, F_k \in \mathcal{F}}  \alpha^{-k} \|F_k F_{k-1} \ldots  F_1 (F_\kp^\pi x) \|_2 \nonumber \\
            &\leq \sup_{k \geq 1} \sup_{F_1,\ldots, F_k \in \mathcal{F}}  \alpha^{-(k-1)} \|F_k F_{k-1} \ldots  F_1 x \|_2 \nonumber \\
            &= \alpha \sup_{k \geq 1} \sup_{F_1,\ldots, F_k \in \mathcal{F}}  \alpha^{-k} \|F_k F_{k-1} \ldots  F_1 x \|_2 \nonumber \\
            &\leq \alpha \sup_{k \geq 0} \sup_{F_1,\ldots, F_k \in \mathcal{F}}  \alpha^{-k} \|F_k F_{k-1} \ldots  F_1 x \|_2 \nonumber \\
            &= \alpha \|x\|_{\rm ext}
        \end{align}
        Since $F_\kp^\pi$ is arbitrary, \eqref{eq:extremalcontraction} implies that for any $F_\kp^\pi \in \mathcal{F}$,
        \begin{equation}
            \|F_\kp^\pi\|_{\rm ext} = \sup_{x\neq \bf{0}}\frac{\|F_\kp^\pi x\|_{\rm ext}}{\|x\|_{\rm ext}} \leq \alpha < 1
        \end{equation}
    \end{proof}
    \end{lemma}

We are now ready to construct our desired $\|\cdot\|_H$. For any
deterministic policy $\pi\in\Pi_{\rm det}$ and any
$Q\in\mathbb R^{S\times A}$, write
\begin{equation*}
Q^\pi(s):=Q(s,\pi(s)),\qquad s\in\mcs .
\end{equation*}
Recall that
\begin{equation}
    \|x\|_{\tilde H}
    \coloneqq
    \sup_{F_\kp^\pi\in\mathcal F}\|F_\kp^\pi x\|_{\rm ext},
    \qquad x\in\mathbb R^S .
\end{equation}
Choose a scalar $\xi\in(0,1-\alpha)$ and define
\begin{equation} \label{eq:Hseminormconstruction}
\|Q\|_H
\coloneqq
\sup_{\pi\in\Pi_{\rm det}}\|Q^\pi\|_{\tilde H}
+
\xi\sup_{\pi\in\Pi_{\rm det}}\inf_{c\in\mathbb R}
\|Q^\pi-c\e\|_{\rm ext}.
\end{equation}

\paragraph{(i) $\|\cdot\|_H$ is a valid semi-norm.}
The first term in \eqref{eq:Hseminormconstruction} is a supremum of semi-norms,
and the second term is a supremum of quotient semi-norms. Hence
$\|\cdot\|_H$ is nonnegative, homogeneous, and satisfies the triangle inequality.

\paragraph{(ii) Kernel of $\|\cdot\|_H$.}
Clearly every constant $Q=c\e$ satisfies $\|Q\|_H=0$. Conversely, if
$\|Q\|_H=0$, then for every deterministic policy $\pi$, the vector $Q^\pi$ is
constant over states. Since deterministic policies can choose arbitrary actions
state by state, this implies that all entries $Q(s,a)$ are equal to the same
constant. Hence the kernel is exactly $\{c\e:c\in\mathbb R\}$.
\paragraph{(iii) One‐Step Contraction}

Recall that the optimal robust Bellman operator is defined as follows
\begin{equation}
{\mathbf H}[Q](s,a)=r(s,a)+\min_{p\in\mathcal P_{s}^a}
\sum_{s'}p(s'| s,a)\,\max_{b\in \mca}Q(s',b)
\end{equation}

Since $V_Q(s)=\max_{b\in\mca}Q(s,b)$. For any $Q_1, Q_2 \in \mathbb{R}^{SA}$ and for each $(s,a)$ pair,
\begin{equation} \label{eq:operatordifferencesa}
{\mathbf H}[Q_1](s,a)-{\mathbf H}[Q_2](s,a)
=
\min_{p\in\mathcal P_{s}^a}p^\top V_{Q_1}
- \min_{p\in\mathcal P_{s}^a}p^\top V_{Q_2} \leq \max_{p\in\mathcal P_{s}^a}p^\top (V_{Q_1}-V_{Q_2})
\end{equation}
Likewise, we also have that 
\begin{align} \label{eq:reversed_operatordifferencesa}
{\mathbf H}[Q_2](s,a)-{\mathbf H}[Q_1](s,a)
&\leq \max_{p\in\mathcal P_{s}^a}p^\top (V_{Q_2}-V_{Q_1})  \nonumber\\
&= \max_{p\in\mathcal P_{s}^a}p^\top (V_{Q_2}-V_{Q_1}) \nonumber\\
&= \max_{p\in\mathcal P_{s}^a}-p^\top (V_{Q_1}-V_{Q_2}) \nonumber\\
& = -\min_{p\in\mathcal P_{s}^a}p^\top (V_{Q_1}-V_{Q_2})
\end{align}
By \eqref{eq:operatordifferencesa}-\eqref{eq:reversed_operatordifferencesa}, we thus have
\begin{equation}
    \min_{p\in\mathcal P_{s}^a}p^\top (V_{Q_1}-V_{Q_2}) \leq {\mathbf H}[Q_1](s,a)-{\mathbf H}[Q_2](s,a)\leq\max_{p\in\mathcal P_{s}^a}p^\top (V_{Q_1}-V_{Q_2})
\end{equation}

\begin{lemma}\label{lem:Hdiff_representation}
Fix $Q_1,Q_2\in\mathbb{R}^{S\times A}$ and let
$\Delta V:=V_{Q_1}-V_{Q_2}\in\mathbb{R}^S$. Assume that each row set
$\mathcal P_s^a$ is convex and compact. Then for every deterministic policy
$\mu\in\Pi_{\rm det}$, there exists a transition kernel
$\kp_\mu\in\cp$ such that
\begin{equation}\label{eq:Hdiff2P}
\bigl({\mathbf H}[Q_1]-{\mathbf H}[Q_2]\bigr)^\mu
=
\kp_\mu^\mu\,\Delta V .
\end{equation}
\end{lemma}

\begin{proof}
Fix $\mu\in\Pi_{\rm det}$. For each $s\in\mcs$, denote
\begin{equation*}
h_s:={\mathbf H}[Q_1](s,\mu(s))-{\mathbf H}[Q_2](s,\mu(s)),
\end{equation*}
and denote
$m_s:=\min_{p\in\mathcal P_s^{\mu(s)}}p^\top\Delta V$, $M_s:=\max_{p\in\mathcal P_s^{\mu(s)}}p^\top\Delta V$. By the elementary inequality for differences of minima, $m_s\le h_s\le M_s$.
Since $\mathcal P_s^{\mu(s)}$ is compact and $p\mapsto p^\top\Delta V$ is
linear, we choose $p_s^{\min}\in\arg\min_{p\in\mathcal P_s^{\mu(s)}}p^\top\Delta V$, and
$p_s^{\max}\in\arg\max_{p\in\mathcal P_s^{\mu(s)}}p^\top\Delta V$. If $M_s=m_s$, we set $\bar p_s=p_s^{\min}$. Otherwise we then define $\lambda_s:=\frac{h_s-m_s}{M_s-m_s}\in[0,1]$, and
$\bar p_s:=(1-\lambda_s)p_s^{\min}+\lambda_s p_s^{\max}$.

By convexity, $\bar p_s\in\mathcal P_s^{\mu(s)}$, and by construction
$\bar p_s^\top\Delta V=h_s$. By $(s,a)$-rectangularity, the rows
$\{\bar p_s:s\in\mcs\}$ can be assembled into a kernel $\kp_\mu\in\cp$ by
setting the $(s,\mu(s))$ row equal to $\bar p_s$ and choosing all other rows
arbitrarily from their row uncertainty sets. Then, for every $s$,
\begin{equation}
[\kp_\mu^\mu\Delta V](s)
=
\sum_{s'}\kp_\mu(s'|s,\mu(s))\Delta V(s')
=
\bar p_s^\top\Delta V
=
h_s .
\end{equation}
This proves \eqref{eq:Hdiff2P}.
\end{proof}

For $x\in\mathbb R^S$, define $N_0(x)\coloneqq \sup_{F\in\mathcal F}\|Fx\|_{\rm ext}$. Then the semi-norm \eqref{eq:Hseminormconstruction} can be written as
\begin{equation*}
\|Q\|_H
=
\sup_{\mu\in\Pi_{\det}}N_0(Q^\mu)
+
\xi
\sup_{\mu\in\Pi_{\det}}
\inf_{c\in\mathbb R}\|Q^\mu-c\mathbf e\|_{\rm ext}.
\end{equation*}

Let $\Delta Q=Q_1-Q_2$ and
$\Delta V=V_{Q_1}-V_{Q_2}$. For every state $s$,
\begin{equation*}
\Delta V(s)
=
\max_a Q_1(s,a)-\max_a Q_2(s,a)
\end{equation*}
belongs to the convex hull of
$\{\Delta Q(s,a):a\in\mathcal A\}$. Hence there exists a randomized
state-wise selector $\theta(\cdot|s)$ such that
\begin{equation*}
\Delta V(s)=\sum_a\theta(a|s)\Delta Q(s,a).
\end{equation*}
Equivalently, $\Delta V$ belongs to the convex hull of the deterministic
policy slices $\{\Delta Q^\mu:\mu\in\Pi_{\det}\}$. Since $N_0$ is a
semi-norm,
\begin{equation}
\label{eq:N0-DeltaV-bound}
N_0(\Delta V)
\leq
\sup_{\mu\in\Pi_{\det}}N_0(\Delta Q^\mu)
\leq
\|\Delta Q\|_H .
\end{equation}

Now fix any deterministic policy $\mu$. By Lemma~\ref{lem:Hdiff_representation},
there exists $\kp_\mu\in\mathcal P$ such that
\begin{equation*}
\bigl({\mathbf H}[Q_1]-{\mathbf H}[Q_2]\bigr)^\mu
=
\kp_\mu^\mu\Delta V .
\end{equation*}
Let $E_\mu$ be the rank-one matrix whose rows are the stationary distribution
of $\kp_\mu^\mu$, and set $F_\mu\coloneqq \kp_\mu^\mu-E_\mu\in\mathcal F$. Since $E_\mu\Delta V$ is a constant vector, $F(E_\mu\Delta V)=0$ for all
$F\in\mathcal F$. Therefore,
\begin{equation}
N_0(\kp_\mu^\mu\Delta V)
=
\sup_{F\in\mathcal F}\|F F_\mu\Delta V\|_{\rm ext}
\leq
\alpha N_0(\Delta V),
\end{equation}
where the last inequality follows from the extremal-norm contraction.

Moreover by $\inf_{c\in\mathbb R}
\|\kp_\mu^\mu\Delta V-c\mathbf e\|_{\rm ext}
\leq
\|F_\mu\Delta V\|_{\rm ext}
\leq
N_0(\Delta V),$ then combining the last two displays and taking the supremum over
$\mu\in\Pi_{\det}$ gives
\begin{equation}
\|{\mathbf H}[Q_1]-{\mathbf H}[Q_2]\|_H
\leq
(\alpha+\xi)N_0(\Delta V).
\end{equation}
Using \eqref{eq:N0-DeltaV-bound}, we obtain $\|{\mathbf H}[Q_1]-{\mathbf H}[Q_2]\|_H
\leq
(\alpha+\xi)\|Q_1-Q_2\|_H $. Thus the contraction holds with $\gamma_H\coloneqq \alpha+\xi<1 $.

\subsection{Proof of Theorem \ref{thm:QleariningComplexity}} \label{QleariningComplexityproof}

Based on the established semi-norm contraction property of the Bellman operator $\mathbf{H}$ in Theorem \ref{thm:Q-learningcontraction}, we can formulate Algorithm \ref{alg:Qlearning} as a stochastic approximation framework as follows:

\begin{equation}\label{eq:SA-update}
   Q_{t+1}=Q_t + \eta_t \bigl[\hat{\mathbf{H}}(Q_t) - Q_t\bigr],
   \quad
   \text{where}\quad
   \hat{\mathbf{H}}(Q_t)={\mathbf{H}}(Q_t) + w^t.
\end{equation}

where $w^t\in\mathbb R^{SA}$ is the support-estimation noise with
\begin{equation*}
w^t(s,a)
=
\hat{\sigma}_{\cp_s^a}(V_{Q_t})
-
\sigma_{\cp_s^a}(V_{Q_t}),
\qquad (s,a)\in\mathcal S\times\mathcal A.
\end{equation*}
The second moment of this noise need not be uniformly bounded over arbitrary
iterates $Q_t$. Instead, Lemma~\ref{lem:xuthm3-5} and norm equivalence on the
quotient space imply a linear-growth bound: there exist constants
$A_0,B_0<\infty$ and a bias level $\varepsilon_{\rm bias}$ such that
\begin{equation}
\label{eq:omegabounded}
\mathbb{E}\!\left[
\|w^t\|_{\mathcal H,\overline E}^2
\mid \mathcal F^t
\right]
\leq
A_0+B_0\|Q_t-Q^*\|_{\mathcal H,\overline E}^2,
\qquad
\left\|
\mathbb{E}[w^t\mid\mathcal F^t]
\right\|_{\mathcal H,\overline E}
\leq
\varepsilon_{\rm bias}
\left(1+\|Q_t-Q^*\|_{\mathcal H,\overline E}\right).
\end{equation}
For contamination uncertainty, $\varepsilon_{\rm bias}=0$. For TV and
Wasserstein uncertainty, $\varepsilon_{\rm bias}=O(2^{-N_{\max}/2})$.

The remaining is to analyze the stochastic approximation of the above setting in \eqref{eq:SA-update}-\eqref{eq:omegabounded}. We first express the semi-norm $\|\cdot\|_H$  as infimal convolution between the norm $\|\cdot\|_\mathcal{H}$ and the indicator function of constant vectors. For any norm $\|\cdot\|_c$ and the constant vector class ${\overline{E}} \coloneqq \{c \mathbf{e} : c \in \mathbb{R}\} $, define the indicator function  $\delta_{\overline{E}}$ as
\begin{equation}
\delta_{\bar{E}}(Q) := 
\begin{cases} 
    0 & Q \in \bar{E} \subseteq \mathbb{R}^{SA}, \\
    \infty & \text{{otherwise}.}
\end{cases}
\end{equation}
 then by \cite{zhang2021finite}, the semi-norm induced by norm  $\|\cdot\|_c$  and $\overline{E}$ can be expressed as the infimal convolution of $\|\cdot\|_c$ and the indicator function $\delta_{\overline{E}}$, defined as follows:
\begin{equation}
    \|Q\|_{c,\overline{E}} \coloneqq (\|\cdot\|_c \ast_{\inf} \delta_{\overline{E}})(Q) = \inf_{Q'}  (\|Q-Q'\|_c + \delta_{\overline{E}}(Q'))= \inf_{e\in\overline{E}} \|Q-e\|_c \quad \forall Q\in\mathbb{R}^{SA}.
\end{equation}
Where $\ast_{\inf}$ denotes the infimal convolution operator. Since $\|\cdot\|_H$ constructed in \eqref{eq:Hseminormconstruction} is a semi-norm with kernel being $\overline{E}$, we can construct a norm $\|\cdot\|_{\mathcal{H}}$ via reverse engineering via Lemma \ref{lem:seminorm2norm} such that 
\begin{equation} \label{eq:normNproperty}
\|Q\|_{\mathcal{H},\overline{E}} \coloneqq  (\|\cdot\|_{\mathcal{H}} \ast_{\inf} \delta_{\overline{E}})(Q)=\|Q\|_H
\end{equation}

We thus restate our problem of analyzing the iteration complexity for solving the fixed equivalent class equation $\hat{\mathbf{H}}(Q^*) - Q^* \in \overline{E}$, with the operator $\mathbf{H}:\mathbb{R}^{SA} \to \mathbb{R}^{SA}$ satisfying the contraction property as follows:
\begin{equation}
  \|\mathbf{H}(Q_1) - \mathbf{H}(Q_2)\|_{\mathcal{H},\overline{E}} \leq \gamma_H\|Q_1 - Q_2\|_{\mathcal{H},\overline{E}},
  \quad
  \gamma_H\in(0,1),
  \quad
  \forall Q_1,Q_2\in \mathbb{R}^{SA}
\end{equation}

Under the above framework, we use a linear-growth noise bound rather than a
uniformly bounded-noise assumption. Denote $D_t\coloneqq \|Q_t-Q^*\|_{\mathcal H,\overline E}$, since both $Q_t$ and $Q^*$ are anchored at $(s_0,a_0)$, Lemma~\ref{lem:normtranslations}
implies
\begin{equation*}
\|Q_t-Q^*\|_\infty
\leq
\frac{1}{c_H}\|Q_t-Q^*\|_{\mathcal H,\overline E}.
\end{equation*}
Moreover, the max operator is $1$-Lipschitz in $\|\cdot\|_\infty$, so
\begin{equation*}
\|V_{Q_t}-V_{Q^*}\|_{\rm sp}
\leq
2\|V_{Q_t}-V_{Q^*}\|_\infty
\leq
2\|Q_t-Q^*\|_\infty
\leq
\frac{2}{c_H}D_t .
\end{equation*}
Let $B_\star\coloneqq \|V_{Q^*}\|_{\rm sp}$, by Lemma~\ref{thm:optimal robust Bellman}, $V_{Q^*}$ is an average-reward relative value function up to an additive constant for some optimal greedy policy and some worst-case kernel in $\mathcal P$. Therefore, by Lemma~\ref{lem:wanglemma9}, we have
\begin{equation*}
B_\star \leq 4R_{\rm sp}t_{\mathrm{mix}}.
\end{equation*}
Under the reward-span normalization $R_{\rm sp}\le 1$, this gives $B_\star\le 4t_{\mathrm{mix}}$.
Then
\begin{equation}
\label{eq:VQt-span-linear-growth}
\|V_{Q_t}\|_{\rm sp}
\leq
B_\star+\frac{2}{c_H}D_t .
\end{equation}
Thus Lemma~\ref{lem:xuthm3-5}, together with norm equivalence, implies that
for each uncertainty model $U\in\{\mathrm{Cont},\mathrm{TV},\mathrm{Wass}\}$
there exist constants $A_0^U,B_0^U<\infty$ and a truncation-bias level
$\varepsilon_N^U$ such that
\begin{equation}
\label{eq:noise-linear-growth-final}
\mathbb E\!\left[
\|w^t\|_{\mathcal H,\overline E}^2
\,\middle|\,\mathcal F^t
\right]
\leq
A_0^U+B_0^U D_t^2,
\end{equation}
and
\begin{equation}
\label{eq:bias-linear-growth-final}
\left\|
\mathbb E[w^t\mid\mathcal F^t]
\right\|_{\mathcal H,\overline E}
\leq
\varepsilon_N^U(1+D_t).
\end{equation}
Here
\begin{equation*}
\varepsilon_N^{\mathrm{Cont}}=0,
\qquad
\varepsilon_N^{\mathrm{TV}}=O(2^{-N_{\max}/2}),
\qquad
\varepsilon_N^{\mathrm{Wass}}=O(2^{-N_{\max}/2}),
\end{equation*}
where the hidden constants depend on the uncertainty radius, $B_\star$, and
the norm-equivalence constants, but not on $t$.
Moreover,
\begin{equation*}
A_0^{\mathrm{Cont}}=O(B_\star^2),
\qquad
B_0^{\mathrm{Cont}}=O(1),
\end{equation*}
and
\begin{equation*}
A_0^{\mathrm{TV}}=O(B_\star^2N_{\max}),
\qquad
B_0^{\mathrm{TV}}=O(N_{\max}),
\end{equation*}
with the same orders for Wasserstein uncertainty:
\begin{equation*}
A_0^{\mathrm{Wass}}=O(B_\star^2N_{\max}),
\qquad
B_0^{\mathrm{Wass}}=O(N_{\max}).
\end{equation*}

\subsubsection{Construction of Semi-Lyapunov $M_{\overline{E}}(\cdot)$, Smoothness, and Choice of the Moreau Parameter.}

We make explicit the smoothing parameter used in the generalized Moreau envelope. For each $\mu>0$, let $M_\mu(\cdot)$ be the Moreau envelope function in Definition 2.2 of \cite{chen2020finite}, applied with smoothing parameter $\mu$. Following \cite[Propositions~1--2]{zhang2021finite}, we quotient out the constant-vector class $\overline E$ by defining
\begin{equation}
\label{eq:moreau-quotient-mu}
M_{\overline E,\mu}(Q)=\bigl(M_\mu\ast_{\inf}\delta_{\overline E}\bigr)(Q).
\end{equation}
For each $\mu>0$, there exist constants $c_l(\mu),c_u(\mu)>0$ such that
\begin{equation}
\label{eq:M2span-mu}
c_l(\mu)M_{\overline E,\mu}(Q)
\leq
\frac12\|Q\|_{\mathcal H,\overline E}^2
\leq
c_u(\mu)M_{\overline E,\mu}(Q),
\quad \forall Q\in\mathbb R^{SA}.
\end{equation}
Moreover, $M_{\overline E,\mu}$ is $L_\mu$-smooth w.r.t. another semi-norm $\|\cdot\|_{s,\overline E,\mu}$, and the Moreau approximation ratio satisfies
\begin{equation}
\label{eq:moreau-ratio-limit}
\lim_{\mu\downarrow0}\frac{c_u(\mu)}{c_l(\mu)}=1.
\end{equation}
Since Theorem~\ref{thm:Q-learningcontraction} gives $\gamma_H<1$, we choose $\mu_H>0$ sufficiently small so that
\begin{equation}
\label{eq:moreau-mu-choice-qlearning}
\gamma_H\sqrt{\frac{c_u(\mu_H)}{c_l(\mu_H)}}
\leq
\frac{1+\gamma_H}{2}<1.
\end{equation}
This is only a choice of the smooth Lyapunov function and imposes no additional assumption on the MDP or on the uncertainty set. In the rest of Appendix~\ref{appendix4biasedSA}, we fix this $\mu_H$ and write $M_{\overline E}$, $c_l$, $c_u$, $L$, and $\|\cdot\|_{s,\overline E}$ for $M_{\overline E,\mu_H}$, $c_l(\mu_H)$, $c_u(\mu_H)$, $L_{\mu_H}$, and $\|\cdot\|_{s,\overline E,\mu_H}$, respectively. Thus
\begin{equation} \label{eq:M2span}
c_lM_{\overline E}(Q)
\leq
\frac12\|Q\|_{\mathcal H,\overline E}^2
\leq
c_uM_{\overline E}(Q),
\quad \forall Q\in\mathbb R^{SA}.
\end{equation}
Concretely, $L$-smoothness means
\begin{equation} \label{eq:lsmooth}
M_{\overline E}(Q_2)
\leq
M_{\overline E}(Q_1)+
\langle\nabla M_{\overline E}(Q_1),Q_2-Q_1\rangle+
\frac{L}{2}\|Q_2-Q_1\|_{s,\overline E}^2,
\quad \forall Q_1,Q_2\in\mathbb R^{SA}.
\end{equation}
Furthermore, the gradient of $M_{\overline E}$ satisfies $\langle\nabla M_{\overline E}(Q),c\mathbf e\rangle=0$ for all $Q$ and all $c\in\mathbb R$, and
\begin{equation}  \label{eq:duallsmooth}
\|\nabla M_{\overline E}(Q_1)-\nabla M_{\overline E}(Q_2)\|_{*,s,\overline E}
\leq
L\|Q_2-Q_1\|_{s,\overline E},
\quad \forall Q_1,Q_2\in\mathbb R^{SA}.
\end{equation}
Define
\begin{equation}
\label{eq:alpha2-positive-choice}
\vartheta_H\coloneqq \gamma_H\sqrt{c_u/c_l},
\qquad
\alpha_2\coloneqq \frac{1-\vartheta_H}{2}.
\end{equation}
By \eqref{eq:moreau-mu-choice-qlearning},
\begin{equation}
\label{eq:alpha2-lower-bound-qlearning}
\alpha_2\geq \frac{1-\gamma_H}{4}>0.
\end{equation}
This is the positive drift parameter used in the stepsize below.

Note that since $\|\cdot\|_{s,\overline E}$ and $\|\cdot\|_{\mathcal H,\overline E}$ are semi-norms on a finite-dimensional space with the same kernel, there exist $\rho_1,\rho_2>0$ such that
\begin{equation} \label{eq:normequivalence}
\rho_1\|Q\|_{\mathcal H,\overline E}
\leq
\|Q\|_{s,\overline E}
\leq
\rho_2\|Q\|_{\mathcal H,\overline E},
\quad \forall Q\in\mathbb R^{SA}.
\end{equation}
Likewise, their dual norms, denoted $\|\cdot\|_{*,s,\overline E}$ and $\|\cdot\|_{*,\mathcal H,\overline E}$, satisfy
\begin{equation} \label{eq:dualnormequivalence}
\frac{1}{\rho_2}\|Q\|_{*,s,\overline E}
\leq
\|Q\|_{*,\mathcal H,\overline E}
\leq
\frac{1}{\rho_1}\|Q\|_{*,s,\overline E},
\quad \forall Q\in\mathbb R^{SA}.
\end{equation}

\subsubsection{Formal Derivation of Theorem \ref{thm:QleariningComplexity}} \label{appendix4biasedSA}

By $L$‐smoothness w.r.t.\ $\|\cdot\|_{s,\overline{E}}$ in \eqref{eq:lsmooth}, for each $t$,
\begin{equation} \label{eq:Mt+1decomposition}
  M_{\overline{E}}(Q_{t+1} - Q^*) \leq M_{\overline{E}}(Q_t - Q^*)+\bigl\langle \nabla M_{\overline{E}}(Q_t - Q^*),Q_{t+1}-Q_t \bigr\rangle+\tfrac{L}{2}\|\,Q_{t+1}-Q_t\|_{s,\overline{E}}^2
\end{equation}
where $Q_{t+1}-Q_t = \eta_t \bigl[\hat{\mathbf{H}}(Q_t) - Q_t\bigr] = \eta_t[\mathbf{H}(Q_t) + w^t - Q_t]$. Taking expectation of the second term of the RHS of \eqref{eq:Mt+1decomposition} conditioned on the filtration $\mathcal{F}^t$ we obtain,
\begin{align} \label{eq:nablaMt+1decomposition}
    \mathbb{E}&[\langle \nabla M_{\overline{E}}(Q_t - Q^*), Q_{t+1} - Q_t \rangle | \mathcal{F}^t] \nonumber\\
    &= \eta_t \mathbb{E}[\langle \nabla M_{\overline{E}}(Q_t - Q^*), \mathbf{H}(Q_t) - Q_t + w^t \rangle| \mathcal{F}^t] \nonumber\\
    &= \eta_t\langle \nabla M_{\overline{E}}(Q_t - Q^*),\mathbf{H}(Q_t) - Q_t \rangle + \eta_t\mathbb{E}[\langle \nabla M_{\overline{E}}(Q_t - Q^*),w^t \rangle | \mathcal{F}^t] \nonumber\\
     &= \eta_t\langle \nabla M_{\overline{E}}(Q_t - Q^*), \mathbf{H}(Q_t) - Q_t \rangle + \eta_t\langle \nabla M_{\overline{E}}(Q_t - Q^*), \mathbb{E}[w^t | \mathcal{F}^t] \rangle .
\end{align}

To analyze the additional bias term $\langle \nabla M_{\overline{E}}(Q_t - Q^*), \mathbb{E}[w^t | \mathcal{F}^t] \rangle$, we use the Holder's inequality. Specifically, for any (semi-)norm $\|\cdot\|$ with dual (semi-)norm $\|\cdot\|_*$ (defined by 
$\|u\|_* = \sup\{\langle u,v\rangle : \|v\|\le1\}$),
we have the general inequality
\begin{equation}\label{eq:dualNormIneq}
  \bigl\langle u,\,v\bigr\rangle\leq\|u\|_{*}\,\|v\|,
  \quad
  \forall\,u,v.
\end{equation}
In the biased noise setting, $u=\nabla M_{\overline{E}}(Q_t - Q^*)$ and $v=\mathbb{E}[w^t| \mathcal{F}^t]$, with $\|\cdot\|=\|\cdot\|_{\mathcal{H},\overline{E}}$.  So
\begin{equation} \label{eq:inner_product_bound}
  \bigl\langle 
    \nabla M_{\overline{E}}(Q_t - Q^*),
    \mathbb{E}[w^t| \mathcal{F}^t]
  \bigr\rangle \leq \bigl\|\nabla M_{\overline{E}}(Q_t - Q^*)\bigr\|_{*,\,\mathcal{H},\overline{E}}\cdot \bigl\|\mathbb{E}[w^t| \mathcal{F}^t]\bigr\|_{\mathcal{H},\overline{E}}.   
\end{equation}
Using the linear-growth bias bound \eqref{eq:bias-linear-growth-final}, it remains to bound
$\|\nabla M_{\overline{E}}(Q_t - Q^*)\|_{*,\mathcal{H},\overline{E}}$. By setting 
 $Q_2=0$ in \eqref{eq:duallsmooth}, we get
\begin{equation}  
  \| \nabla M_{\overline{E}}(Q) - \nabla M_{\overline{E}}(0)\|_{*,s,\overline{E}} \leq 
   L\|  Q \|_{s,\overline{E}},
  \quad
  \forall\,Q\in\mathbb{R}^{SA}.
\end{equation}
Thus,
\begin{equation} \label{eq:step1}
     \| \nabla M_{\overline{E}}(Q) \|_{*,s,\overline{E}} \leq  \| \nabla M_{\overline{E}}(0) \|_{*,s,\overline{E}} +
   L\| Q \|_{s,\overline{E}},
  \quad
  \forall\,Q\in\mathbb{R}^{SA}.
\end{equation}
By \eqref{eq:dualnormequivalence}, we know that there exists $\frac{1}{\rho_2}\leq \alpha \leq \frac{1}{ \rho_1}$ such that
\begin{equation} \label{eq:step2}
    \|  \nabla M_{\overline{E}}(Q) \|_{*,\mathcal{H},\overline{E}} \leq \alpha   \|  \nabla M_{\overline{E}}(Q) \|_{*,s,\overline{E}} 
\end{equation}
Thus, combining \eqref{eq:step1} and \eqref{eq:step2} would give:
\begin{equation} 
     \| \nabla M_{\overline{E}}(Q) \|_{*,\mathcal{H},\overline{E}} \leq  \alpha\big(\| \nabla M_{\overline{E}}(0) \|_{*,s,\overline{E}} +
   L\| Q \|_{s,\overline{E}}\big),
  \quad
  \forall\,Q\in\mathbb{R}^{SA}.
\end{equation}
By \eqref{eq:normequivalence}, $\| Q \|_{s,\overline{E}} \leq \rho_2\| Q \|_{\mathcal{H},\overline{E}}$, thus we have:
\begin{equation} 
     \| \nabla M_{\overline{E}}(Q) \|_{*,\mathcal{H},\overline{E}} \leq  \alpha\big(\| \nabla M_{\overline{E}}(0) \|_{*,s,\overline{E}} +
   L\rho_2\| Q \|_{\mathcal{H},\overline{E}}\big),
  \quad
  \forall\,Q\in\mathbb{R}^{SA}.
\end{equation}
Hence, combining the above with \eqref{eq:inner_product_bound}, there exist some 
\begin{equation} \label{eq:G_value}
G := \frac{1}{\rho_1}\max\Big\{L\rho_2,\ \|\nabla M_{\overline{E}}(0)\|_{*,s,\overline{E}}\Big\}=
\mathcal{O}\big(\frac{1}{\rho_1} \max\{L\rho_2, \| \nabla M_{\overline{E}}(0) \|_{*,s,\overline{E}}\} \big)
\end{equation}
such that
\begin{equation}  \label{eq:bias-inner-product-identity}
  \mathbb{E}
  \Bigl[\langle 
    \nabla M_{\overline{E}}(Q_t - Q^*),\,w^t
  \rangle|\mathcal{F}^t\Bigr]=\bigl\langle 
    \nabla M_{\overline{E}}(Q_t - Q^*),
    \mathbb{E}[w^t|\mathcal{F}^t]
  \bigr\rangle.
\end{equation}
Using \eqref{eq:bias-linear-growth-final}, we obtain
\begin{align}
\mathbb{E}
\Bigl[
\langle 
\nabla M_{\overline{E}}(Q_t - Q^*),\,w^t
\rangle
\,|\,\mathcal{F}^t
\Bigr]
&=
\left\langle 
\nabla M_{\overline{E}}(Q_t - Q^*),
\mathbb{E}[w^t|\mathcal{F}^t]
\right\rangle
\nonumber\\
&\leq
G\,\varepsilon_N^U
\bigl(1+D_t\bigr)^2 .
\end{align}
By \eqref{eq:M2span}, $D_t^2\leq 2c_uM_{\overline E}(Q_t-Q^*)$. Hence there exists a constant
$C_{\rm drift}<\infty$ such that
\begin{equation}
\label{eq:Gepsilonbias}
\mathbb{E}
\Bigl[
\langle 
\nabla M_{\overline{E}}(Q_t - Q^*),\,w^t
\rangle
\,|\,\mathcal{F}^t
\Bigr]
\leq
C_{\rm drift}\varepsilon_N^U
\left(
1+M_{\overline E}(Q_t-Q^*)
\right).
\end{equation}
Combining \eqref{eq:Gepsilonbias} with \eqref{eq:nablaMt+1decomposition} we obtain
\begin{align}
\label{eq:nablaMt+1decomposition2}
\mathbb{E}\!\left[
\left\langle
\nabla M_{\overline E}(Q_t-Q^*), Q_{t+1}-Q_t
\right\rangle
\,\middle|\,\mathcal F^t
\right]
&\leq
\eta_t
\left\langle
\nabla M_{\overline E}(Q_t-Q^*),
\mathbf H(Q_t)-Q_t
\right\rangle
\nonumber\\
&\quad+
\eta_t C_{\rm drift}\varepsilon_N^U
\left(
1+M_{\overline E}(Q_t-Q^*)
\right).
\end{align}
To bound the first term in the RHS of \eqref{eq:nablaMt+1decomposition2}, set $X_t:=Q_t-Q^*$ and $Y_t:=\mathbf H(Q_t)-Q^*$. Since $Q^*$ solves the optimal Bellman equation up to an element of $\overline E$, we have $\mathbf H(Q^*)-Q^*\in\overline E$. Hence, because $\|\cdot\|_{\mathcal H,\overline E}$ quotients out $\overline E$, Theorem~\ref{thm:Q-learningcontraction} gives
\begin{equation}
\label{eq:Yt-contraction-H}
\|Y_t\|_{\mathcal H,\overline E}
=
\|\mathbf H(Q_t)-\mathbf H(Q^*)\|_{\mathcal H,\overline E}
\leq
\gamma_H\|X_t\|_{\mathcal H,\overline E}.
\end{equation}
By convexity of $M_{\overline E}$ and by \eqref{eq:M2span},
\begin{align}
\label{eq:nablaMtdecomposition}
\left\langle\nabla M_{\overline E}(X_t),\mathbf H(Q_t)-Q_t\right\rangle
&=
\left\langle\nabla M_{\overline E}(X_t),Y_t-X_t\right\rangle \nonumber\\
&\leq
M_{\overline E}(Y_t)-M_{\overline E}(X_t) \nonumber\\
&\leq
\frac{1}{2c_l}\|Y_t\|_{\mathcal H,\overline E}^2-M_{\overline E}(X_t) \nonumber\\
&\leq
\frac{\gamma_H^2}{2c_l}\|X_t\|_{\mathcal H,\overline E}^2-M_{\overline E}(X_t) \nonumber\\
&\leq
\left(\gamma_H^2\frac{c_u}{c_l}-1\right)M_{\overline E}(X_t) \nonumber\\
&=
-(1-\vartheta_H)(1+\vartheta_H)M_{\overline E}(X_t) \nonumber\\
&\leq
-2\alpha_2M_{\overline E}(X_t).
\end{align}
The last inequality uses $\vartheta_H<1$ and the definition $2\alpha_2=1-\vartheta_H$ in \eqref{eq:alpha2-positive-choice}. Combining \eqref{eq:nablaMtdecomposition}, \eqref{eq:nablaMt+1decomposition2}, and Lemma~\ref{lem:zhanglemma6} with \eqref{eq:Mt+1decomposition}, we arrive as follows:
\begin{align}
\label{eq:modified_proposition3}
\mathbb{E}\!\left[
M_{\overline{E}}(Q_{t+1} - Q^*)
\,\middle|\,
\mathcal{F}^t
\right]
&\leq
\left(
1-2\alpha_2\eta_t+\alpha_3^U\eta_t^2
+
C_{\rm drift}\varepsilon_N^U\eta_t
\right)
M_{\overline{E}}(Q_t-Q^*) \nonumber\\
&\quad+
A_0^U\alpha_4\eta_t^2
+
C_{\rm drift}\varepsilon_N^U\eta_t .
\end{align}
Here $\alpha_2$ is the positive constant from \eqref{eq:alpha2-positive-choice}, and
\begin{align}
\ell_s&\coloneqq \rho_2^2, &
\alpha_3^U&\coloneqq (8+2B_0^U)c_u\ell_sL, &
\alpha_4&\coloneqq \ell_sL .
\end{align}
Choose
\begin{align}
\eta_t&\coloneqq \frac{2}{\alpha_2(t+K_U)}, &
K_U&\coloneqq \max\left\{\frac{4\alpha_3^U}{\alpha_2^2},3\right\}.
\end{align}
Then $\alpha_3^U\eta_t^2\leq \frac{1}{t+K_U}=\frac{\alpha_2}{2}\eta_t$.
For TV and Wasserstein uncertainty, choose $N_{\max}$ large enough so that $C_{\rm drift}\varepsilon_N^U\leq \frac{\alpha_2}{2}$. This condition is used only for the drift inequality. For a target accuracy $\epsilon$, we also choose $N_{\max}$ so that the final bias term in \eqref{eq:biasedSAformal} is at most a constant multiple of $\epsilon^2$. For contamination uncertainty, both conditions are automatic because $\varepsilon_N^{\rm Cont}=0$. Therefore $C_{\rm drift}\varepsilon_N^U\eta_t\leq \frac{1}{t+K_U}$, and \eqref{eq:modified_proposition3} implies
\begin{equation}
\label{eq:linear-growth-SA-recursion}
\mathbb{E}\!\left[
M_{\overline{E}}(Q_{t+1}-Q^*)
\,\middle|\,
\mathcal{F}^t
\right]
\leq
\left(1-\frac{2}{t+K_U}\right)M_{\overline{E}}(Q_t-Q^*)
+
A_0^U\alpha_4\eta_t^2
+
C_{\rm drift}\varepsilon_N^U\eta_t .
\end{equation}

Taking total expectation in \eqref{eq:linear-growth-SA-recursion} and unrolling the resulting recursion gives
\begin{equation}
\label{eq:recursion-unroll-alpha2-fixed}
\mathbb E[M_{\overline E}(Q_T-Q^*)]
\leq
\Gamma_T M_{\overline E}(Q_0-Q^*)
+
\Gamma_T\sum_{t=0}^{T-1}
\frac{A_0^U\alpha_4\eta_t^2+C_{\rm drift}\varepsilon_N^U\eta_t}{\Gamma_{t+1}},
\end{equation}
where $\Gamma_t\coloneqq \prod_{i=0}^{t-1}\left(1-\frac{2}{i+K_U}\right)$. Since $K_U\geq3$,
\begin{equation}
\label{eq:gamma-product-alpha2-fixed}
\Gamma_t=
\frac{(K_U-2)(K_U-1)}{(t+K_U-2)(t+K_U-1)}.
\end{equation}
Using $\eta_t=\frac{2}{\alpha_2(t+K_U)}$, a direct calculation gives
\begin{equation}
\label{eq:weighted-sums-alpha2-fixed}
\Gamma_T\sum_{t=0}^{T-1}\frac{\eta_t^2}{\Gamma_{t+1}}
\leq
\frac{4}{\alpha_2^2(T+K_U-1)},
\qquad
\Gamma_T\sum_{t=0}^{T-1}\frac{\eta_t}{\Gamma_{t+1}}
\leq
\frac{2}{\alpha_2}.
\end{equation}
Combining \eqref{eq:recursion-unroll-alpha2-fixed}--\eqref{eq:weighted-sums-alpha2-fixed} with \eqref{eq:M2span} yields
\begin{equation}
\label{eq:biasedSAformal}
\mathbb{E}\!\left[
\|Q_T-Q^*\|_{\mathcal H,\overline E}^2
\right]
\leq
\frac{K_U^2c_u}{(T+K_U-2)^2c_l}
\|Q_0-Q^*\|_{\mathcal H,\overline E}^2
+
\frac{8A_0^U\alpha_4c_u}{(T+K_U-1)\alpha_2^2}
+
\frac{4c_uC_{\rm drift}\varepsilon_N^U}{\alpha_2}.
\end{equation}

\begin{theorem}[Formal version of Theorem~\ref{thm:QleariningComplexity}]
\label{thm:formalVresult}
Let $Q_t$ be generated by Algorithm~\ref{alg:Qlearning}, and let $Q^*$ be the anchored robust optimal $Q$-function satisfying \eqref{eq:optimal bellman} with $Q^*(s_0,a_0)=0$. Suppose Assumptions~\ref{ass:Qlearning} and~\ref{assump:B} hold, and suppose the radius conditions of Lemmas~\ref{lem:contamination_radius_both}--\ref{lem:wass_radius_both} hold.

For each uncertainty model $U\in\{\mathrm{Cont},\mathrm{TV},\mathrm{Wass}\}$, let $A_0^U$, $B_0^U$, and $\varepsilon_N^U$ be the constants in \eqref{eq:noise-linear-growth-final} and \eqref{eq:bias-linear-growth-final}. Choose the Moreau parameter as in \eqref{eq:moreau-mu-choice-qlearning}, and define $\vartheta_H$ and $\alpha_2$ as in \eqref{eq:alpha2-positive-choice}. In particular, $\alpha_2\geq \frac{1-\gamma_H}{4}>0$. Define
\begin{align}
\ell_s&\coloneqq \rho_2^2, &
\alpha_3^U&\coloneqq (8+2B_0^U)c_u\ell_sL, &
\alpha_4&\coloneqq \ell_sL,
\end{align}
and
\begin{equation}
K_U\coloneqq \max\left\{\frac{4\alpha_3^U}{\alpha_2^2},3\right\}.
\end{equation}
We choose $\eta_t=\frac{2}{\alpha_2(t+K_U)}$. For TV and Wasserstein uncertainty, first choose $N_{\max}$ large enough that $C_{\rm drift}\varepsilon_N^U\leq \frac{\alpha_2}{2}$. For a target accuracy $\epsilon$, choose it large enough also that $\frac{4c_uC_{\rm drift}\varepsilon_N^U}{\alpha_2c_H^2}\leq \epsilon^2/3$. For contamination uncertainty these conditions are automatic because $\varepsilon_N^{\rm Cont}=0$. Then we have
\begin{align}
\mathbb E\!\left[\|Q_T-Q^*\|_\infty^2\right]
&\leq
\frac{K_U^2c_uC_H^2}{(T+K_U-2)^2c_lc_H^2}
\|Q_0-Q^*\|_\infty^2
+
\frac{8A_0^U\alpha_4c_u}{(T+K_U-1)\alpha_2^2c_H^2}
\nonumber\\
&\quad+
\frac{4c_uC_{\rm drift}\varepsilon_N^U}{\alpha_2c_H^2}.
\label{eq:formal-qlearning-bound-clean}
\end{align}
Consequently, choosing $T=\Theta(\epsilon^{-2})$ for contamination uncertainty and choosing $T=\widetilde{\Theta}(\epsilon^{-2})$ with $N_{\max}=\Theta(\log(1/\epsilon))$ for TV and Wasserstein uncertainty yields $\mathbb E\!\left[\|Q_T-Q^*\|_\infty^2\right]\leq \epsilon^2$ after increasing the hidden constants if necessary.
\end{theorem}

\begin{proof}
The bound follows from \eqref{eq:biasedSAformal} and the norm-translation Lemma~\ref{lem:normtranslations}. The stated choices of $T$ and $N_{\max}$ use $\varepsilon_N^{\rm Cont}=0$ and $\varepsilon_N^{\rm TV},\varepsilon_N^{\rm Wass}=O(2^{-N_{\max}/2})$.
\end{proof}

\subsubsection{Sample Complexities for robust $Q$-Learning }

To derive the sample complexities, we use the above results and set $\mathbb{E}\Bigl[\|Q_T - Q^*\|^2_{\infty}\Bigr] \leq \epsilon^2$. Let $C_Q^U$ denote the finite problem-dependent constant obtained from \eqref{eq:formal-qlearning-bound-clean} after absorbing $\alpha_2,\alpha_4,c_l,c_u,c_H,C_H,K_U,C_{\rm drift}$ and the fixed initialization. This constant may depend on the contraction and Moreau-smoothing construction, but it is independent of $\epsilon$, $T$, $N_{\max}$, and the realized policy sequence.

For contamination uncertainty, $\varepsilon_N^{\rm Cont}=0$ and $A_0^{\rm Cont}=\cO(t_{\mathrm{mix}}^2)$. Hence it is enough to take $T=\cO(C_Q^{\rm Cont}t_{\mathrm{mix}}^2\epsilon^{-2})$, resulting in $\cO(SA C_Q^{\rm Cont}t_{\mathrm{mix}}^2\epsilon^{-2})$ nominal transition samples. For TV and Wasserstein uncertainty, take $N_{\max}=\cO(\log(1/\epsilon))$ large enough so that the bias term in \eqref{eq:formal-qlearning-bound-clean} is at most $\epsilon^2$, and use $A_0^U=\cO(t_{\mathrm{mix}}^2N_{\max})$. It is then enough to take $T=\cO(C_Q^{U}t_{\mathrm{mix}}^2N_{\max}\epsilon^{-2})$. Since one call of Algorithm~\ref{alg:sampling} has expected cost $\cO(N_{\max})$ by Lemma~\ref{lem:xuthm3-5}, the resulting sample complexity is $\tilde{\cO}(SA C_Q^{U}t_{\mathrm{mix}}^2\epsilon^{-2})$ for $U\in\{\mathrm{TV},\mathrm{Wass}\}$.

\section{Missing Proofs for Section \ref{actor-critic}} \label{actor-critic-proof}

\subsection{Proof of Proposition \ref{thm:QfromV}}
Let $V^\pi$ be the anchored solution of the robust Bellman equation \eqref{eq:bellman}. For each $(s,a)$, compactness of $\cp_s^a$ and continuity of $p\mapsto p^\top V^\pi$ imply that the set $\mathcal M_{s,a}(V^\pi)$ in \eqref{eq:rowwise-minimizer-set} is nonempty. By rectangularity, choose $\kp_V^\pi\in\cp$ satisfying \eqref{eq:rowwise-selector} for all $(s,a)$.

By the definition of $\kp_V^\pi$, for every $(s,a)$,
\begin{equation}
\sigma_{\cp_s^a}(V^\pi)=\big(\kp_V^\pi(\cdot|s,a)\big)^\top V^\pi .
\label{eq:selector-achieves-sigma}
\end{equation}
Substituting \eqref{eq:selector-achieves-sigma} into the robust Bellman equation gives, for every $s$,
\begin{align}
V^\pi(s)
&=\sum_a \pi(a|s)\left(r(s,a)-g^\pi_\cp+\big(\kp_V^\pi(\cdot|s,a)\big)^\top V^\pi\right) \nonumber\\
&=\sum_a\pi(a|s)r(s,a)-g^\pi_\cp+\sum_{s'}\kp_V^{\pi,\pi}(s,s')V^\pi(s'),
\label{eq:poisson-under-selector}
\end{align}
where $\kp_V^{\pi,\pi}$ is the state transition matrix induced by policy $\pi$ and kernel $\kp_V^\pi$.

Under the radius restrictions used in this section, $\kp_V^{\pi,\pi}$ is irreducible and aperiodic. Let $d^\pi_{\kp_V^\pi}$ be its stationary distribution. Multiplying \eqref{eq:poisson-under-selector} by $d^\pi_{\kp_V^\pi}$ and summing over $s$ cancels the value terms, since $(d^\pi_{\kp_V^\pi})^\top \kp_V^{\pi,\pi}=(d^\pi_{\kp_V^\pi})^\top$. Hence
\begin{equation}
g^\pi_\cp=\sum_s d^\pi_{\kp_V^\pi}(s)\sum_a\pi(a|s)r(s,a)=g^\pi_{\kp_V^\pi}.
\end{equation}
Thus $\kp_V^\pi\in\Omega_g^\pi$.

Now define $Q^\pi$ by \eqref{eq:robustQdef}. The identity \eqref{eq:Qformula} is exactly the first line of \eqref{eq:robustQdef}. Moreover, using the Bellman equation,
\begin{align}
\sum_a\pi(a|s)Q^\pi(s,a)
&=\sum_a\pi(a|s)\left(r(s,a)-g^\pi_\cp+\sigma_{\cp_s^a}(V^\pi)\right) \nonumber\\
&=\mathbf T_g^\pi(V^\pi)(s)
=V^\pi(s).
\end{align}
Finally, since $V^\pi$ solves the Poisson equation \eqref{eq:poisson-under-selector} for policy $\pi$ and kernel $\kp_V^\pi$, the usual one-step decomposition of the relative value gives
\begin{align}
\mE_{\pi,\kp_V^\pi}\bigg[\sum_{t=0}^\infty(r_t-g^\pi_\cp)\,\bigg|\,S_0=s,A_0=a\bigg]
&=r(s,a)-g^\pi_\cp+\sum_{s'}\kp_V^\pi(s'|s,a)V^\pi(s') \nonumber\\
&=r(s,a)-g^\pi_\cp+\sigma_{\cp_s^a}(V^\pi)
=Q^\pi(s,a),
\end{align}
where the second equality uses \eqref{eq:selector-achieves-sigma}. This proves the proposition.

\subsection{Proof of Theorem \ref{thm:pg-bound}}
\label{acproof}
Regarding the detailed radius restrictions for the uncertainty sets, we use Corollary~\ref{cor:AC-radius-to-JSR-primitive}. For ease of notation, we write $g_\mathcal P^\pi$ as $g^\pi$ in this proof.

The robust performance-difference inequalities used below are the reward-side version of the inequalities in \citet{sunpolicy2024}. Their paper states the result for robust average cost. Applying it to $c=-r$ changes the robust cost objective into $-g^\pi$ and changes the cost $Q$-function into $-Q^\pi$. Therefore, for any policies $\pi$ and $\pi'$, and for rowwise Bellman-minimizing worst-case kernels $P_\pi\in\Omega_g^\pi$ and $P_{\pi'}\in\Omega_g^{\pi'}$, we have
\begin{align}
g^{\pi'}-g^\pi
&\geq
\mathbb E_{s\sim d^{\pi'}_{P_{\pi'}}}
\left[
\left\langle Q^\pi(s,\cdot),\pi'(\cdot|s)-\pi(\cdot|s)\right\rangle
\right],
\label{eq:robust-pdl-lower}\\
g^{\pi^\star}-g^\pi
&\leq
\mathbb E_{s\sim d^{\pi^\star}_{P_\pi}}
\left[
\left\langle Q^\pi(s,\cdot),\pi^\star(\cdot|s)-\pi(\cdot|s)\right\rangle
\right].
\label{eq:robust-pdl-upper}
\end{align}

Recall that Algorithm~\ref{alg:AC} uses the update
\begin{align}
\label{eq:update-app}
    \pi_{k+1}(\cdot|s)=\arg\max_{p\in\Delta(\mca)}\{\zeta_k \langle \hat{Q}^{\pi_k}(s,\cdot),p \rangle - \|p - \pi_k(\cdot |s)\|^2\}.
\end{align}
For a fixed state $s$, the objective in \eqref{eq:update-app} is concave in $p$ over the simplex. Hence the first-order optimality condition implies that for any reference distribution $p(\cdot|s)\in\Delta(\mca)$,
\begin{align}
\label{eq:1}
&\zeta_k \left\langle \hat{Q}^{\pi_k}(s, \cdot), p(\cdot|s)-\pi_{k+1}(\cdot|s) \right\rangle
+\left\|\pi_{k+1}(\cdot|s)-\pi_k(\cdot|s)\right\|^2 \nonumber\\
&\quad\leq \left\|p(\cdot|s)-\pi_k(\cdot|s)\right\|^2-
\left\|p(\cdot|s)-\pi_{k+1}(\cdot|s)\right\|^2 .
\end{align}
Taking $p=\pi_k$ in \eqref{eq:1} gives
\begin{align}
\label{eq:105}
\zeta_k \left\langle \hat{Q}^{\pi_k}(s,\cdot),\pi_k(\cdot|s)-\pi_{k+1}(\cdot|s)\right\rangle
\leq -2\left\|\pi_{k+1}(\cdot|s)-\pi_k(\cdot|s)\right\|^2\leq 0 .
\end{align}
Taking $p=\pi^\star$ in \eqref{eq:1} gives
\begin{align}
\label{eq:107}
\left\langle \hat{Q}^{\pi_k}(s,\cdot),\pi^\star(\cdot|s)-\pi_{k+1}(\cdot|s)\right\rangle
&\leq \frac{1}{\zeta_k}\Big[\left\|\pi^\star(\cdot|s)-\pi_k(\cdot|s)\right\|^2 \nonumber\\
&\quad-\left\|\pi^\star(\cdot|s)-\pi_{k+1}(\cdot|s)\right\|^2\Big].
\end{align}

Applying \eqref{eq:robust-pdl-lower} with $(\pi,\pi')=(\pi_k,\pi_{k+1})$ gives
\begin{align}
g^{\pi_k}-g^{\pi_{k+1}}
&\leq
\mathbb E_{s\sim d^{\pi_{k+1}}_{P_{\pi_{k+1}}}}
\left[
\left\langle Q^{\pi_k}(s,\cdot),\pi_k(\cdot|s)-\pi_{k+1}(\cdot|s)\right\rangle
\right] \nonumber\\
&=
\mathbb E_{s\sim d^{\pi_{k+1}}_{P_{\pi_{k+1}}}}
\left[
\left\langle \hat Q^{\pi_k}(s,\cdot),\pi_k(\cdot|s)-\pi_{k+1}(\cdot|s)\right\rangle
\right] \nonumber\\
&\quad+
\mathbb E_{s\sim d^{\pi_{k+1}}_{P_{\pi_{k+1}}}}
\left[
\left\langle Q^{\pi_k}(s,\cdot)-\hat Q^{\pi_k}(s,\cdot),\pi_k(\cdot|s)-\pi_{k+1}(\cdot|s)\right\rangle
\right].
\label{eq:actor-lower-applied}
\end{align}
The first inner product on the right side of \eqref{eq:actor-lower-applied} is nonpositive for every state by \eqref{eq:105}. From the definition of $M$ in \eqref{eq:distribution-mismatch-M},
\begin{equation}
\frac{d^{\pi^\star}_{P_{\pi_k}}(s)}{d^{\pi_{k+1}}_{P_{\pi_{k+1}}}(s)}\leq M,
\qquad \forall s\in\mcs .
\end{equation}
Since multiplying a nonpositive quantity by a larger nonnegative weight makes it smaller, we have
\begin{align}
g^{\pi_k}-g^{\pi_{k+1}}
&\leq
\frac{1}{M}
\mathbb E_{s\sim d^{\pi^\star}_{P_{\pi_k}}}
\left[
\left\langle \hat Q^{\pi_k}(s,\cdot),\pi_k(\cdot|s)-\pi_{k+1}(\cdot|s)\right\rangle
\right] \nonumber\\
&\quad+
\mathbb E_{s\sim d^{\pi_{k+1}}_{P_{\pi_{k+1}}}}
\left[
\left\langle Q^{\pi_k}(s,\cdot)-\hat Q^{\pi_k}(s,\cdot),\pi_k(\cdot|s)-\pi_{k+1}(\cdot|s)\right\rangle
\right].
\label{eq:108}
\end{align}
Applying \eqref{eq:robust-pdl-upper} with $\pi=\pi_k$ gives
\begin{align}
g^{\pi^\star}-g^{\pi_k}
&\leq
\mathbb E_{s\sim d^{\pi^\star}_{P_{\pi_k}}}
\left[
\left\langle \hat Q^{\pi_k}(s,\cdot),\pi^\star(\cdot|s)-\pi_k(\cdot|s)\right\rangle
\right] \nonumber\\
&\quad+
\mathbb E_{s\sim d^{\pi^\star}_{P_{\pi_k}}}
\left[
\left\langle Q^{\pi_k}(s,\cdot)-\hat Q^{\pi_k}(s,\cdot),\pi^\star(\cdot|s)-\pi_k(\cdot|s)\right\rangle
\right].
\label{eq:actor-upper-applied}
\end{align}
Multiplying \eqref{eq:108} by $M$ and adding \eqref{eq:actor-upper-applied}, we obtain
\begin{align}
&\left(g^{\pi^\star}-g^{\pi_k}\right)+M\left(g^{\pi_k}-g^{\pi_{k+1}}\right) \nonumber\\
&\leq
\mathbb E_{s\sim d^{\pi^\star}_{P_{\pi_k}}}
\left[
\left\langle \hat Q^{\pi_k}(s,\cdot),\pi^\star(\cdot|s)-\pi_{k+1}(\cdot|s)\right\rangle
\right] \nonumber\\
&\quad+
\mathbb E_{s\sim d^{\pi^\star}_{P_{\pi_k}}}
\left[
\left\langle Q^{\pi_k}(s,\cdot)-\hat Q^{\pi_k}(s,\cdot),\pi^\star(\cdot|s)-\pi_k(\cdot|s)\right\rangle
\right] \nonumber\\
&\quad+M\mathbb E_{s\sim d^{\pi_{k+1}}_{P_{\pi_{k+1}}}}
\left[
\left\langle Q^{\pi_k}(s,\cdot)-\hat Q^{\pi_k}(s,\cdot),\pi_k(\cdot|s)-\pi_{k+1}(\cdot|s)\right\rangle
\right].
\label{eq:actor-combined-before-error}
\end{align}
For any two distributions $u,v\in\Delta(\mca)$, $\|u-v\|_1\leq 2$. Therefore H\"older's inequality gives
\begin{equation}
\left|\left\langle Q^{\pi_k}(s,\cdot)-\hat Q^{\pi_k}(s,\cdot),u-v\right\rangle\right|
\leq 2\|Q^{\pi_k}-\hat Q^{\pi_k}\|_\infty .
\end{equation}
Using this bound in the two error terms of \eqref{eq:actor-combined-before-error}, and then using \eqref{eq:107}, yields
\begin{align}
&\left(g^{\pi^\star}-g^{\pi_k}\right)+M\left(g^{\pi_k}-g^{\pi_{k+1}}\right) \nonumber\\
&\leq
\frac{1}{\zeta_k}
\mathbb E_{s\sim d^{\pi^\star}_{P_{\pi_k}}}
\left[
\left\|\pi^\star(\cdot|s)-\pi_k(\cdot|s)\right\|^2-
\left\|\pi^\star(\cdot|s)-\pi_{k+1}(\cdot|s)\right\|^2
\right] \nonumber\\
&\quad+2(M+1)\|Q^{\pi_k}-\hat Q^{\pi_k}\|_\infty .
\label{eq:actor-combined-after-error}
\end{align}
The left side can be rewritten as
\begin{equation}
\left(g^{\pi^\star}-g^{\pi_k}\right)+M\left(g^{\pi_k}-g^{\pi_{k+1}}\right)
=M\left(g^{\pi^\star}-g^{\pi_{k+1}}\right)-(M-1)\left(g^{\pi^\star}-g^{\pi_k}\right).
\label{eq:109}
\end{equation}
Combining \eqref{eq:actor-combined-after-error} and \eqref{eq:109} gives the samplewise recursion
\begin{align}
g^{\pi^\star}-g^{\pi_{k+1}}
&\leq \left(1-\frac{1}{M}\right)\left(g^{\pi^\star}-g^{\pi_k}\right)
+\frac{1}{M\zeta_k}
\mathbb E_{s\sim d^{\pi^\star}_{P_{\pi_k}}}
\left[
\left\|\pi^\star(\cdot|s)-\pi_k(\cdot|s)\right\|^2\right. \nonumber\\
&\quad\left.-\left\|\pi^\star(\cdot|s)-\pi_{k+1}(\cdot|s)\right\|^2
\right]
+2\left(1+\frac{1}{M}\right)\|Q^{\pi_k}-\hat Q^{\pi_k}\|_\infty \nonumber\\
&\leq \left(1-\frac{1}{M}\right)\left(g^{\pi^\star}-g^{\pi_k}\right)
+\frac{2}{M\zeta_k}
+2\left(1+\frac{1}{M}\right)\|Q^{\pi_k}-\hat Q^{\pi_k}\|_\infty .
\label{eq:actor-samplewise-recursion}
\end{align}
The last inequality uses $\|u-v\|^2\leq 2$ for probability vectors under the Euclidean norm.

Let $\mathcal F_k$ be the history before the $k$-th critic call. The policy $\pi_k$ is $\mathcal F_k$-measurable, and Algorithm~\ref{alg:Qsampling} uses fresh nominal simulator samples. By the assumed conditional critic guarantee,
\begin{equation}
\mathbb E\left[\|Q^{\pi_k}-\hat Q^{\pi_k}\|_\infty\mid\mathcal F_k\right]\leq \epsilon.
\end{equation}
Taking conditional expectation in \eqref{eq:actor-samplewise-recursion} and then total expectation gives
\begin{align}
\mathbb E[g^{\pi^\star}-g^{\pi_{k+1}}]
&\leq \left(1-\frac{1}{M}\right)\mathbb E[g^{\pi^\star}-g^{\pi_k}]
+\frac{2}{M\zeta_k}
+2\left(1+\frac{1}{M}\right)\epsilon .
\label{eq:actor-expected-recursion}
\end{align}
Unrolling \eqref{eq:actor-expected-recursion} gives
\begin{align}
\mathbb E[g^{\pi^\star}-g^{\pi_K}]
&\leq \left(1-\frac{1}{M}\right)^K\mathbb E[g^{\pi^\star}-g^{\pi_0}] \nonumber\\
&\quad+\sum_{k=0}^{K-1}\left(1-\frac{1}{M}\right)^{K-1-k}
\left[\frac{2}{M\zeta_k}+2\left(1+\frac{1}{M}\right)\epsilon\right].
\label{eq:actor-unrolled}
\end{align}
Since $\zeta_k\geq \zeta_{k-1}(1-1/M)^{-1}$, we have $\zeta_k\geq \zeta_0(1-1/M)^{-k}$ and hence $1/\zeta_k\leq (1/\zeta_0)(1-1/M)^k$. Substituting this into \eqref{eq:actor-unrolled} yields
\begin{align}
\mathbb E[g^{\pi^\star}-g^{\pi_K}]
&\leq \left(1-\frac{1}{M}\right)^K\mathbb E[g^{\pi^\star}-g^{\pi_0}]
+\frac{2K}{M\zeta_0}\left(1-\frac{1}{M}\right)^{K-1}
+2(M+1)\epsilon .
\label{eq:actor-final-bound}
\end{align}
Choosing the hidden constant in $K=\Theta(\log(1/\epsilon))$ large enough makes the first two terms in \eqref{eq:actor-final-bound} at most a constant multiple of $\epsilon$. Therefore $\mathbb E[g^{\pi^\star}-g^{\pi_K}]\leq \mathcal O(\epsilon)$.

\begin{remark}
\label{rem:fixed-zeta-appendix}
We note that the geometrically increasing actor parameter in Theorem~\ref{thm:pg-bound} is not essential for the stated sample complexity. Starting from the one-step recursion \eqref{eq:actor-expected-recursion}, one may instead choose a fixed actor parameter $\zeta_k\equiv \zeta$. If $\mathbb E[\|\hat Q^{\pi_k}-Q^{\pi_k}\|_\infty\mid\mathcal F_k]\leq \delta$ for every generated policy, then unrolling the recursion gives
\begin{equation}
\mathbb E[g^{\pi^\star}-g^{\pi_K}]
\leq
\left(1-\frac1M\right)^K
\mathbb E[g^{\pi^\star}-g^{\pi_0}]
+\frac2\zeta
+2(M+1)\delta .
\end{equation}
Thus, for a prescribed target accuracy $\epsilon$, taking $\zeta=\Theta(\epsilon^{-1})$, $\delta=\Theta(\epsilon)$, and $K=\Theta(\log(1/\epsilon))$ again yields $\mathbb E[g^{\pi^\star}-g^{\pi_K}]\leq \mathcal O(\epsilon)$.

Since Theorem~\ref{thm:Qestimation} achieves $\delta=\Theta(\epsilon)$ with $\widetilde O(|\mathcal S||\mathcal A|\epsilon^{-2})$ samples under contamination uncertainty and $\widetilde O((|\mathcal S||\mathcal A|)^2\epsilon^{-2})$ samples under TV or Wasserstein uncertainty, the same uncertainty-set-dependent sample complexities apply to the actor-critic method. Here the actor parameter is fixed across actor iterations but depends on the target accuracy; an accuracy-independent constant $\zeta=O(1)$ would leave an $O(1/\zeta)$ residual term under this proof.
\end{remark}

\subsection{Proof of Theorem \ref{thm:Qestimation}} \label{Qestimationproof}
We first show the contraction property of the Bellman operator for the policy evaluation in the following lemma, different from the semi-norm used in \cite{xu2025finite}, our setting requires the contraction to hold uniformly for all $\pi \in \Pi$. We first construct the family fluctuation matrices for each fixed $\pi$ as follows: 
\begin{equation} \label{eq:spectralboundpi}
    \mathcal{F}^\pi \coloneqq \{F^\pi_\kp = \kp^\pi - E_\kp : \kp \in \cp\}
\end{equation}
where $E_\kp$ and $\kp^\pi$ are defined in \eqref{eq:spectralbound}.

To establish uniform contraction for all $\pi\in\Pi$, we need to find the radius conditions such that $\sup_{\pi\in\Pi} \hat{\rho}(\mathcal{F}^\pi)<1$. We start by characterizing the ergodicity behavior of $\kp^\pi$ for all $\pi$:

\begin{lemma}\label{lem:uniform-m-a0}
Under Assumption \ref{ass:AC}, there exist a finite integer $\bar m\in\mathbb N$ and a constant $\bar a_0\in(0,1]$ such that, for all $\pi\in\Pi$,
\begin{equation}
(\tilde \kp^\pi)^{\bar m} > 0 \quad\text{(for each entry)},\qquad
\min_{i<j}\ \sum_{s\in\mathcal S}\min\Big\{(\tilde \kp^\pi)^{\bar m}_{is},(\tilde \kp^\pi)^{\bar m}_{js}\Big\}\geq \bar a_0.
\end{equation}
Equivalently, $\tau\big((\tilde \kp^\pi)^{\bar m}\big)\le 1-\bar a_0$ for all $\pi\in\Pi$, where $\tau$ is the Dobrushin’s coefficient defined in \eqref{eq:Dobrushindef}.

\begin{proof}
For each fixed $\pi\in\Pi$, since $\tilde \kp^\pi$ is irreducible and aperiodic, there exists an integer $m_\pi\in\mathbb N$ with $(\tilde \kp^\pi)^{m_\pi}>0$ for each entry.
Define $\varepsilon_\pi(\pi') =\min_{i,j}(\tilde\kp^{\pi'})^{m_\pi}_{ij}$. Since the map $\pi'\mapsto (\tilde\kp^{\pi'})^{m_\pi}$ is continuous due to each entry being a polynomial in the entries of $\tilde\kp^{\pi'}$, $\varepsilon_\pi$ is continuous; denote $\varepsilon_\pi(\pi)=:\eta_\pi>0$. Therefore, there exists an open neighborhood $U_\pi\subset\Pi$ such that for each $s,s'\in\mcs$,
\begin{equation}
(\tilde \kp^{\pi'})^{m_\pi}(s,s') \geq \tfrac{\eta_\pi}{2}>0\qquad\forall\,\pi'\in U_\pi.
\end{equation}
For each simplex $\Delta_{|\mathcal A|}=\{x\in\mathbb{R}^{|\mathcal A|}: x\ge 0,\ \mathbf 1^\top x=1\}$, it is closed and bounded. Thus, $\Pi$ is compact and the open cover $\{U_\pi:\pi\in\Pi\}$ admits a finite subcover. We can accordingly
choose $\pi^1,\ldots,\pi^K$ with $\Pi\subset \bigcup_{k=1}^K U_{\pi^k}$.
Define the common exponent $\bar m = \max_{1\le k\le K} m_{\pi^k} <\infty.$ Then for any $\pi\in\Pi$, pick $k$ such that $\pi\in U_{\pi^k}$. Then $(\tilde \kp^\pi)^{m_{\pi^k}}\ge \eta_{\pi^k}/2 >0$ entrywise. Since all later powers of a stochastic matrix with one positive power remain positive, we thus obtain $(\tilde \kp^\pi)^{\bar m}>0$ entrywise.

Define the continuous function
\begin{equation}
f(\pi) \coloneqq \min_{i<j}\ \sum_{s\in\mathcal S}\min\Big\{(\tilde \kp^\pi)^{\bar m}_{is},(\tilde \kp^\pi)^{\bar m}_{js}\Big\}.
\end{equation}
We then have $f(\pi)>0$ for every $\pi\in\Pi$. Since $f$ is continuous and $\Pi$ is compact, the minimum $\bar a_0\coloneqq \min_{\pi\in\Pi} f(\pi)$ is attained and satisfies $\bar a_0>0$. The stated Dobrushin bound follows from the identity $\tau(P)=1-\min_{i<j}\sum_{s}\min\{P_{is},P_{js}\}$.
\end{proof}
\end{lemma}

Equipped with Lemma \ref{lem:uniform-m-a0}, we are now ready to quantify the radius restrictions in the critic settings as follows:

\begin{corollary}[Uniform $\bar m$-step contraction for the critic]
\label{cor:AC-radius-to-JSR-primitive}
Let Assumption~\ref{ass:AC} hold and let $\bar m,\bar a_0$ be given by Lemma~\ref{lem:uniform-m-a0}. Define also the uniform $\bar m$-step entrywise lower bound at the center,
\begin{equation}
\bar b_0\coloneqq \min_{\pi\in\Pi}\min_{i,s\in\mathcal S}\big((\tilde \kp^\pi)^{\bar m}\big)_{is}\in(0,1].
\end{equation}
Fix any policy $\pi\in\Pi$ and the fluctuation family $\mathcal F^\pi=\{F^\pi_\kp=\kp^\pi-E^\pi_\kp:\kp\in\mathcal P\}$. Under the following radius conditions on $\mathcal P$, there exists $\beta\in(0,1)$, independent of $\pi$, such that for every $\pi\in\Pi$ and every $\kp_1,\ldots,\kp_{\bar m}\in\cp$,
\begin{equation}\label{eq:mstep-gap}
\tau\big(\kp_{\bar m}^\pi\kp_{\bar m-1}^\pi\cdots\kp_1^\pi\big)\leq 1-\beta.
\end{equation}
Consequently,
\begin{equation}\label{eq:JSR-uniform}
\sup_{\pi\in\Pi}\hat\rho(\mathcal F^\pi)\leq (1-\beta)^{1/\bar m}<1.
\end{equation}
Moreover, for each admissible $\kp\in\cp$ and $\pi\in\Pi$, the matrix $\kp^\pi$ is irreducible and aperiodic. The radius and the resulting $\beta$ are:
\begin{itemize}\itemsep2pt
\item \textbf{Contamination uncertainty:} Let $\delta\in[0,1)$, then \eqref{eq:mstep-gap} holds with $\beta=(1-\delta)^{\bar m}\bar a_0$.
\item \textbf{TV distance uncertainty:} Let $\delta\in[0,\min\{\bar a_0/(2\bar m),\bar b_0/(2\bar m)\})$, then \eqref{eq:mstep-gap} holds with $\beta=\bar a_0-2\bar m\delta$.
\item \textbf{Wasserstein-$p$ uncertainty set:} Let $\delta\in[0,\min\{d_{\min}\bar a_0/(2\bar m),d_{\min}\bar b_0/(2\bar m)\})$, where $d_{\min}:=\min_{x\neq y}d(x,y)>0$. Then \eqref{eq:mstep-gap} holds with $\beta=\bar a_0-2\bar m\delta/d_{\min}$.
\end{itemize}

\begin{proof}
Fix $\pi\in\Pi$. For notational simplicity write $M_i=\kp_i^\pi$ and $\widetilde M=\tilde\kp^\pi$. We first prove the $\bar m$-step product bound in \eqref{eq:mstep-gap}.

For contamination uncertainty, each $M_i$ can be written as $M_i=(1-\delta)\widetilde M+\delta Q_i$, where $Q_i$ is row-stochastic. Therefore, by nonnegativity,
\begin{equation}
M_{\bar m}M_{\bar m-1}\cdots M_1\geq (1-\delta)^{\bar m}\widetilde M^{\bar m}
\end{equation}
entrywise. Hence, for any two rows $r,r'$,
\begin{equation}
\sum_{s\in\mcs}\min\{(M_{\bar m}\cdots M_1)_{rs},(M_{\bar m}\cdots M_1)_{r's}\}\geq (1-\delta)^{\bar m}\bar a_0.
\end{equation}
This proves \eqref{eq:mstep-gap} with $\beta=(1-\delta)^{\bar m}\bar a_0$.

For TV uncertainty, convexity of TV gives $TV(M_i(r,\cdot),\widetilde M(r,\cdot))\leq\delta$ for every row $r$ and every $i$. By a standard telescoping argument and the fact that TV distance is non-expansive under multiplication by a stochastic matrix, for each $1\leq\ell\leq\bar m$,
\begin{equation}
\sup_r TV\big((M_\ell\cdots M_1)(r,\cdot),\widetilde M^\ell(r,\cdot)\big)\leq \ell\delta.
\end{equation}
Using the identity $\sum_s\min\{p_s,q_s\}=1-TV(p,q)$ and the triangle inequality for TV, for any two rows $r,r'$,
\begin{align}
&\sum_{s\in\mcs}\min\{(M_{\bar m}\cdots M_1)_{rs},(M_{\bar m}\cdots M_1)_{r's}\} \nonumber\\
&\quad\geq \sum_{s\in\mcs}\min\{(\widetilde M^{\bar m})_{rs},(\widetilde M^{\bar m})_{r's}\}-2\bar m\delta \nonumber\\
&\quad\geq \bar a_0-2\bar m\delta.
\end{align}
Thus \eqref{eq:mstep-gap} holds with $\beta=\bar a_0-2\bar m\delta$ when $\delta<\bar a_0/(2\bar m)$. The same telescoping bound also gives, for every $r,s$,
\begin{equation}
(M_{\bar m}\cdots M_1)_{rs}\geq (\widetilde M^{\bar m})_{rs}-2\bar m\delta\geq \bar b_0-2\bar m\delta>0
\end{equation}
when $\delta<\bar b_0/(2\bar m)$. Applying this last display with the constant sequence $M_1=\cdots=M_{\bar m}=\kp^\pi$ shows that $(\kp^\pi)^{\bar m}>0$ for every $\kp\in\cp$ and $\pi\in\Pi$.

For Wasserstein uncertainty, the finite-state metric satisfies $TV(u,v)\leq W_p(u,v)/d_{\min}$ for all distributions $u,v$ on $\mcs$. Hence the Wasserstein case reduces to the TV case with $\delta$ replaced by $\delta/d_{\min}$, giving the stated radius condition and $\beta=\bar a_0-2\bar m\delta/d_{\min}$.

Finally, Lemma~\ref{lem:xulemA.3} applied to the family $\{\kp^\pi:\kp\in\cp\}$ gives, for each fixed $\pi$,
\begin{equation}
\hat\rho(\mathcal F^\pi)\leq \left(\sup_{\kp_1,\ldots,\kp_{\bar m}\in\cp}\tau(\kp_{\bar m}^\pi\cdots\kp_1^\pi)\right)^{1/\bar m}\leq (1-\beta)^{1/\bar m}.
\end{equation}
Taking the supremum over $\pi\in\Pi$ proves \eqref{eq:JSR-uniform}. The entrywise positivity above implies irreducibility and aperiodicity of each induced kernel.
\end{proof}
\end{corollary}

\begin{lemma}\label{lem:uniform-stationary-lower-bound}
Under Assumption~\ref{ass:AC} and the radius restrictions in Corollary~\ref{cor:AC-radius-to-JSR-primitive}, there exists $b_\star>0$ such that for every policy $\pi\in\Pi$, every admissible kernel $\kp\in\cp$, and every state $s\in\mcs$,
\begin{equation}
d^\pi_\kp(s)\geq b_\star .
\end{equation}
Consequently, the mismatch constant $M$ in \eqref{eq:distribution-mismatch-M} is finite.
\begin{proof}
By Corollary~\ref{cor:AC-radius-to-JSR-primitive}, there are $\bar m\in\mathbb N$ and $b>0$, independent of $\pi$ and $\kp$, such that
\begin{equation}
(\kp^\pi)^{\bar m}(i,j)\geq b,
\qquad \forall i,j\in\mcs .
\end{equation}
For contamination uncertainty, one can take $b=(1-\delta)^{\bar m}\bar b_0$. For TV uncertainty, one can take $b=\bar b_0-2\bar m\delta$. For Wasserstein uncertainty, one can take $b=\bar b_0-2\bar m\delta/d_{\min}$.
Let $d^\pi_\kp$ be the stationary distribution of $\kp^\pi$. Then, for any $j\in\mcs$,
\begin{align}
d^\pi_\kp(j)
&=\sum_i d^\pi_\kp(i)(\kp^\pi)^{\bar m}(i,j) \nonumber\\
&\geq \sum_i d^\pi_\kp(i)b=b .
\end{align}
Thus $b_\star=b$ works. Since every numerator in \eqref{eq:distribution-mismatch-M} is at most $1$ and every denominator is at least $b_\star$, we have $M\leq 1/b_\star<\infty$.
\end{proof}
\end{lemma}

\begin{lemma} \label{lem:robust_seminorm-contraction}
Under Assumption \ref{ass:AC}, and the radius conditions in Corollary~\ref{cor:AC-radius-to-JSR-primitive}, for any fixed policy $\pi\in\Pi$, there exists a corresponding semi-norm $\|\cdot\|_{\pi,\cp}$ with kernel $\{c \e : c \in \mathbb{R}\}$ such that the robust Bellman operator $\mathbf{T}^\pi_g$ under $\pi$ is a one-step contraction with a uniform contraction factor. Specifically, there exist $\gamma\in(0,1)$ independent of $\pi$ such that
\begin{equation} \label{eq:contractiongamma}
\|\mathbf{T}^\pi_g(V_1) - \mathbf{T}^\pi_g(V_2)\|_{\pi,\cp}
\leq\gamma \|V_1 - V_2\|_{\pi,\cp},
\forall \;  V_1,V_2\in\mathbb R^{ S},\,g\in\mathbb R.
\end{equation}
\begin{proof}
By Corollary~\ref{cor:AC-radius-to-JSR-primitive}, for every $\pi$ the fluctuation family $\mathcal F^\pi=\{F_\kp^\pi=\kp^\pi-E_\kp^\pi: \kp\in\mathcal P\}$ has joint spectral radius
\begin{equation}
\widehat\rho(\mathcal F^\pi)\leq \gamma_* \coloneqq (1-\beta)^{1/\bar m}<1,
\end{equation}
where $\gamma_*$ is independent of $\pi$. Choose $\alpha$ such that $\gamma_*<\alpha<1$. Define
\begin{equation} \label{eq:extremalnorm_pi_def}
    \|x\|_{\pi,\rm ext} \coloneqq \sup_{k \geq 0} \sup_{F_1,\ldots, F_k \in \mathcal{F}^\pi}  \alpha^{-k} \|F_k F_{k-1} \ldots  F_1 x \|_2.
\end{equation}
As in Lemma~\ref{lem:extremalnorm}, this is a norm and it satisfies
\begin{equation}
\|Fx\|_{\pi,\rm ext}\leq \alpha\|x\|_{\pi,\rm ext},
\qquad \forall F\in\mathcal F^\pi .
\label{eq:pi-ext-contraction}
\end{equation}
Choose $\xi\in(0,1-\alpha)$ and define
\begin{equation} \label{eq:robustseminorm_pi_construction}
    \|x\|_{\pi,\cp} \coloneqq \sup_{F\in\mathcal{F}^\pi}\bigl\|F x\bigr\|_{\pi,\rm ext}+\xi \inf_{c\in\mathbb R}\bigl\|x - c \e\bigr\|_{\pi,\rm ext}.
\end{equation}
This is a semi-norm with kernel $\{c\e:c\in\mathbb R\}$.

Fix $V_1,V_2\in\mathbb R^S$ and set $\Delta=V_1-V_2$. The scalar $g$ cancels in the difference $\mathbf T_g^\pi(V_1)-\mathbf T_g^\pi(V_2)$. For every $(s,a)$, the elementary inequality for differences of minima gives
\begin{equation}
\min_{p\in\cp_s^a}p^\top\Delta
\leq
\sigma_{\cp_s^a}(V_1)-\sigma_{\cp_s^a}(V_2)
\leq
\max_{p\in\cp_s^a}p^\top\Delta .
\label{eq:sigma-difference-sandwich}
\end{equation}
Since $\cp_s^a$ is compact and convex, and since $p\mapsto p^\top\Delta$ is linear, there exists $p_{s,a}\in\cp_s^a$ such that
\begin{equation}
p_{s,a}^\top\Delta=\sigma_{\cp_s^a}(V_1)-\sigma_{\cp_s^a}(V_2).
\label{eq:interpolated-row-policy-eval}
\end{equation}
Indeed, choose minimizer and maximizer rows in \eqref{eq:sigma-difference-sandwich}, and take the convex combination that produces the middle value. By rectangularity, the rows $p_{s,a}$ can be assembled into a kernel $\kp_\Delta\in\cp$. Therefore,
\begin{align}
\mathbf T_g^\pi(V_1)(s)-\mathbf T_g^\pi(V_2)(s)
&=\sum_a\pi(a|s)\left(\sigma_{\cp_s^a}(V_1)-\sigma_{\cp_s^a}(V_2)\right) \nonumber\\
&=\sum_a\pi(a|s)p_{s,a}^\top\Delta
=(\kp_\Delta^\pi\Delta)(s).
\label{eq:policy-eval-induced-kernel-representation}
\end{align}
Let $E_\Delta$ be the rank-one matrix whose rows equal the stationary distribution of $\kp_\Delta^\pi$, and let $F_\Delta=\kp_\Delta^\pi-E_\Delta\in\mathcal F^\pi$. Since $E_\Delta\Delta$ is a constant vector and every $F\in\mathcal F^\pi$ annihilates constant vectors,
\begin{align}
\sup_{F\in\mathcal F^\pi}\|F\kp_\Delta^\pi\Delta\|_{\pi,\rm ext}
&=\sup_{F\in\mathcal F^\pi}\|F F_\Delta\Delta\|_{\pi,\rm ext} \nonumber\\
&\leq \alpha\sup_{F\in\mathcal F^\pi}\|F\Delta\|_{\pi,\rm ext}.
\label{eq:first-term-policy-contraction}
\end{align}
For the quotient term, choose the constant vector $E_\Delta\Delta$. Then
\begin{align}
\inf_{c\in\mathbb R}\|\kp_\Delta^\pi\Delta-c\e\|_{\pi,\rm ext}
&\leq \|\kp_\Delta^\pi\Delta-E_\Delta\Delta\|_{\pi,\rm ext} \nonumber\\
&=\|F_\Delta\Delta\|_{\pi,\rm ext}
\leq \sup_{F\in\mathcal F^\pi}\|F\Delta\|_{\pi,\rm ext}.
\label{eq:second-term-policy-contraction}
\end{align}
Combining \eqref{eq:first-term-policy-contraction} and \eqref{eq:second-term-policy-contraction} with the definition \eqref{eq:robustseminorm_pi_construction} gives
\begin{align}
\|\mathbf T_g^\pi(V_1)-\mathbf T_g^\pi(V_2)\|_{\pi,\cp}
&=\|\kp_\Delta^\pi\Delta\|_{\pi,\cp} \nonumber\\
&\leq (\alpha+\xi)\sup_{F\in\mathcal F^\pi}\|F\Delta\|_{\pi,\rm ext} \nonumber\\
&\leq (\alpha+\xi)\|\Delta\|_{\pi,\cp}.
\end{align}
Thus \eqref{eq:contractiongamma} holds with $\gamma=\alpha+\xi<1$, and this $\gamma$ is independent of $\pi$.
\end{proof}
\end{lemma}

Unlike the policy-evaluation settings in \citep{xu2025finite}, the actor–critic procedure repeatedly re-solves the critic at changing policies. Hence all constants used in the finite-sample analysis must be independent of the $\pi$. Lemma \ref{lem:uniform-m-a0}-\ref{lem:robust_seminorm-contraction} provide uniform $\bar m\in\mathbb N$ and $\bar a_0,\bar b_0>0$ such there exist a semi-norm $\|\cdot\|_{\pi,\cp}$ for each $\pi$ in which the robust Bellman operator contracts with the same factor $\gamma_*$. In order to carry out the remaining analysis of the critic, we now provide the uniform sample complexity bound (independent of the policy) for estimating the robust value function $V^\pi_\cp$ and robust average reward $g^\pi_\cp$ for any $\pi$.

\begin{lemma}\label{lem:uniform-ext-envelope}
Under the norm $\|\cdot\|_{\pi, \rm ext}$ constructed in \eqref{eq:extremalnorm_pi_def}, there exists a constant $\bar B<\infty$, independent of $\pi$, such that for all $\pi\in\Pi, x\in\mathbb R^{S}$
\begin{equation}
\|x\|_{\pi,\rm ext}\ \le\ \bar B\|x\|_2.
\end{equation}

\begin{proof}
Fix $\pi\in\Pi$ and a product $F_k\cdots F_1$ with $F_i=F_{\kp_i}^\pi\in\mathcal F^\pi$. Write $M_i=\kp_i^\pi$. Since $F_i\e=0$ and $E_{\kp_i}^\pi z$ is a constant vector for every $z$, an induction gives
\begin{equation}
F_k\cdots F_1x=M_k\cdots M_1x-c_k\e
\end{equation}
for some scalar $c_k$. Hence $\|F_k\cdots F_1x\|_{\mathrm{sp}}=\|M_k\cdots M_1x\|_{\mathrm{sp}}$. By Corollary~\ref{cor:AC-radius-to-JSR-primitive}, every product of $\bar m$ admissible induced kernels has Dobrushin coefficient at most $1-\beta$. Since Dobrushin coefficients are submultiplicative and at most one, if $k=q\bar m+r$ with $0\leq r<\bar m$, then
\begin{equation}
\|M_k\cdots M_1x\|_{\mathrm{sp}}\leq (1-\beta)^q\|x\|_{\mathrm{sp}}.
\end{equation}
For $k\geq1$, the vector $F_k\cdots F_1x$ has zero mean under the stationary distribution of $M_k$, which has strictly positive entries by Corollary~\ref{cor:AC-radius-to-JSR-primitive}. Therefore $0$ lies between its minimum and maximum, so $\|F_k\cdots F_1x\|_2\leq \sqrt S\|F_k\cdots F_1x\|_{\mathrm{sp}}$. Using $\|x\|_{\mathrm{sp}}\leq 2\|x\|_2$, for $k\geq1$ we obtain
\begin{equation}
\|F_k\cdots F_1x\|_2\leq 2\sqrt S(1-\beta)^q\|x\|_2.
\end{equation}
Since $\alpha>(1-\beta)^{1/\bar m}$, we have $(1-\beta)/\alpha^{\bar m}<1$. Thus, for $k=q\bar m+r\geq1$,
\begin{equation}
\alpha^{-k}\|F_k\cdots F_1x\|_2\leq 2\sqrt S\alpha^{-r}\left(\frac{1-\beta}{\alpha^{\bar m}}\right)^q\|x\|_2\leq 2\sqrt S\alpha^{-(\bar m-1)}\|x\|_2.
\end{equation}
For $k=0$, the term in \eqref{eq:extremalnorm_pi_def} equals $\|x\|_2$. Taking the supremum over $k$ and over all admissible products gives
\begin{equation}
\|x\|_{\pi,\rm ext}\leq \max\{1,2\sqrt S\alpha^{-(\bar m-1)}\}\|x\|_2.
\end{equation}
The constant $\bar B:=\max\{1,2\sqrt S\alpha^{-(\bar m-1)}\}$ is finite and independent of $\pi$.
\end{proof}

\end{lemma}

\begin{lemma}\label{lem:uniform-bridge-linf}
Let $\bar B<\infty$ be the policy-uniform constant from Lemma~\ref{lem:uniform-ext-envelope}, i.e.
$\|x\|_{\pi,\rm ext}\leq \bar B\|x\|_2$ for all $\pi,x$. Then, for all $\pi\in\Pi$ and $x\in\mathbb R^{\mathcal S}$,
\begin{equation}\label{eq:bridge-centered-linf}
\xi \|x\|_{\infty,0}
\leq
\|x\|_{\pi,\mathcal P}
\leq
(\alpha+\xi)\bar B\sqrt{S}\|x\|_{\infty,0},
\end{equation}
where $\|x\|_{\infty,0}:=\min_{c\in\mathbb R}\|x-c\mathbf e\|_\infty$ is the centered $\ell_\infty$ semi-norm.
Equivalently,
\begin{equation}\label{eq:bridge-span}
\frac{\xi}{2}\|x\|_{\mathrm{sp}}
\leq
\|x\|_{\pi,\mathcal P}
\leq
\frac{(\alpha+\xi)\bar B\sqrt{S}}{2}\|x\|_{\mathrm{sp}}.
\end{equation}
Here $\xi$ and $\alpha$ are chosen in Lemma~\ref{lem:robust_seminorm-contraction} and are independent of $\pi$.
\end{lemma}

\begin{proof}
Let $c_\infty\in\arg\min_{c\in\mathbb R}\|x-c\mathbf e\|_\infty$ and write $z:=x-c_\infty\mathbf e$. Since $F\mathbf e=\mathbf 0$ for all $F\in\mathcal F^\pi$, we have $Fx=Fz$. By the definition of $\|x\|_{\pi,\mathcal P}$ in \eqref{eq:robustseminorm_pi_construction}, the extremal norm contraction, and the choice of $c_\infty$,
\begin{align}
\|x\|_{\pi,\mathcal P}
&\leq \sup_{F\in\mathcal F^\pi}\|Fz\|_{\pi,\rm ext}+\xi\|z\|_{\pi,\rm ext} \nonumber\\
&\leq (\alpha+\xi)\|z\|_{\pi,\rm ext} \nonumber\\
&\leq (\alpha+\xi)\bar B\|z\|_2 \nonumber\\
&\leq (\alpha+\xi)\bar B\sqrt S\|z\|_\infty \nonumber\\
&= (\alpha+\xi)\bar B\sqrt S\|x\|_{\infty,0}.
\end{align}
This proves the upper bound in \eqref{eq:bridge-centered-linf}. For the lower bound, drop the first term in \eqref{eq:robustseminorm_pi_construction} and use the $k=0$ term in $\|\cdot\|_{\pi,\rm ext}$:
\begin{align}
\|x\|_{\pi,\mathcal P}
&\geq \xi\inf_{c\in\mathbb R}\|x-c\mathbf e\|_{\pi,\rm ext} \nonumber\\
&\geq \xi\inf_{c\in\mathbb R}\|x-c\mathbf e\|_2 \nonumber\\
&\geq \xi\inf_{c\in\mathbb R}\|x-c\mathbf e\|_\infty \nonumber\\
&= \xi\|x\|_{\infty,0}.
\end{align}
Finally, $\|x\|_{\infty,0}=\tfrac12(\max_i x_i-\min_i x_i)=\tfrac12\|x\|_{\mathrm{sp}}$ yields \eqref{eq:bridge-span}.
\end{proof}

\begin{remark}\label{rem:uniform-ergodicity}
By the above lemmas, the sampling chains are geometrically ergodic with a {policy-uniform} mixing time
\begin{equation}\label{eq:bar-tmix}
\bar t_{\text{mix}} := \bar m \left\lceil \frac{\log 4}{-\log(1-\beta)}\right\rceil.
\end{equation}
and stationary distributions admit a policy-uniform lower bound. Consequently, the semi--Lyapunov constants (Section B.1.2 in \citep{xu2025finite}) in the policy evaluation analysis now admit {policy-uniform} versions. As in Appendix~\ref{appendix4biasedSA}, choose the Moreau smoothing parameter in this policy-evaluation Lyapunov construction sufficiently small, using the uniform contraction factor $\gamma<1$, so that the resulting policy-uniform constants satisfy
\begin{equation}
\label{eq:bar-alpha2-positive-choice}
\bar\vartheta\coloneqq \gamma\sqrt{\bar c_u/\bar c_\ell}
\leq
\frac{1+\gamma}{2}<1,
\qquad
\bar\alpha_2\coloneqq \frac{1-\bar\vartheta}{2}\geq \frac{1-\gamma}{4}>0.
\end{equation}
This is again only a choice of the smooth Lyapunov function. With this choice, there exist $\bar c_\ell>0$ and $\bar c_u<\infty$ such that the inequalities of the form
$
\bar c_\ell\,M_{\overline{E}}(x)\ \le\ \tfrac12\|x\|^2_{\mathcal N,\overline{E}}\ \le\ \bar c_u\,M_{\overline{E}}(x)
$
hold uniformly over $\pi$ (notations borrowed from Section B.1.2 of \cite{xu2025finite}). In addition, by setting $ \bar c_{\mathcal P}:=\xi $ and $
\bar C_{\mathcal P}:=(\alpha+\xi)\bar B\,\sqrt{|\mathcal S|}$, the coefficients $ \bar c_{\mathcal P} $ and $ \bar C_{\mathcal P}$ also recover the coefficients  $  c_{\mathcal P} $ and $  C_{\mathcal P}$ from Section B.1.2 of \cite{xu2025finite} uniformly independent of $\pi$.
\end{remark}

We are now able to present the formal results for the critic estimation bounds.

\begin{theorem}\label{thm:uniform-V}
Let $\pi \in \Pi$ and $\kp \in \cp$. Consider the iterates $(V^\pi_t,g^\pi_t)$ from Algorithm~\ref{alg:RobustTD} with stepsizes $\beta_t := \frac{1}{t+1}$ and $\eta_t := \frac{2}{\bar\alpha_2(t+K)}$. Under Assumption~\ref{ass:AC} and the radii of Corollary~\ref{cor:AC-radius-to-JSR-primitive}, define the policy-uniform constants as follows.
Let $\gamma := \alpha+\xi$ be the contraction factor from Lemma~\ref{lem:robust_seminorm-contraction}. Let $\bar t_{\mathrm{mix}}$ be as in \eqref{eq:bar-tmix}. Set
\begin{equation}
\bar c_{\mathcal P}:=\xi,\qquad
\bar C_{\mathcal P}:=(\alpha+\xi)\bar B\,\sqrt{|\mathcal S|}.
\end{equation}
where $\bar B$ is the policy-uniform constant in Lemma~\ref{lem:uniform-ext-envelope}. Let $\bar c_\ell$ and $\bar c_u$ denote the policy-uniform semi--Lyapunov constants from Remark~\ref{rem:uniform-ergodicity} (notations borrowed from Section B.1.2 of \cite{xu2025finite}). Let $\bar\rho_2,\bar L,\bar G$ denote the corresponding policy-uniform constants in Section B.1.2 of \cite{xu2025finite}. Define the following
\begin{align}
\bar\vartheta &:= \gamma\sqrt{\bar c_u/\bar c_\ell}, &
\bar\alpha_2 &:= \frac{1-\bar\vartheta}{2}, &
\bar\alpha_4 &:= \bar\rho_2\,\bar L, \nonumber\\
\bar\alpha_3 &:= 8\,\bar c_u\,\bar\alpha_4, &
K &:= \max\{4\bar\alpha_3/\bar\alpha_2^2,3\}, &
\bar C_2 &:= \frac{1}{K}+\log\Big(\frac{T-1+K}{K}\Big), \nonumber\\
\bar C_3 &:= \bar G\big(1+8\bar C_{\mathcal P}\bar t_{\mathrm{mix}}\big). \label{eq:uniform-constants}
\end{align}
By Remark~\ref{rem:uniform-ergodicity}, the Moreau smoothing parameter is chosen so that $\bar\vartheta\leq (1+\gamma)/2$, hence $\bar\alpha_2\geq (1-\gamma)/4>0$.
Finally, let $(\bar\varepsilon_{\mathrm{bias}},\bar A)$ be the (policy-uniform) bias and second-moment constants of the support-function estimator from Algorithm~\ref{alg:sampling} (cf.\ Lemma D.1 of \cite{xu2025finite}), i.e., under contamination uncertainty,

\begin{equation}
    (\bar A,\bar\varepsilon_{\mathrm{bias}})=\big(32\bar C_{\mathcal P}^2\bar t_{\mathrm{mix}}^2,\ 0\big),
\end{equation}

under TV uncertainty,
\begin{equation}
    \bar A=2\bar C_{\mathcal P}^2\big(24(1+\tfrac1\delta)\sqrt{|\mathcal S|\,2^{-N_{\max}}}\,\bar t_{\mathrm{mix}}\big)^2 + 96\bar C_{\mathcal P}^2\bar t_{\mathrm{mix}}^2 + 4608\bar C_{\mathcal P}^2(1+\tfrac1\delta)^2|\mathcal S|\,\bar t_{\mathrm{mix}}^2N_{\max},
\end{equation}
\begin{equation}
    \bar\varepsilon_{\mathrm{bias}}=
48\bar C_{\mathcal P}(1+\tfrac1\delta)\sqrt{|\mathcal S|\,2^{-N_{\max}}}\,\bar t_{\mathrm{mix}},
\end{equation}

and under Wasserstein uncertainty, 

\begin{equation}
    (\bar A,\bar\varepsilon_{\mathrm{bias}})=\Big(2\bar C_{\mathcal P}^2\big(24\sqrt{|\mathcal S|\,2^{-N_{\max}}}\,\bar t_{\mathrm{mix}}\big)^2 + 96\bar C_{\mathcal P}^2\bar t_{\mathrm{mix}}^2 + 4608\bar C_{\mathcal P}^2|\mathcal S|\,\bar t_{\mathrm{mix}}^2N_{\max},\ 
48\bar C_{\mathcal P}\sqrt{|\mathcal S|\,2^{-N_{\max}}}\,\bar t_{\mathrm{mix}}\Big).
\end{equation}

Then, for all $T\ge1$,
\begin{align*}
\mathbb{E}\!\left[\|V_T^\pi - V^\pi_\cp\|_{\infty,0}^2\right]
&\le
\frac{4K^2\bar c_u\bar C_{\mathcal P}^2}{(T+K)^2\bar c_\ell \bar c_{\mathcal P}^2}\|V_0^\pi-V^\pi_\cp\|_{\infty,0}^2 +\frac{8 \bar A \bar\alpha_4 \bar c_u}{(T+K)\bar\alpha_2^2 \bar c_{\mathcal P}^2}+ \frac{\bar c_u \bar C_3 \bar C_2 \bar\varepsilon_{\mathrm{bias}}}{\bar\alpha_2 \bar c_{\mathcal P}^2},\\
\mathbb{E}\!\left[ |g_T^\pi - g^\pi_\cp|^2 \right]
&\le
\frac{\bar B_1}{T+K}+\frac{\bar B_2 \log^2(T)}{(T+K)^2}+\frac{\bar B_3\bar t_{\mathrm{mix}}^{2} N_{\max}\log^2(T)}{(T+K)(1-\gamma)^2}+\frac{\bar B_4\bar t_{\mathrm{mix}}^{2}2^{-N_{\max}/2}\log^3(T)}{1-\gamma}+\bar B_5 \bar t_{\mathrm{mix}}^{2}2^{-N_{\max}}\log^2(T),
\end{align*}
for some absolute constants $\bar B_1,\bar B_2,\bar B_3,\bar B_4,\bar B_5>0$ depending only on the policy-uniform quantities above.
\end{theorem}

\begin{proof}
The proof is identical to the robust value bound and robust average reward bound in the policy evaluation analysis (Theorem D.2 and D.3 of \cite{xu2025finite}): use Lemma~\ref{lem:robust_seminorm-contraction} for the one-step contraction in $\|\cdot\|_{\pi,\mathcal P}$ with the uniform modulus $\gamma=\alpha+\xi$; translate all norm occurrences to the centered $\ell_\infty$ semi-norm using Lemma~\ref{lem:uniform-bridge-linf} and Remark~\ref{rem:uniform-ergodicity} (uniform $\bar c_{\mathcal P},\bar C_{\mathcal P}$); and replace $(c_\ell,c_u,t_{\mathrm{mix}})$ by $(\bar c_\ell,\bar c_u,\bar t_{\mathrm{mix}})$ per Remark~\ref{rem:uniform-ergodicity}.
\end{proof}

We now start the formal proof of Theorem \ref{thm:Qestimation}. Define the error decomposition
\begin{equation}
\|\hat Q^\pi - Q^\pi\|_\infty \leq |g_T^\pi - g^\pi_\cp|+\max_{s,a}\bigl|\bar\sigma_{\cp_s^a}(V_T^\pi) - \sigma_{\cp_s^a}(V^\pi_\cp)\bigr|.
\end{equation}
Squaring and taking expectation gives
\begin{equation} \label{eq:Qestimationerror}
\E\bigl[\|\hat Q^\pi - Q^\pi\|_\infty^2\bigr]
\leq
2\E\bigl[|g_T^\pi-g^\pi_\cp|^2\bigr]
+
2\E\Bigl[\max_{s,a}\bigl(\bar\sigma_{\cp_s^a}(V_T^\pi) - \sigma_{\cp_s^a}(V^\pi_\cp)\bigr)^2\Bigr].
\end{equation}
By Theorem~\ref{thm:uniform-V} (the second inequality), we have the following direct second-moment bound:
\begin{align}
\mathbb{E}\!\left[ \left| g_T^\pi - g^\pi_{\mathcal P} \right|^2 \right]
\le\;&
\frac{\bar B_1}{T+K}
+\frac{\bar B_2 \log^2(T+K)}{(T+K)^2}
+\frac{\bar B_3 \,\bar t_{\mathrm{mix}}^{2} N_{\max}\log^2(T+K)}{(T+K)(1-\gamma)^2} \nonumber\\
&\quad+
\frac{\bar B_4 \,\bar t_{\mathrm{mix}}^{2}2^{-N_{\max}/2}\log^3(T+K)}{1-\gamma}
+\bar B_5 \,\bar t_{\mathrm{mix}}^{2}2^{-N_{\max}}\log^2(T+K),
\label{eq:g-second-moment-AC}
\end{align}
where $\bar B_1,\dots,\bar B_5>0$ are absolute constants independent of $\pi$.

Regarding the second term on the RHS of \eqref{eq:Qestimationerror}, write
\begin{equation}
\bar\sigma_{\cp_s^a}(V_T^\pi)
=
\frac1{L_Q}
\sum_{\ell=1}^{L_Q}
\hat\sigma^{(\ell)}_{\cp_s^a}(V_T^\pi).
\end{equation}
For any $(s,a)$,
\begin{align}
\left|
\bar\sigma_{\cp_s^a}(V_T^\pi)
-
\sigma_{\cp_s^a}(V^\pi_\cp)
\right|^2
&\leq
2\left|
\bar\sigma_{\cp_s^a}(V_T^\pi)
-
\sigma_{\cp_s^a}(V_T^\pi)
\right|^2  \nonumber\\
&\quad+
2\left|
\sigma_{\cp_s^a}(V_T^\pi)
-
\sigma_{\cp_s^a}(V^\pi_\cp)
\right|^2 .
\end{align}
The support map is $1$-Lipschitz in $\|\cdot\|_\infty$, hence
\begin{equation}
\left|
\sigma_{\cp_s^a}(V_T^\pi)
-
\sigma_{\cp_s^a}(V^\pi_\cp)
\right|
\leq
\|V_T^\pi-V^\pi_\cp\|_\infty .
\end{equation}
Both $V_T^\pi$ and $V^\pi_\cp$ are anchored at $s_0$, so $\|V_T^\pi-V^\pi_\cp\|_\infty\leq \|V_T^\pi-V^\pi_\cp\|_{\mathrm{sp}}=2\|V_T^\pi-V^\pi_\cp\|_{\infty,0}$.
Conditional on $V_T^\pi$, the samples $\{\hat\sigma^{(\ell)}_{\cp_s^a}(V_T^\pi)\}_{\ell=1}^{L_Q}$ are independent. Therefore Lemma~\ref{lem:xuthm3-5} gives
\begin{equation}
\mathbb E\!\left[
\left|
\bar\sigma_{\cp_s^a}(V_T^\pi)
-
\sigma_{\cp_s^a}(V_T^\pi)
\right|^2
\,\middle|\,V_T^\pi
\right]
\leq
\frac{\bar A_\sigma}{L_Q}
+
\bar\varepsilon_\sigma^2,
\end{equation}
where $\bar A_\sigma=O(1)$ for contamination uncertainty and
$\bar A_\sigma=O(N_{\max})$ for TV and Wasserstein uncertainty, while
$\bar\varepsilon_\sigma=0$ for contamination and
$\bar\varepsilon_\sigma=O(2^{-N_{\max}/2})$ for TV and Wasserstein.

We now bound the maximum over all state-action pairs. Let $n=|\mathcal S||\mathcal A|$. Under contamination uncertainty, conditional on $V_T^\pi$, the centered variables $\hat\sigma^{(\ell)}_{\cp_s^a}(V_T^\pi)-\sigma_{\cp_s^a}(V_T^\pi)$ are mean-zero and bounded in absolute value by $\|V_T^\pi\|_{\mathrm{sp}}$. Therefore, the standard maximal bound for bounded empirical averages gives
\begin{equation}
\label{eq:contamination-maximal-support-bound}
\mathbb E\!\left[
\max_{s,a}
\left|
\bar\sigma_{\cp_s^a}(V_T^\pi)
-
\sigma_{\cp_s^a}(V_T^\pi)
\right|^2
\,\middle|\,V_T^\pi
\right]
\leq
\frac{C\log(2n)}{L_Q}\|V_T^\pi\|_{\mathrm{sp}}^2,
\end{equation}
for a universal constant $C>0$. Also, $\|V_T^\pi\|_{\mathrm{sp}}^2\leq 2\|V_T^\pi-V^\pi_\cp\|_{\mathrm{sp}}^2+2\|V^\pi_\cp\|_{\mathrm{sp}}^2$, and Lemma~\ref{lem:wanglemma9} gives $\|V^\pi_\cp\|_{\mathrm{sp}}\leq 4\bar t_{\mathrm{mix}}$ under the reward-span normalization. Since $\|x\|_{\mathrm{sp}}=2\|x\|_{\infty,0}$, \eqref{eq:contamination-maximal-support-bound}, Theorem~\ref{thm:uniform-V}, and the Lipschitz bound above imply that $L_Q=\widetilde{\Theta}(\epsilon^{-2})$ is sufficient for the contamination case.

For TV and Wasserstein uncertainty, Lemma~\ref{lem:xuthm3-5} gives the second-moment and bias bound above, and we use the deterministic inequality $\max_i z_i^2\leq \sum_i z_i^2$. This yields
\begin{align}
\mathbb E\!\left[
\max_{s,a}
\left|
\bar\sigma_{\cp_s^a}(V_T^\pi)
-
\sigma_{\cp_s^a}(V^\pi_\cp)
\right|^2
\right]
&\leq
2n
\left(
\frac{\bar A_\sigma}{L_Q}
+
\bar\varepsilon_\sigma^2
\right)
+
2\mathbb E\!\left[
\|V_T^\pi-V^\pi_\cp\|_\infty^2
\right].
\label{eq:Q-estimation-support-avg-bound}
\end{align}
For TV and Wasserstein uncertainty, $\bar A_\sigma=\cO(N_{\max})$ and $\bar\varepsilon_\sigma=\cO(2^{-N_{\max}/2})$. Thus, choosing $L_Q=\Theta(nN_{\max}\epsilon^{-2})$ and $N_{\max}=\Theta(\log(n/\epsilon))$ makes the RHS of \eqref{eq:Q-estimation-support-avg-bound} $\widetilde O(\epsilon^2)$.

Now set $T=\widetilde{\Theta}(\epsilon^{-2})$ large enough to absorb the logarithmic factors in \eqref{eq:g-second-moment-AC}. For TV and Wasserstein uncertainty, choose $N_{\max}=\Theta(\log(n/\epsilon))$, with a sufficiently large absolute constant so that $2^{-N_{\max}/2}\lesssim \epsilon^2$ and $n2^{-N_{\max}}\lesssim \epsilon^2$. Then \eqref{eq:g-second-moment-AC} gives
\begin{equation}
\mathbb{E}\!\left[ \left| g_T^\pi - g^\pi_{\mathcal P} \right|^2 \right]
\leq \cO(\epsilon^2)
\end{equation}
after increasing the hidden constants if necessary.

Combining \eqref{eq:Qestimationerror}, \eqref{eq:g-second-moment-AC}, the contamination bound \eqref{eq:contamination-maximal-support-bound} or the TV/Wasserstein bound \eqref{eq:Q-estimation-support-avg-bound}, and the value-function bound from
Theorem~\ref{thm:uniform-V}, we obtain
\begin{equation*}
\mathbb E\!\left[\|\hat Q^\pi-Q^\pi\|_\infty^2\right]
=
\widetilde O(\epsilon^2).
\end{equation*}
Increasing the constants in the choices of $T$, $L_Q$, and $N_{\max}$ when applicable gives
\begin{equation*}
\mathbb E\!\left[\|\hat Q^\pi-Q^\pi\|_\infty^2\right]\leq \epsilon^2 .
\end{equation*}
The sample-complexity claim follows because Algorithm~\ref{alg:RobustTD}
uses $\widetilde O(|\mathcal S||\mathcal A|\epsilon^{-2})$ samples under contamination uncertainty and $\widetilde O(|\mathcal S||\mathcal A|N_{\max}\epsilon^{-2})$ samples under TV or Wasserstein uncertainty. The final batched support-estimation step uses $nL_Q$ nominal samples under contamination uncertainty, which is $\widetilde O(|\mathcal S||\mathcal A|\epsilon^{-2})$, and has expected cost $\cO(nL_QN_{\max})$ under TV or Wasserstein uncertainty, which is $\widetilde O((|\mathcal S||\mathcal A|)^2\epsilon^{-2})$.


\section{Some Auxiliary Lemmas for the Proofs}

\begin{lemma}[Theorem IV in \cite{berger1992bounded}]\label{lem:bergerlemmaIV}
    Let $\mathcal{Q}$ be a bounded set of square matrix such that $\rho(Q) < \infty$ for all $Q\in \mathcal{Q}$ where $\rho(\cdot)$ denotes the spectral radius. Then the joint spectral radius of $\mathcal{Q}$ can be defined as 
    \begin{equation}
        \hat{\rho}(\mathcal{Q}) \coloneqq \lim_{k\rightarrow \infty} \sup_{Q_i \in \mathcal{Q}}\rho(Q_k \ldots Q_1)^{\frac{1}{k}} = \lim_{k\rightarrow \infty} \sup_{Q_i \in \mathcal{Q}}\|Q_k \ldots Q_1\|^{\frac{1}{k}}
    \end{equation}
    where $\|\cdot\|$ is an arbitrary norm.
\end{lemma}


\begin{lemma}[Ergodic case of Lemma 9 in \cite{wang2022near}]
\label{lem:wanglemma9}
Consider a finite average-reward Markov reward process induced by a stationary policy $\pi$ and transition kernel $P$, with transition matrix $P_\pi$ and stationary distribution $\nu$. Define
\begin{equation} \label{eq:tmix}
\tau_{\mathrm{mix}}\coloneqq
\arg\min_{t \geq 1}
\left\{
\max_{\mu_0}
\left\|(\mu_0 P_{\pi}^{t})^{\top}-\nu^{\top}\right\|_1
\leq \frac{1}{2}
\right\}.
\end{equation}
Let $\bar r_\pi(s)\coloneqq \sum_{a\in\mca}\pi(a|s)r(s,a)$, $g^\pi_P\coloneqq \nu^\top \bar r_\pi $, and let $R_{\rm sp}\coloneqq \max_{s,a}r(s,a)-\min_{s,a}r(s,a)$. If $P_\pi$ is irreducible and aperiodic, then $\tau_{\mathrm{mix}}<+\infty$. Moreover, for every $T\geq 1$, the centered finite-horizon value defined as
\begin{equation}
h_T^{\pi,P}(s)
\coloneqq
\mathbb E_{\pi,P}\left[
\sum_{t=0}^{T-1}
\big(\bar r_\pi(S_t)-g^\pi_P\big)
\,\middle|\, S_0=s
\right]
\end{equation}
satisfies
\begin{equation}
\|h_T^{\pi,P}\|_{\mathrm{sp}}
\leq
4R_{\rm sp}\tau_{\mathrm{mix}}.
\end{equation}
Consequently, any average-reward relative value function $h^{\pi,P}$ obtained as the limit of $h_T^{\pi,P}$, equivalently any solution to the Poisson equation up to an additive constant, satisfies
\begin{equation}
\|h^{\pi,P}\|_{\mathrm{sp}}
\leq
4R_{\rm sp}\tau_{\mathrm{mix}}.
\end{equation}
In particular, under the normalization $R_{\rm sp}\leq 1$, we have $\|h^{\pi,P}\|_{\mathrm{sp}}\leq 4\tau_{\mathrm{mix}}$.
\end{lemma}

\begin{lemma}[Linear-growth version of Lemma 6 in \cite{zhang2021finite}]
\label{lem:zhanglemma6}
Under the setup and notation in Appendix~\ref{QleariningComplexityproof}, suppose we have
\begin{equation*}
\mathbb E\!\left[
\|w^t\|_{\mathcal H,\overline E}^2
\,\middle|\,\mathcal F^t
\right]
\leq
A_0+B_0\|x^t-x^*\|_{\mathcal H,\overline E}^2 .
\end{equation*}
Let $\ell_s\coloneqq \rho_2^2$, so that $\|x\|_{s,\overline E}^2
\leq
\ell_s\|x\|_{\mathcal H,\overline E}^2$. Then we have
\begin{equation}
\E\!\left[
\|x^{t+1}-x^t\|^2_{s,\overline E}
\,\middle|\,\mathcal F^t
\right]
\leq
(16+4B_0)c_u\ell_s\eta_t^2
M_{\overline E}(x^t-x^*)
+
2A_0\ell_s\eta_t^2 .
\end{equation}
\end{lemma}

\begin{lemma}[Theorem 5.2-5.4 in \cite{xu2025finite}, Lemma 12 in \cite{wang2023model}] \label{lem:xuthm3-5}
    Let $\hat{\sigma}_{\cp^a_s}(V)$ be the estimator of ${\sigma}_{\cp^a_s}(V)$ obtained from Algorithm \ref{alg:sampling}, then the following properties hold:
    \begin{enumerate}
        \item Denote $M$ as the number of state-action pair samples needed. Then $M=1$ under contamination uncertainty set. Furthermore, under TV and Wasserstein distance uncertainty sets, $M = 2^{N'+1}$ (where $N'=\min\{N,N_{\max}\}$), which then implies,
            \begin{equation}
            \mathbb{E}[M]=N_{\max}+2=\mathcal{O}(N_{\max}).
            \end{equation}
        \item Under contamination uncertainty set, $\hat{\sigma}_{\cp^a_s}(V)$ is unbiased satisfying 
        \begin{equation}
            \E\left[\hat{\sigma}_{\cp^a_s}(V)\right] = {\sigma}_{\cp^a_s}(V).
        \end{equation}
         Under TV uncertainty set, we have a bias that is exponentially decaying with respect to the maximum level $N_{\mathrm {max}}$ as follows:
            \begin{equation}
            \abs{\mathbb{E}\bigl[\hat{\sigma}_{\cp^a_s}(V) - {\sigma}_{\cp^a_s}(V)\bigr] } \leq
            6(1+\frac{1}{\delta}) 2^{-\frac{N_{\max}}{2}}\|V\|_{\mathrm{sp}}
            \end{equation}
            where $\delta$ denotes the radius of TV distance and $\|\cdot\|_{\rm sp}$ denotes the span semi-norm. Similarly, under Wasserstein uncertainty set, we have:
            \begin{equation}
            \abs{\mathbb{E}\bigl[\hat{\sigma}_{\cp^a_s}(V) - {\sigma}_{\cp^a_s}(V)\bigr] } \leq
            6\cdot 2^{-\frac{N_{\max}}{2}}\|V\|_{\mathrm{sp}}
            \end{equation}
        \item Under contamination uncertainty set, the variance of $\hat{\sigma}_{\cp^a_s}(V)$ is bounded by the $l_2$ norm of $V$ as 
        \begin{equation}
            \mathrm{Var}(\hat{\sigma}_{\cp^a_s}(V)) \leq  \|V\|^2.
        \end{equation}
         Under TV uncertainty set, the  variance of $\hat{\sigma}_{\cp^a_s}(V)$ is bounded by is as follows:
            \begin{equation}
            \mathrm{Var}(\hat{\sigma}_{\cp^a_s}(V)) \leq  3\|V\|_{\mathrm{sp}}^2 + 144(1+\frac{1}{\delta})^2\|V\|_{\mathrm{sp}}^2 N_{\max}
            \end{equation}
        Similarly, under Wasserstein uncertainty set, we have:
            \begin{equation}
             \mathrm{Var}(\hat{\sigma}_{\cp^a_s}(V)) \leq  3\|V\|_{\mathrm{sp}}^2 + 144\|V\|_{\mathrm{sp}}^2 N_{\max}
            \end{equation}
    \end{enumerate}
\end{lemma}

\begin{lemma}[Lemma A.3 in \cite{xu2025finite}] \label{lem:xulemA.3}
    Define the Dobrushin’s coefficient of an $n$ dimensional Markov matrix $P$ as 
\begin{equation} 
\tau(P)\coloneqq 1 - \min_{i<j}\sum_{s=1}^n \min(P_{is},P_{js}),
\end{equation}
Let $\cp$ be a family of Markov matrix with the same dimension with each having a unique stationary distribution. We define the family of fluctuation matrices to be 
\begin{equation} 
    \mathcal{F} \coloneqq \{\kp - E_\kp : \kp \in \cp\}
\end{equation}
where $E_\kp$ is the matrix with all rows being identical to the stationary distribution of $\kp$. Then the joint spectral radius of the family $\mathcal{F}$ is upper bounded by the following:
\begin{equation} 
    \hat{\rho}(\mathcal{F}) \leq \inf_{m \geq 1}\left(\sup_{\kp_i \in \cp}\tau(\kp_1 \cdot \ldots \cdot \kp_m)\right)^{\frac{1}{m}}
\end{equation}
\end{lemma}

\begin{lemma} \label{lem:seminorm2norm}
For any $n\in \mathbb{N}$, let $X=\mathbb{R}^n$, let $\overline{E}=\{c\e:c\in\mathbb{R}\}$, and let $\|\cdot\|_{sn}$
be a semi-norm with kernel being $\overline{E}$.  Define the linear functional $\ell(x)=\frac{1}{n}\sum_{i=1}^n x_i$, and then define $\|x\|_n = \|x-\ell(x)\e \|_{sn} +\bigl|\ell(x)\bigr|$. Then $\|\cdot\|_n$ is a norm on $X$, and moreover
\begin{equation}
\bigl(\|\cdot\|_n \ast_{\inf}\delta_{\overline{E}}\bigr)(x)
=\inf_{c\in \mathbb{R}}\|x-c\e\|_n
= \|x\|_{sn}
\quad\text{for all }x\in X.
\end{equation}
\end{lemma}

\begin{proof}
We first show that $\|\cdot\|_n$ is a genuine norm. 
\begin{itemize}
  \item Positive homogeneity.  For any $\lambda\in\mathbb{R}$,
  \begin{equation*}
    \|\lambda x\|_n
    = \|\lambda x-\ell(\lambda x)\e\|_{sn}
      + \bigl|\ell(\lambda x)\bigr|
    = |\lambda| \|x-\ell(x)\e\|_{sn}
      + |\lambda||\ell(x)|
    = |\lambda|\|x\|_n.
  \end{equation*}
  \item Triangle inequality.  For $x,y\in X$,
  \begin{align*}
    \|x+y\|_n
    &= \|(x+y)-\ell(x+y)\e\|_{sn}
        + \bigl|\ell(x+y)\bigr|\\
    &= \|(x-\ell(x)\e)+(y-\ell(y)\e)\|_{sn}
        + \bigl|\ell(x)+\ell(y)\bigr|\\
    &\le \|x-\ell(x)\e\|_{sn}+\|y-\ell(y)\e\|_{sn}
         + |\ell(x)|+|\ell(y)|\\
    &= \|x\|_n + \|y\|_n.
  \end{align*}
  \item Positive definiteness.  If $\|x\|_n=0$, then $|\ell(x)|=0$ and
  $\|x-\ell(x)\e\|_{sn}=0$. Hence $\ell(x)=0$ and
  $x-\ell(x)\e\in\overline E$. Thus $x\in\overline E$, so $x=c\e$ for some
  $c\in\mathbb R$. Since $\ell(x)=c=0$, we obtain $x=0$.
\end{itemize}

Regarding the infimal convolution, by definition,
\begin{equation*}
(\|\cdot\|_n\ast_{\inf}\delta_{\overline{E}})(x)
=\inf_{c\in \mathbb{R}}\|x-c\e\|_n.
\end{equation*}
Then
\begin{equation*}
\|x - c\e\|_n
= \|(x-c\e)-\ell(x-c\e)\e\|_{sn}
  + \bigl|\ell(x-c\e)\bigr|.
\end{equation*}
But $\ell(x-c\e)=\ell(x)-c$ and
\begin{equation*}
(x-c\e)-\ell(x-c\e)\e
= x-\ell(x)\e
\end{equation*}
independent of $c$.  Hence
\begin{equation*}
\|x - c\e\|_n
= \|x-\ell(x)\e\|_{sn} + |\ell(x)-c|.
\end{equation*}
Minimizing over $c\in\mathbb{R}$ by taking $c=\ell(x)$ gives
\begin{equation*}
\inf_{c\in\mathbb{R}}\|x - c\e\|_n
= \|x-\ell(x)\e\|_{sn}
= \|x\|_{sn},
\end{equation*}
since $\ell(x)\e\in\ker(\|\cdot\|_{sn})$.  This completes the proof.
\end{proof}

\begin{lemma} \label{lem:normtranslations}
Let $x\in \mathbb{R}^{SA}$ satisfy $x_i=0$ for some fixed index $i$.  Then
\begin{equation*}
\|x\|_\infty\leq \|x\|_{\mathrm{sp}}
=\max_{1\le j\le n}x_j -\min_{1\le j\le n}x_j
\leq 2\|x\|_\infty.
\end{equation*}
Moreover, since all semi-norms with the same kernel spaces are equivalent, there are constants $c_{H},C_{H}>0$ such that for the semi-norm $\|\cdot\|_H$ defined in \eqref{eq:Q-learningcontraction},
\begin{equation*}
c_H\|x\|_{\mathrm{sp}}\leq \|x\|_{H}\le\;C_H\|x\|_{\mathrm{sp}}
\quad\forall\,x\in\mathbb{R}^n,
\end{equation*}
then
\begin{equation} 
c_H\|x\|_{\infty}\leq c_H\|x\|_{\mathrm{sp}}\leq\|x\|_{H}\leq C_H\|x\|_{\mathrm{sp}}\leq 2C_H \|x\|_\infty.
\end{equation}

\begin{proof}
Since $x_i=0$, for every $j$ we have $-\|x\|_\infty \le x_j\le\|x\|_\infty$.  Hence
\begin{equation*}
\max_j x_j\;\le\;\|x\|_\infty,
\quad
\min_j x_j\;\ge\;-\,\|x\|_\infty,
\end{equation*}
and so
\begin{equation*}
\|x\|_\infty = \max \{\max_j x_j - 0,0-\min_j x_j\} \leq \|x\|_{\mathrm{sp}}
=\max_j x_j - \min_j x_j
\le \|x\|_\infty -(-\|x\|_\infty)
=2\,\|x\|_\infty.
\end{equation*}
Since $\|\cdot\|_{\mathrm{sp}}$ and $\|\cdot\|_{H}$ both have the same kernel of $\{c\e : c\in\mathbb{R}\}$, by the equivalence of semi-norms, it follows that there exists $c_H$ and $C_H$ such that
\begin{equation}
c_H\|x\|_{\infty}\leq c_H\|x\|_{\mathrm{sp}}\leq\|x\|_{H}\leq C_H\|x\|_{\mathrm{sp}}\leq 2C_H \|x\|_\infty \nonumber
\end{equation}
as claimed.
\end{proof}
\end{lemma}

\section{Numerical Evaluations}

In this section, we provide some numerical evaluation results to demonstrate the performance of our algorithms. We evaluate our Algorithm \ref{alg:Qlearning} on a ride-hailing domain with five zones and Algorithm \ref{alg:AC} on a minimal three-state control loop. Both environments are average-reward problems with two actions and ergodic dynamics. We apply robust evaluation with contamination, TV distance, and Wasserstein-$1$ distance uncertainty sets around the nominal transition models; the corresponding robustness parameters are reported with each experiment.

We note that Figures \ref{fig:Q}-\ref{fig:AC} show convergence end-to-end across contamination, TV distance and Wasserstein-1 uncertainty sets with various radius. In the figure, the anchored sup-norm error $\|Q_t - Q^*\|_\infty$ decays for the $Q$-learning algorithm , while the robust average reward $g_t$ rises toward $g^*$ and stabilizes for the actor-critic algorithm. We note that the $Q^*$ and $g^*$ are computed via a coarse grid search over stationary policies with a discretized uncertainty set, so any small terminal gap may reflect {reference discretization} rather than algorithmic bias. Overall, the trends align with our $\tilde{O}(\varepsilon^{-2})$ dependence on the target accuracy. The remaining section provides the detailed settings of the numerical experiments.

\begin{figure*}[t]
  \centering   \subfloat{\includegraphics[width=.3\textwidth]{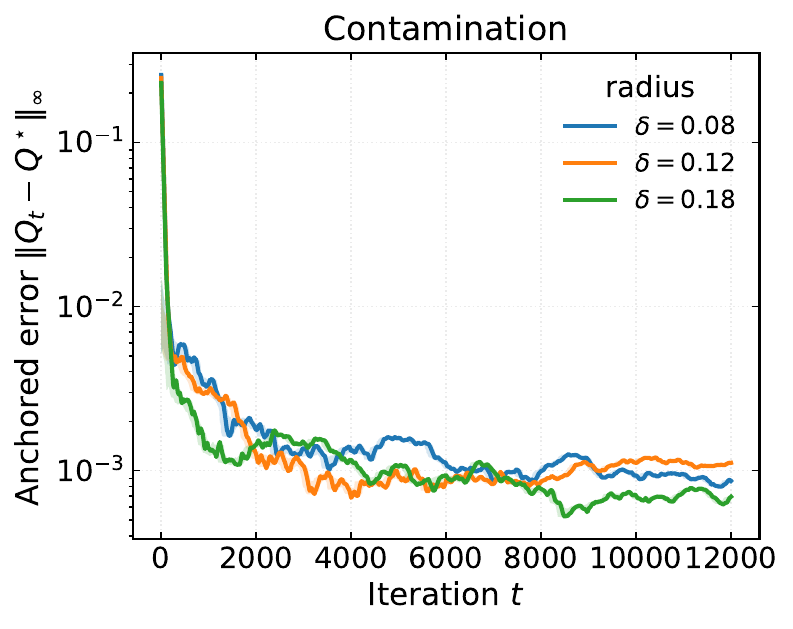}\label{fig:sub1_ride}}\hfil
  \subfloat{\includegraphics[width=.3\textwidth]{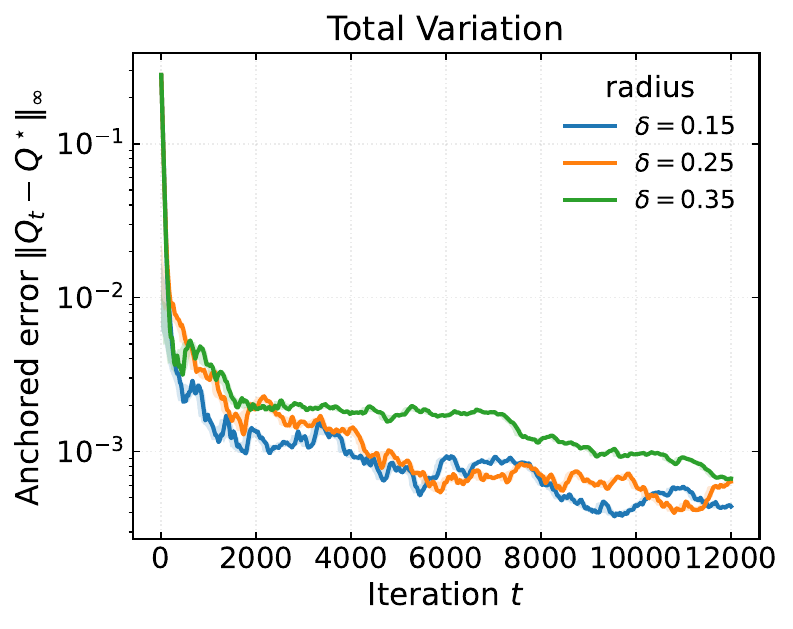}\label{fig:sub2_ride}}\hfil
  \subfloat{\includegraphics[width=.3\textwidth]{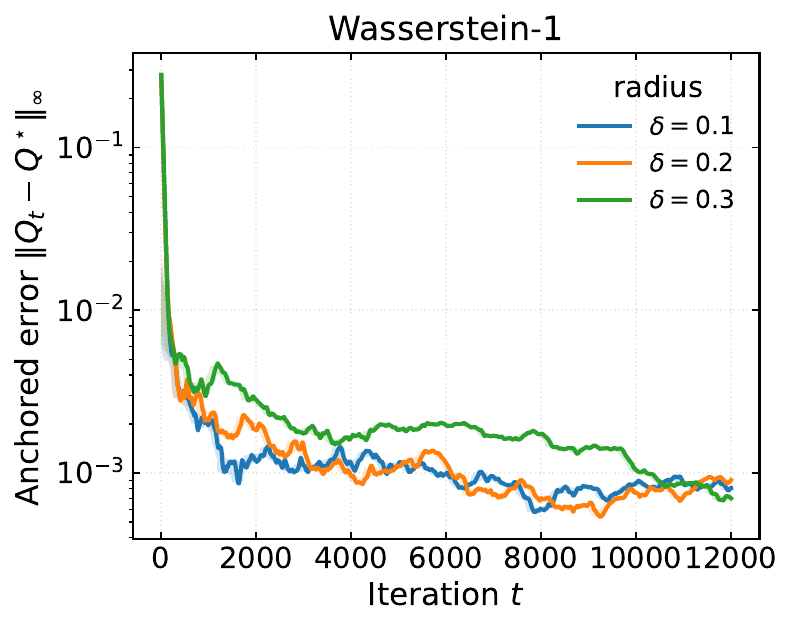}\label{fig:sub3_ride}}
  \caption{Ride-hailing via robust $Q$-learning} \label{fig:Q} 
\end{figure*}

\begin{figure*}[t] 
  \centering 
  \subfloat{\includegraphics[width=.3\textwidth]{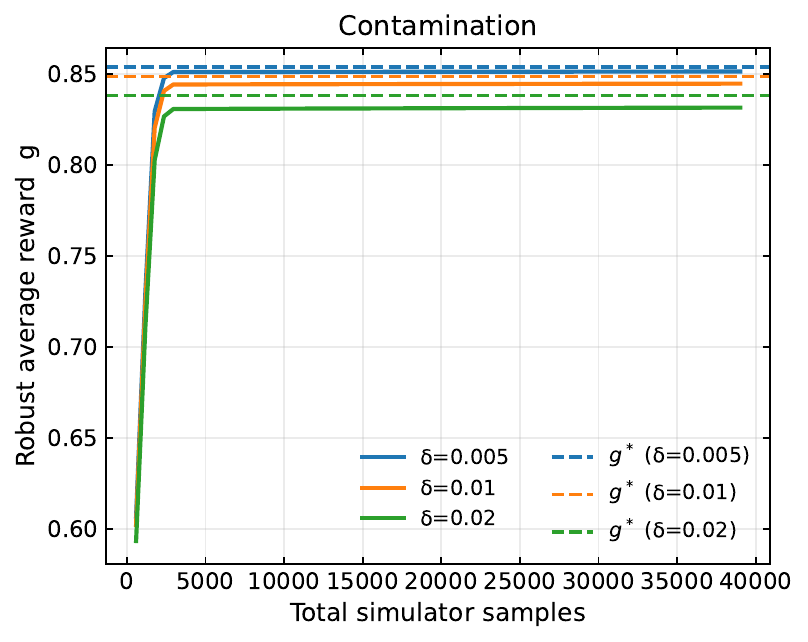}\label{fig:sub1_loop}}\hfil
  \subfloat{\includegraphics[width=.3\textwidth]{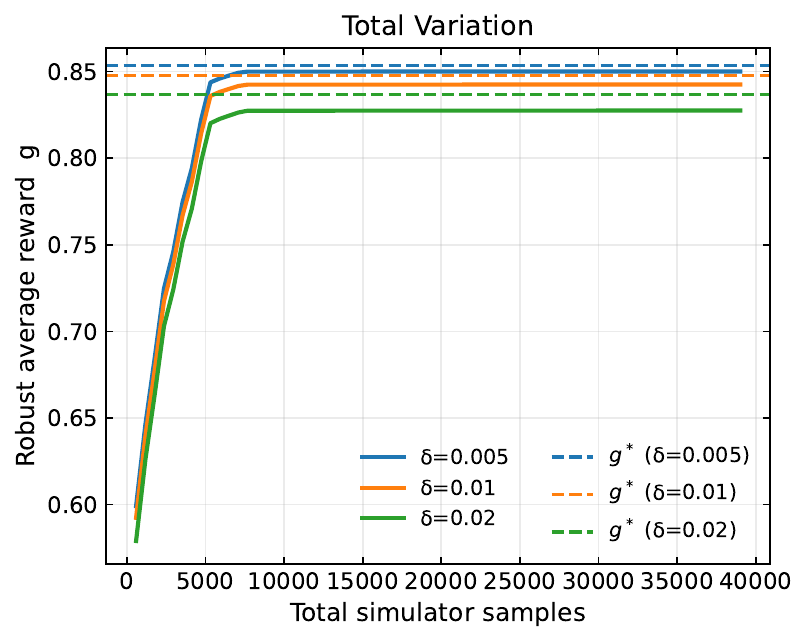}\label{fig:sub2_loop}}\hfil
  \subfloat{\includegraphics[width=.3\textwidth]{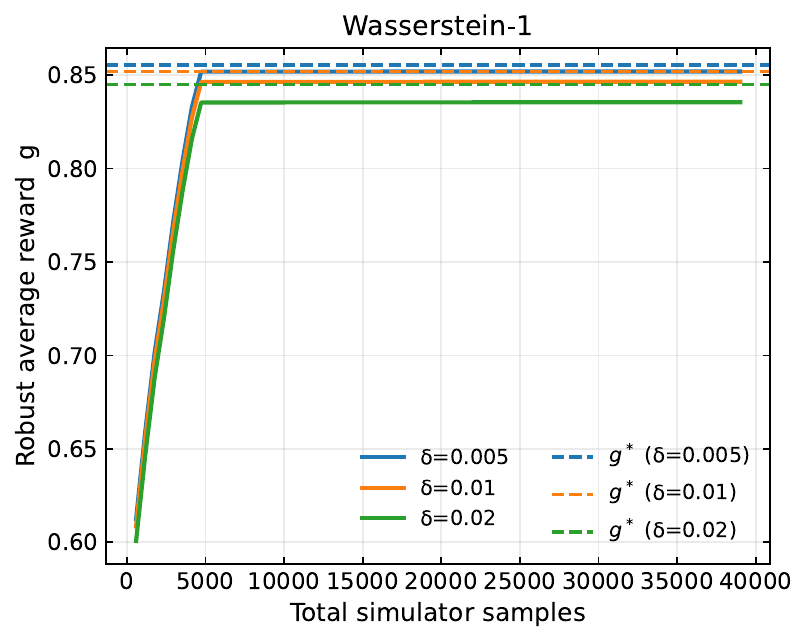}\label{fig:sub3_loop}}
  \caption{Control loop via robust actor-critic}\label{fig:AC}
\end{figure*}

\subsection{Three-state “operate vs.\ protect” loop}
This environment abstracts a system that trades off performance against resource protection (e.g., server throughput vs.\ thermal throttling, or battery discharge vs.\ charge). The state is a coarse operating level with three values: \textit{low}, \textit{medium}, and \textit{high}. Action \textit{boost} drives performance upward; action \textit{conserve} protects resources and favors staying or moving left. Rewards increase with the operating level; \textit{conserve} adds a small immediate penalty.

\paragraph{Rewards.} Let the base per-step utilities by level be $(0.20, 0.55, 0.90)$. The two actions incur penalties $(0.00, 0.05)$ respectively, so the realized rewards are:
\begin{center}
\begin{tabular}{lcc}
\toprule
State & Boost & Conserve\\
\midrule
Low    & 0.20 & 0.15\\
Medium & 0.55 & 0.50\\
High   & 0.90 & 0.85\\
\bottomrule
\end{tabular}
\end{center}

\paragraph{Transitions.} Rows index the current state; columns index the next state in the order Low, Medium, High. The \textit{boost} action pushes right ($\approx$80\% right, 17\% self, 3\% left in the interior), while \textit{conserve} pushes left (symmetrically). At the boundaries we use the exact probabilities induced by the code’s normalization.\footnote{At the edges, the “left” or “right” assignment can coincide with the self index; rows are normalized afterwards.}
\begin{center}
\textbf{Boost}
\quad
\begin{tabular}{lccc}
\toprule
 & Low & Medium & High\\
\midrule
Low    & 0.036 & 0.964 & 0.000\\
Medium & 0.030 & 0.170 & 0.800\\
High   & 0.000 & 0.100 & 0.900\\
\bottomrule
\end{tabular}

\textbf{Conserve}
\quad
\begin{tabular}{lccc}
\toprule
 & Low & Medium & High\\
\midrule
Low    & 0.900 & 0.100 & 0.000\\
Medium & 0.800 & 0.170 & 0.030\\
High   & 0.000 & 0.964 & 0.036\\
\bottomrule
\end{tabular}
\end{center}

\subsection{Five-state ride-hailing domain}
This environment models a single driver operating across five zones: Downtown, University, Airport, Suburbs, and a Depot (garage/charging). The state is the driver’s current zone. The action is an operating mode: \textit{aggressive} (accept surge/long hauls; proactive repositioning toward hubs) or \textit{conservative} (prefer short local rides; stay nearby). Immediate rewards represent net earnings by zone; hubs pay more, Depot is low, and aggressive typically yields higher per-step revenue. The nominal transition model captures destination mixes induced by each mode: aggressive drifts the driver toward major hubs and the Depot, whereas conservative increases self-loops and nearby moves. Zones are arranged along a line to define geographic closeness for Wasserstein robustness.

\paragraph{Rewards.}
\begin{center}
\begin{tabular}{lcc}
\toprule
Zone & Aggressive & Conservative\\
\midrule
Downtown  & 2.0 & 1.2\\
University& 1.6 & 1.1\\
Airport   & 2.2 & 1.4\\
Suburbs   & 1.2 & 0.9\\
Depot     & 0.8 & 0.5\\
\bottomrule
\end{tabular}
\end{center}

\paragraph{Transitions.} Rows index the current zone; columns index the next zone in the order Downtown, University, Airport, Suburbs, Depot.
\begin{center}
\textbf{Aggressive}
\quad
\begin{tabular}{lccccc}
\toprule
 & DT & Univ & Airp & Sub & Depot\\
\midrule
Downtown   & 0.20 & 0.20 & 0.30 & 0.10 & 0.20\\
University & 0.30 & 0.15 & 0.25 & 0.10 & 0.20\\
Airport    & 0.35 & 0.10 & 0.20 & 0.10 & 0.25\\
Suburbs    & 0.25 & 0.15 & 0.20 & 0.20 & 0.20\\
Depot      & 0.30 & 0.20 & 0.25 & 0.10 & 0.15\\
\bottomrule
\end{tabular}

\textbf{Conservative}
\quad
\begin{tabular}{lccccc}
\toprule
 & DT & Univ & Airp & Sub & Depot\\
\midrule
Downtown   & 0.40 & 0.20 & 0.10 & 0.20 & 0.10\\
University & 0.20 & 0.45 & 0.05 & 0.20 & 0.10\\
Airport    & 0.15 & 0.10 & 0.45 & 0.15 & 0.15\\
Suburbs    & 0.10 & 0.25 & 0.10 & 0.45 & 0.10\\
Depot      & 0.15 & 0.15 & 0.15 & 0.15 & 0.40\\
\bottomrule
\end{tabular}
\end{center}

\noindent
In both environments, the learning objective is to find a stationary policy that maximizes long-run average performance while remaining reliable under plausible misspecification of the nominal transition model. The concrete rewards and transitions above yield ergodic chains and cleanly expose the effects of contamination, total-variation, and geography-aware Wasserstein uncertainty.

\end{document}